\documentclass[11pt,a4paper]{article}

\usepackage[british]{babel}
\usepackage[a4paper,top=2.5cm,bottom=2.5cm,left=2.5cm,right=2.5cm]{geometry}
\usepackage[utf8]{inputenc}
\usepackage[T1]{fontenc}
\usepackage{lmodern}

\usepackage{amsmath,amssymb,amsthm}
\usepackage{amsfonts}
\usepackage{mathrsfs}
\usepackage{bm}
\usepackage{graphicx}
\usepackage{xcolor}
\usepackage{booktabs}
\usepackage{threeparttable}
\usepackage{multirow}
\usepackage{diagbox}
\usepackage{subcaption}
\usepackage{algorithm}
\usepackage{algorithmic}
\usepackage{listings}
\usepackage{enumitem}
\usepackage{chngcntr}
\usepackage{float}
\usepackage{comment}
\usepackage{authblk}
\usepackage[justification=raggedright,singlelinecheck=false]{caption}
\usepackage{setspace}
\usepackage{natbib}
\usepackage[colorlinks=true, allcolors=blue]{hyperref}

\newtheorem{theorem}{Theorem}
\newtheorem{proposition}{Proposition}
\newtheorem{lemma}{Lemma}
\newtheorem{corollary}{Corollary}
\newtheorem{remark}{Remark}
\theoremstyle{definition}
\newtheorem{definition}{Definition}
\newtheorem{assumption}{Assumption}

\newcommand{\E}{\mathbb{E}}
\newcommand{\R}{\mathbb{R}}
\newcommand{\N}{\mathcal{N}}
\newcommand{\loss}{\mathcal{L}}
\newcommand{\enc}{f_\phi}
\newcommand{\dec}{G_\theta}
\newcommand{\disc}{D_\psi}

\newcommand{\Zz}{\mathcal{Z}}
\newcommand{\Uset}{\mathcal{U}}
\DeclareMathOperator{\MMD}{MMD}
\DeclareMathOperator{\diag}{diag}

\newif\ifinformsbuild
\informsbuildfalse

\title{Generative Robust Optimisation}

\author[1]{Yuhui Yin}
\author[1,*]{Vassilis M. Charitopoulos}
\affil[1]{\small Department of Chemical Engineering, Sargent Centre for Process
Systems Engineering, UCL (University College London), Torrington Place,
London WC1E 7JE, UK}
\affil[*]{\small Corresponding author. \texttt{v.charitopoulos@ucl.ac.uk}}
\date{}

\begin{document}
\maketitle

\begin{abstract}
Classical uncertainty sets for robust optimisation impose fixed geometric shapes that cannot represent the complex dependencies present in real-world data.
We propose \emph{Generative Robust Optimisation} (GRO), a framework in which a deep generative model defines the uncertainty set as the image of a neural network decoder over a calibrated latent set, naturally accommodating nonlinear correlations, asymmetry, and multimodality.
A five-point evaluation framework (reconstruction fidelity, distribution matching, latent regularity, robust relevance, and computational tractability) provides systematic, model-agnostic criteria for assessing any neural network-based uncertainty set.
We instantiate this framework with a Wasserstein Adversarial Autoencoder employing Gaussian mixture model-guided training for latent regularity and constraint-consistency regularisation for robust relevance.
Restricting the decoder to ReLU activations enables exact worst-case verification through mixed-integer programming embedding.
Extensive experiments on a production planning problem across six uncertainty distributions and six generative architectures, together with a multi-period facility location study, validate the framework and demonstrate that systematic attention to all five criteria yields uncertainty sets that are simultaneously expressive, well-calibrated, and optimisation-tractable.

\medskip
\noindent\textbf{Keywords:} robust optimisation, data-driven uncertainty sets, deep generative models, five-point evaluation framework, Wasserstein autoencoder, MILP
\end{abstract}


\section{Introduction}\label{sec:introduction}

Robust optimisation (RO) seeks decisions that remain feasible under all realisations of uncertain parameters within a prescribed uncertainty set~\citep{BenTal2009,Bertsimas2011}.
The geometry of this set governs both the conservatism of the resulting decisions and the computational tractability of the robust counterpart.
Box uncertainty sets, introduced by \citet{Soyster1973}, preserve the original mathematical nature of the deterministic problem but ignore correlations among uncertain parameters and hence they rank first in terms of conservatism.
Ellipsoidal sets~\citep{BenTalNemirovski1998} capture second-order correlations and lead to second-order cone programs, yet impose symmetric, unimodal structure.
In order to ameliorate the worst-case orientation of RO problems, polyhedral or budget uncertainty sets were introduced by~\citet{BertsimasSim2004} to control the number of parameters that simultaneously deviate from their nominal values, providing an interpretable approach for adjusting conservatism.
The robust-optimisation literature shows that this geometry choice fundamentally constrains the trade-off between conservatism, tractability, and probabilistic protection~\citep{BenTalNemirovski1998,BertsimasSim2004,Bertsimas2011,Bertsimas2018}.
Conventional uncertainty sets are broadly applicable and yield tractable robust counterparts~\citep{Bertsimas2011}. However, many practical applications involve uncertainty that can deviate substantially from the assumptions embedded in these geometric shapes.
Production costs may show heavy-tailed or multimodal fluctuations, renewable generation varies with weather regimes that induce nonlinear correlations, and demand patterns display asymmetric dependencies across time periods and regions.
When the true uncertainty resides within an irregularly shaped region, any fixed geometric set must either enclose excessive empty space, which may lead to over-conservatism, or fail to cover critical tail scenarios that cause constraint violations.

Data-driven methods have sought to bridge this gap between assumed geometry and observed data.
\citet{Bertsimas2018} proposed a systematic schema for constructing uncertainty sets from data via statistical hypothesis tests (Kolmogorov--Smirnov, Anderson--Darling, $\chi^2$) with finite-sample probabilistic guarantees, representing a major advance in principled set construction while retaining geometric shapes.
Machine learning methods offer further flexibility. Kernel-learning uncertainty sets capture nonlinear decision boundaries in robust combinatorial optimisation~\citep{Loger2024}. Learning-based robust-optimisation procedures provide statistical guarantees for data-trained decisions~\citep{Hong2021}. Neural networks extend uncertainty-set learning to nonconvex geometry~\citep{Goerigk2023}.
Ensemble methods also learn set sizes jointly with forecasting models~\citep{Andrianesis2024,Stratigakos2025}.
Distributionally robust optimisation (DRO) takes a different approach by hedging against distributional ambiguity via Wasserstein balls~\citep{Kuhn2019,Gao2024} or moment-based ambiguity sets~\citep{WiesemannKuhnSim2014}.
These methods impose prescribed shapes (typically divergence balls or moment polytopes) on the set of possible distributions rather than on the parameter space directly.
A growing literature develops tractable reformulations for an expanding range of such ambiguity sets~\citep{KuhnShafieeWiesemann2025}. As in robust optimisation, however, the geometry of the ambiguity set is prescribed a priori rather than learned from the data.
Despite this progress, all existing data-driven methods remain \emph{discriminative}: they learn only the boundary of a region (classification surfaces, kernel contours, or convex hulls).

Advances in generative modelling, including variational autoencoders~\citep{KingmaWelling2014}, generative adversarial networks~\citep{Goodfellow2014}, normalising flows~\citep{Papamakarios2021}, and diffusion models~\citep{Ho2020}, have produced flexible models for learning complex distributions from data.
These models have been applied to \emph{scenario generation} and \emph{distributional modelling} in optimisation: producing samples or approximating distributions for stochastic programming or distributionally robust optimisation~\citep{Xu2024flowdro,Chenreddy2022}.
However, generating scenarios from a learned distribution is fundamentally different from defining an \emph{uncertainty set} for robust optimisation: the former produces samples, while the latter requires a bounded region over which worst-case feasibility must be guaranteed.

We propose \emph{Generative Robust Optimisation} (GRO), which bridges this gap by defining the uncertainty set as the image of a trained decoder:
\begin{equation}\label{eq:intro-gen-set}
    \Uset = \bigl\{\dec(\bm{z}) : \bm{z} \in \Zz\bigr\}
\end{equation}
where $\dec: \R^{n_z} \to \R^{n_{\xi}}$ is a decoder trained so that $\dec(\bm{z}),\; \bm{z} \sim \N(\bm{0}, \bm{I})$ approximates the data distribution, and $\Zz \subset \R^{n_z}$ is a simple set (box or ball) in the latent space.
Concurrent work~\citep{Brenner2025} explores this approach in adaptive robust optimisation, embedding a variational autoencoder (VAE) as a preliminary instantiation. A plain VAE, however, guarantees neither a calibrated latent-to-data mapping nor exact worst-case verification. This paper proposes a framework and architecture that provide both.
By the universal approximation theorem~\citep{LuLu2020}, a sufficiently expressive decoder can in principle represent complex set geometries (including sets with nonlinear correlations, asymmetric tails, and multimodal structure) that no single geometric shape can capture.
The same trained decoder also supports different levels of conservatism by adjusting the latent bounds (box widths, ball radius, or quantile level) without retraining.
This representation also yields unified calibration: training transports heterogeneous data-space distributions to a chosen latent reference distribution, here $\N(\bm{0},\bm{I})$. In that standard-normal space, one set of box, ball, or quantile rules applies across cases. The decoder then maps the calibrated latent set back to data space, recovering the case-specific correlations, skewness, and multimodality.
However, the practical validity of this construction depends on requirements that go well beyond standard generative-model training, and that to date have not been systematically identified or addressed.

Despite the potential of such decoder-based uncertainty sets, no systematic framework exists for evaluating whether a generative model is suitable for deployment in robust optimisation.
A model that achieves low reconstruction error may still produce a poorly calibrated latent space, causing the uncertainty set to miss critical tail scenarios.
A model with excellent distribution matching may employ architectures that are computationally intractable for worst-case verification.
A model satisfying both reconstruction and distributional criteria may still fail to preserve the constraint-violation structure of the uncertainty that matters for the downstream optimisation problem.
Each of these limitations corresponds to a distinct stage of the pipeline from data to robust decisions, and each requires its own evaluation criterion.
We propose a five-point evaluation framework that identifies the jointly necessary conditions a generative model must satisfy to serve as a reliable uncertainty set for robust optimisation: reconstruction fidelity (P1), distribution matching (P2), latent regularity (P3), robust relevance (P4), and computational tractability (P5).
We instantiate this framework with the Wasserstein Adversarial Autoencoder with Gaussian mixture model-guided training (WAAE-GMM-RO), an architecture designed to satisfy all five criteria through targeted training objectives. Any encoder--decoder architecture satisfying all five criteria can serve as the generative uncertainty set.

\paragraph{Contributions.}
The specific contributions of this paper are as follows.
\begin{enumerate}[leftmargin=*]
    \item A \emph{five-point evaluation framework} for generative uncertainty sets in robust optimisation. The framework is model-agnostic and identifies the jointly necessary conditions that any decoder-based uncertainty set must satisfy: reconstruction fidelity (P1), distribution matching (P2), latent regularity (P3), robust relevance (P4), and computational tractability (P5). As part of the framework we also develop three latent set constructions with calibration theory linking latent bounds to data-space coverage, with formal statements and proof details collected in Appendix~\ref{ec:proofs}.

    \item A \emph{neural-network instantiation, WAAE-GMM-RO, that simultaneously satisfies all five points.} Key components include an encoder--decoder pair preprocessed by principal component analysis (PCA), a Wasserstein adversarial data critic for (P1)--(P2), a GMM-guided component moment constraint that provides a constructive sufficient condition for (P3), a constraint-consistency loss that teaches the decoder to preserve the constraint-violation structure of the uncertainty distribution in the tails that govern worst-case feasibility (P4), and a single-pass ReLU decoder with skip connection that is exactly MILP-embeddable (P5).

    \item A \emph{generative robust-optimisation solution pipeline} that turns the trained decoder into a deployable oracle. We export the ReLU-based decoder in its MILP form, followed by an offline optimality-based bound tightening (OBBT) to obtain an exact MIP worst-case verification of the decoded uncertainty set, and run that verification inside a parallel GRO cutting-plane algorithm~\citep{Marousi2025} with a feasibility check that prunes the bulk of MIP solves.

    \item \emph{Computational validation.} On a production planning problem we evaluate six uncertainty distributions (independent / correlated $\times$ normal / skewed / mixture) and five generative architectures (VAE, WAAE, WAAE-GMM-RO, normalising flow, diffusion model), with analytical SOCP ground truth on the independent normal case and empirical worst-case violation matrices on the other five. We further extend the pipeline to combinatorial robust optimisation on a multi-period facility location problem with binary opening decisions and per-constraint MISOCP oracles, demonstrating that the same WAAE-GMM-RO decoder and parallel cutting-plane pipeline carry through to MILP-master settings.
\end{enumerate}

\paragraph{Paper organisation.}
The remainder of this paper is organised as follows.
Section~\ref{sec:gen-set} introduces the generative uncertainty set construction and the latent calibration theory. Section~\ref{sec:five-points} presents the five-point evaluation framework. Section~\ref{sec:waae} describes the proposed WAAE-GMM-RO architecture and Section~\ref{sec:waae-ro} the GMM-guided training with PCA preprocessing and the constraint-consistency (P4) loss. Section~\ref{sec:omlt} details the MILP - OBBT steps and the cutting-plane solving pipeline. Section~\ref{sec:results} reports the production planning case study (Section~\ref{sec:cs-case-study}), with analytical SOCP ground-truth comparison on the independent normal distribution and decoder-versus-empirical evaluation on the other five distributions. Section~\ref{sec:fl-case-study} extends the same pipeline to the multi-period facility location problem.  Finally, in Section~\ref{sec:conclusion} we provide concluding remarks and some open research directions.

\section{Methodology}
\label{sec:methodology}

The generative robust set construction follows a two-layer workflow (Fig.~\ref{fig:two-layer-workflow}) that organises the rest of this section. The offline layer (top, Algorithms~\ref{alg:waae-ro-p1} and~\ref{alg:waae-ro-p2}) produces the trained decoder. It learns the encoder--decoder pair from historical data $\bm{\xi} \sim P^{*}(\bm{\xi})$, preprocessed by PCA and standardisation and regularised in latent space by a $K$-component GMM guide. A second, tail-weighted phase adds robust relevance. Of the whole pipeline, only the trained decoder (yellow box) is reused at inference. The \emph{online} layer (bottom, Algorithm~\ref{alg:ro-cutting-plane}) employs that decoder within a robust cutting planes algorithmic procedure~\citep{Marousi2025}: at each iteration the master proposes a candidate decision $\bm{x}$, and the oracle then takes that $\bm{x}$ as a fixed parameter and searches the calibrated latent ball $\|\bm{z}\|_2 \leq \gamma$ for the worst-case constraint-violating scenario. The oracle runs with a three-phase inner structure: feasibility check, parallel exact-MIP verification, and robust cut generation. Confirmed cuts are added to the master and the algorithm converges when a robust-feasible decision $\bm{x}^{\star}$ is certified or when the computational budget set is exhausted. Sections~\ref{sec:gen-set}--\ref{sec:waae-ro} develop the offline layer and Section~\ref{sec:omlt} the online layer. The five-point framework (Section~\ref{sec:five-points}) is model-agnostic and applies to any encoder--decoder generative model. We instantiate it with the WAAE-GMM-RO architecture (Sections~\ref{sec:waae}--\ref{sec:waae-ro}), and Section~\ref{sec:results} evaluates VAE, normalising-flow, and diffusion alternatives against the same five points. The architecture choice matters most at (P5), where only piecewise-linear decoders admit an exact MILP embedding.

\begin{figure}[tbp]
\centering
\includegraphics[width=\textwidth]{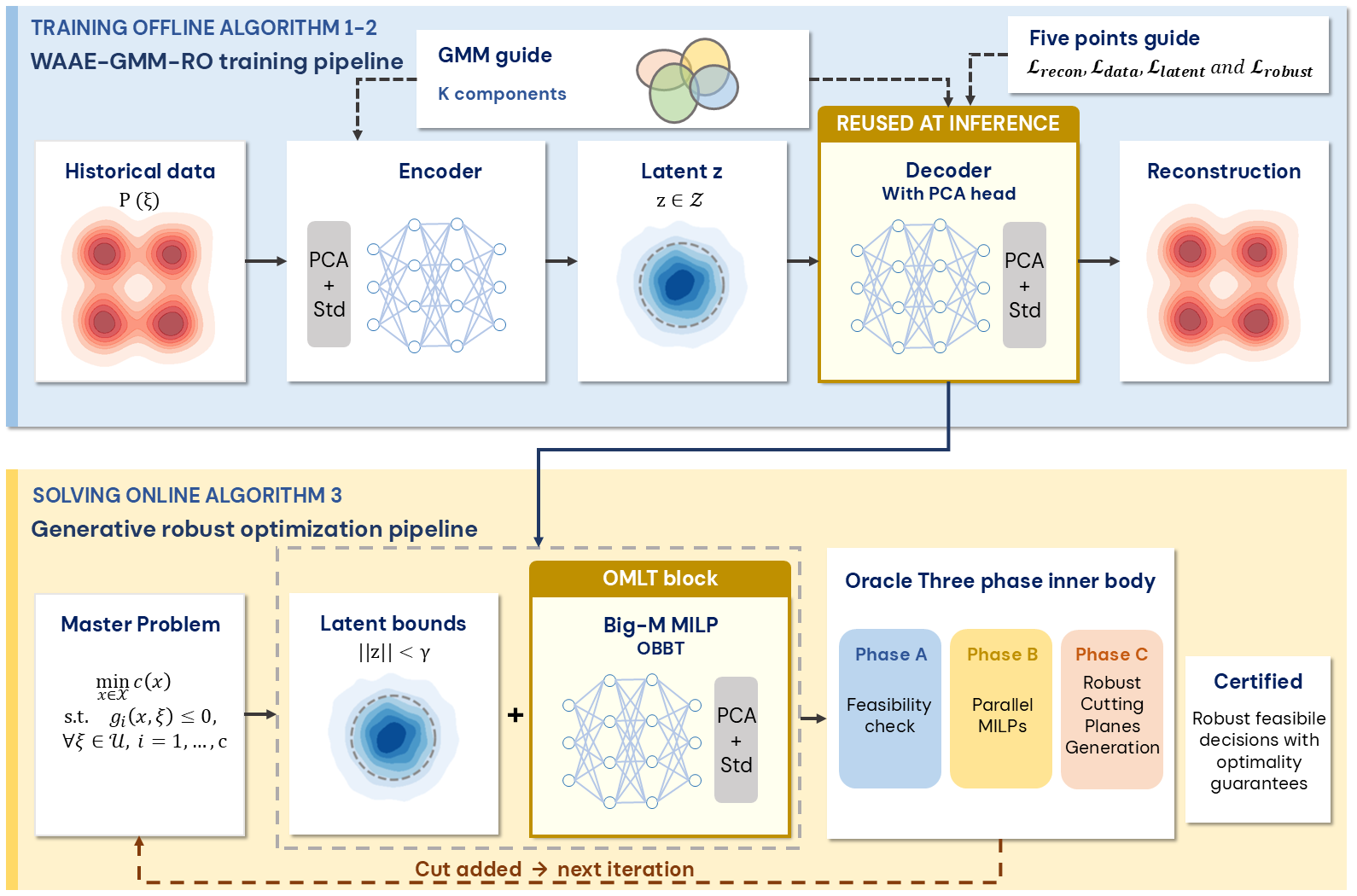}
\caption{Two-layer workflow. Top (offline): WAAE-GMM-RO training pipeline (Algorithms~\ref{alg:waae-ro-p1}--\ref{alg:waae-ro-p2}). Bottom (online): Generative robust optimization pipeline (Algorithm~\ref{alg:ro-cutting-plane}). The trained decoder is the shared component reused at inference.}
\label{fig:two-layer-workflow}
\end{figure}

\subsection{Problem Formulation}
\label{sec:problem}

Consider a decision-maker who must choose $\bm{x} \in \mathcal{X} \subseteq \R^{n_x}$ to minimise a cost function subject to constraints that depend on uncertain parameters $\bm{\xi} \in \R^{n_{\xi}}$. The robust optimisation (RO) formulation hedges against all realisations within an uncertainty set $\Uset$:
\begin{equation}\label{eq:ro}
\begin{aligned}
    \min_{\bm{x} \in \mathcal{X}} \; & c(\bm{x}) \\
    \text{s.t.} \; & g_i(\bm{x}, \bm{\xi}) \leq 0, \quad \forall\, \bm{\xi} \in \Uset,\; i = 1, \dots, m
\end{aligned}
\end{equation}
We assume that the decision-maker has access to historical observations of the uncertain parameters,
\(\mathcal{S} = \{\bm{\xi}^{(i)}\}_{i=1}^{N}\), drawn from an unknown data-generating distribution \(P^*\). The role of \(P^*\) is therefore not to provide an analytic probability model, but to name the distribution represented by the available data.

The generative construction uses an encoder--decoder pair,
\begin{align}
    \enc &: \R^{n_{\xi}} \to \R^{n_z} \label{eq:encoder}\\
    \dec &: \R^{n_z} \to \R^{n_{\xi}} \label{eq:decoder}
\end{align}
Eq.~\eqref{eq:encoder} maps a data-space scenario to latent coordinates, and Eq.~\eqref{eq:decoder} maps latent coordinates back to data space. The latent coordinates are calibrated against a reference distribution. In this paper we choose it to be the standard normal \(\N(\bm{0},\bm{I})\), because its independent, isotropic coordinates give closed-form box and ball quantiles. This choice standardises calibration in latent space.

We replace the classical uncertainty set $\Uset$ with the \emph{image of the trained decoder} over a calibrated latent set:
\begin{equation}\label{eq:gen-set}
    \Uset_{\dec} = \bigl\{\dec(\bm{z}) : \bm{z} \in \Zz \bigr\}
\end{equation}
where $\Zz \subseteq \R^{n_z}$ is a bounded calibrated latent set. Training makes \(\dec(\bm{z})\), with \(\bm{z} \sim \N(\bm{0},\bm{I})\), approximate \(P^*\), while the encoder is used to check and calibrate how historical scenarios occupy latent space.

Our goal is to produce a trained decoder $\dec$ that learns the geometry of $P^*$ from data and yields uncertainty sets that adapt to correlations, skewness, and multimodality without requiring a parametric family. The five-point evaluation framework introduced in Section~\ref{sec:five-points} establishes the conditions under which this construction is reliable.

Under this convention, using the generative set of eq.~\eqref{eq:gen-set}, the constraints of eq.~\eqref{eq:ro} can be recast as eq.~\eqref{eq:ro-gen}.
\begin{equation}\label{eq:ro-gen}
\ifinformsbuild
\begin{aligned}
    g_i\bigl(\bm{x},\, \dec(\bm{z})\bigr) &\leq 0,\\
    &\forall\, \bm{z} \in \Zz,\; i = 1, \dots, m
\end{aligned}
\else
    g_i\bigl(\bm{x},\, \dec(\bm{z})\bigr) \leq 0,
    \quad \forall\, \bm{z} \in \Zz,\; i = 1, \dots, m
\fi
\end{equation}
By reformulating eq.~\eqref{eq:ro} to eq.~\eqref{eq:ro-gen} in the latent space, the worst-case search is carried out over the latent variable $\bm{z}$, and $\Zz$ is a simple set that we construct.

\subsection{Generative Uncertainty Set Construction}
\label{sec:gen-set}

\subsubsection{Latent-Space Calibration and Oracle Bounds}
\label{sec:calibration}

After training, we encode the historical data and calibrate a latent set $\Zz$. The calibration is valid when the encoded distribution $q(\bm{z})$ is close to $\N(\bm{0}, \bm{I})$.

The latent set is a box or a ball of radius $\gamma$,
\begin{align}
    \Zz_{\text{box}} &= \bigl\{\bm{z} \in \R^{n_z} : |z_j| \leq \gamma,\; j=1,\dots,n_z\bigr\}, \label{eq:z-box}\\
    \Zz_{\text{ball}} &= \bigl\{\bm{z} \in \R^{n_z} : \|\bm{z}\|_2 \leq \gamma\bigr\} \label{eq:z-ball}
\end{align}
For the same radius the ball is inscribed in the box, so $\Zz_{\text{ball}} \subsetneq \Zz_{\text{box}}$.

The radius $\gamma$ is fixed by the calibration convention under $\N(\bm{0}, \bm{I})$. A per-coordinate (marginal) level $p \in (0,1)$ gives
\begin{equation}\label{eq:gamma-per-dim}
    \gamma_{\text{marg}} = \Phi^{-1}\!\Bigl(\frac{1+p}{2}\Bigr)
\end{equation}
covering each coordinate with probability $p$. A joint level $q \in (0,1)$ gives the box half-width $\alpha_q$ and the ball radius $\gamma_q$,
\begin{equation}\label{eq:gamma-joint}
    \alpha_q = \Phi^{-1}\!\Bigl(\frac{1 + q^{1/n_z}}{2}\Bigr), \qquad \gamma_q = \sqrt{F^{-1}_{\chi^2_{n_z}}(q)}
\end{equation}
each covering the whole set with probability $q$. The radius is thus $\gamma = \gamma_{\text{marg}}$ under marginal calibration (eq.~\eqref{eq:gamma-per-dim}), and $\gamma = \alpha_q$ (box) or $\gamma = \gamma_q$ (ball) under joint calibration (eq.~\eqref{eq:gamma-joint}). Here $\Phi$ is the standard normal CDF and $F_{\chi^2_{n_z}}$ the chi-squared CDF with $n_z$ degrees of freedom. Table~\ref{tab:radii} lists representative values.

\begin{table}[tbp]
\centering
\footnotesize
\caption{Latent calibration radii ($n_z = 6$). Joint box half-width $\alpha_q$, joint ball radius $\gamma_q$, and marginal radius $\gamma_{\text{marg}}$ at representative coverage levels $q$ / $p$.}
\label{tab:radii}
\begin{tabular}{cccc}
\toprule
$q$ / $p$ & Joint box $\alpha_q$ & Joint ball $\gamma_q$ & Marginal $\gamma_{\text{marg}}$ \\
\midrule
$50\%$ & $1.60$ & $2.31$ & $0.67$ \\
$90\%$ & $2.38$ & $3.26$ & $1.64$ \\
$99\%$ & $3.14$ & $4.10$ & $2.58$ \\
\bottomrule
\end{tabular}
\end{table}

Under marginal calibration ($\gamma = \gamma_{\text{marg}}$), the joint probability mass of the two sets is
\begin{equation}\label{eq:joint-prob}
    \mathbb{P}\bigl(\bm{z} \in \Zz_{\text{box}}\bigr) = p^{\,n_z}, \qquad
    \mathbb{P}\bigl(\bm{z} \in \Zz_{\text{ball}}\bigr) = F_{\chi^2_{n_z}}\!\bigl(\gamma_{\text{marg}}^{\,2}\bigr)
\end{equation}
below the per-coordinate level $p$.

When $q(\bm{z}) \approx \N(\bm{0},\bm{I})$, the decoded set $\Uset_{\dec}$ inherits the coverage of the latent set. This parallels the calibration schema of~\citep{Bertsimas2018}, where $\Gamma$ is set from data to achieve a desired coverage level. The difference is that our set is the image of a neural network rather than a geometric primitive, so it adapts its shape to distributions that boxes and ellipsoids cannot capture.

\begin{remark}[Nested family across coverage levels]
\label{rem:nested}
Applying different coverage levels to the \emph{same trained decoder} produces a family of nested uncertainty sets,
\[
    \Uset_{\dec}^{(0.50)} \subset \Uset_{\dec}^{(0.90)} \subset \Uset_{\dec}^{(0.99)}
\]
giving the practitioner a continuum of conservatism levels \emph{without retraining}.
\end{remark}

\subsubsection{MILP-embedding of decoded uncertainty sets}
\label{sec:decoded-set}

Let $\mathcal{S}_{\text{pre}}$ denote the data standardisation applied during the training phase (zero mean and unit variance per dimension, with any invertible preprocessing admissible, for example the PCA whitening of Section~\ref{sec:five-points}). Its inverse $\mathcal{S}_{\text{pre}}^{-1}$ maps back to raw units. The trained decoder defines the mapping $\bm{\xi} = \mathcal{S}_{\text{pre}}^{-1}(\dec(\bm{z}))$ from latent to data space. When $\bm{z}$ ranges over the whole space $\R^{n_z}$ (sampled $\bm{z} \sim \N(\bm{0},\bm{I})$), this mapping induces a generative distribution $P^{G}$ over the data, i.e. the learned approximation of $P^*$. However, at this stage its support is generally unbounded. Robust optimisation needs a \emph{bounded} set whose worst-case constraint violation is finite, so we restrict the domain to the calibrated set $\bm{z} \in \Zz$ from Section~\ref{sec:calibration}: the mapping is unchanged but its image is now bounded, and that bounded image is the uncertainty set the oracle optimises over.

Formally, $\Uset_{\dec} = \{\mathcal{S}_{\text{pre}}^{-1}(\dec(\bm{z})) : \bm{z} \in \Zz\}$ is the decoder image of the calibrated latent set, a bounded subset of the full, unbounded generative support obtained when $\bm{z}$ ranges over all of $\R^{n_z}$. It inherits its shape from $\Zz$ (marginal box, marginal ball, or joint ball).

For the oracle problem induced by the robust constraints in eq.~\eqref{eq:ro-gen},
\begin{equation}\label{eq:oracle-milp}
    v_{\text{gen}}(\bm{x}) = \max_{\bm{z} \in \Zz} \; g\bigl(\bm{x},\, \mathcal{S}_{\text{pre}}^{-1}(\dec(\bm{z}))\bigr)
\end{equation}
to be tractable as a (mixed-integer) linear program, the decoder must be \emph{MILP-representable}: composed of piecewise-linear activations whose input--output relation admits a finite disjunctive encoding. For the sake of this work we restrict ourselves to ReLU activation functions. Nonetheless, other piecewise-linear units are also admissible and the proposed framework extends seamlessly. The ReLU decoder used here has the form
\begin{equation}\label{eq:relu}
\begin{aligned}
    \dec(\bm{z}) = \bm{W}_L \cdot \sigma_{\text{ReLU}}\!\bigl( & \bm{W}_{L-1} \cdots \\
     \sigma_{\text{ReLU}}(\bm{W}_1 &\bm{z}  + \bm{b}_1) \cdots + \bm{b}_{L-1}\bigr) + \bm{b}_L
\end{aligned}
\end{equation}
where $\sigma_{\text{ReLU}}(x) = \max(0, x)$. Eq.~\eqref{eq:relu} is MILP-representable because each scalar ReLU unit is exactly encoded via the Big-M formulation~\citep{Tjeng2019}:
\begin{equation}\label{eq:bigm}
    y = \max(0, x) \iff
    \begin{cases}
        y \geq x, \quad y \geq 0, \\
        y \leq x - m(1-a), \quad y \leq Ma, \\
        a \in \{0, 1\}
    \end{cases}
\end{equation}
where $m$ and $M$ are pre-computed lower and upper bounds on $x$, and $a \in \{0,1\}$ is an activation indicator. When $\Zz$ is a polytope (marginal box, joint box), the overall oracle is a MILP. When $\Zz$ is a ball, it is a mixed-integer SOCP (convex quadratic constraint on $\bm{z}$, still tractable with modern solvers such as Gurobi~\citep{Gurobi}). To export the MIP-representable oracles we employ OMLT~\citep{Ceccon2022} to generate this encoding from an ONNX export of the trained decoder. Nevertheless, custom implementations can also be used. 

\begin{remark}[NLP export via global optimisation]
\label{rem:nlp-global}
A decoder built from smooth but non-piecewise-linear activations cannot be encoded as a MILP, yet it can be exported as a nonlinear program (NLP) and solved to global optimality. We adopt the MILP route because the exact Big-M encoding above is solved quickly in practice, whereas global NLP solution times scale poorly with decoder depth.
\end{remark}

\subsection{The Five-Point Evaluation Framework}
\label{sec:five-points}

A generative model intended for robust optimisation should satisfy requirements beyond standard generative quality metrics. We identify five criteria that are \emph{jointly necessary}: (P1) reconstruction fidelity, (P2) distribution matching in data space, (P3) latent regularity, (P4) robust relevance, and (P5) computational tractability. Throughout, we distinguish the \emph{prior} $p(\bm{z}) = \N(\bm{0},\bm{I})$ from the \emph{aggregated posterior} $q(\bm{z}) = \frac{1}{N}\sum_i \delta(\bm{z} - \enc(\bm{\xi}^{(i)}))$, the data-marginal distribution of latent codes across the training set, following the Wasserstein autoencoder framework~\citep{Tolstikhin2018}. A model that excels on distribution matching (P2) but has poor latent regularity (P3) will produce miscalibrated uncertainty sets that systematically underestimate worst-case violations. A model that satisfies (P1)--(P4) but is computationally intractable (P5) cannot be verified.

The whole pipeline operates on PCA-preprocessed data. Before training, the data is centred, scaled, and rotated by PCA, giving the concrete preprocessor $\mathcal{S}_{\text{pre}}$ used throughout this section:
\begin{equation}\label{eq:pca-scaler}
    \bm{\xi}_{\text{pca}} = V^{\top}\bigl((\bm{\xi} - \bm{\mu}) \oslash \bm{\sigma}\bigr)
\end{equation}
where $\bm{\mu}, \bm{\sigma} \in \R^{n_{\xi}}$ are the marginal training mean and standard deviation, $V \in \R^{n_{\xi} \times d_{\text{pca}}}$ has orthonormal columns (the PCA loading matrix, so $V^{\top} V = I_{d_{\text{pca}}}$), and $\oslash$ denotes elementwise division. The inverse map back to raw units is the linear head
\begin{equation}\label{eq:dec-head}
    h(\bm{y}) = \mathrm{diag}(\bm{\sigma})\, V\, \bm{y} + \bm{\mu}
\end{equation}
which is the exact inverse of eq.~\eqref{eq:pca-scaler} since $V^{\top} V = I_{d_{\text{pca}}}$. The head $h$ of eq.~\eqref{eq:dec-head} is reused as the final linear layer of the decoder.

\subsubsection{Point 1: Reconstruction Fidelity}

The encoder--decoder pair must faithfully reconstruct individual observations:
\begin{equation}\label{eq:p1}
    \loss_{\text{recon}} = \frac{1}{N}\sum_{i=1}^{N} \bigl\|\bm{\xi}^{(i)} - \dec(\enc(\bm{\xi}^{(i)}))\bigr\|^2
\end{equation}
This is essential because the latent calibration (Section~\ref{sec:calibration}) relies on the encoder mapping: if $\enc(\bm{\xi})$ does not map to a region whose decoder image is close to $\bm{\xi}$, the calibrated set $\Zz$ will not correspond to the intended data-space set. Low reconstruction error ensures the $\enc \to \Zz \to \dec$ pipeline is geometrically faithful. This is a per-sample fidelity requirement on training observations. (P2) and (P3) below prevent overfitting by matching the \emph{distribution} of generated samples to $P^*$.

Consider a scenario $\bm{\xi}$ that causes a constraint violation in the downstream optimisation. If $\dec(\enc(\bm{\xi})) \neq \bm{\xi}$, then the constraint value $g(\bm{x}, \dec(\enc(\bm{\xi})))$ may differ from the true value $g(\bm{x}, \bm{\xi})$, meaning the model may not identify this scenario as infeasibility-inducing during the worst-case search of eq.~\eqref{eq:oracle-milp}.

\subsubsection{Point 2: Distribution Matching in Data Space}

Generated samples $\hat{\bm{\xi}} = \dec(\bm{z})$, $\bm{z} \sim \N(\bm{0}, \bm{I})$, must approximate the data distribution $P^*$. We match it with an adversarial Wasserstein critic and a maximum mean discrepancy, and optionally a moment-matching term on the leading marginal statistics.

The first mechanism is adversarial. We train two WGAN-GP critics in parallel, one in PCA-standardised coordinates and one in raw standardised coordinates. The overall data-critic loss is the sum of the two, as shown by eq.~\eqref{eq:p2-critic}.
\begin{equation}\label{eq:p2-critic}
    \loss_{\text{critic}} = \loss_{\text{critic}}^{\text{pca}} + \loss_{\text{critic}}^{\text{std}}
\end{equation}
where each term has the standard WGAN-GP form
\begin{equation}\label{eq:p2-critic-single}
\begin{aligned}
    \loss_{\text{critic}}^{\star} = {} & \E_{\hat{\bm{\xi}}^{\star}}[\disc^{\star}(\hat{\bm{\xi}}^{\star})] - \E_{\bm{\xi}^{\star}}[\disc^{\star}(\bm{\xi}^{\star})] \\
    & + \lambda_{\text{gp}}\, \E_{\tilde{\bm{\xi}}^{\star}}\!\bigl[(\|\nabla_{\tilde{\bm{\xi}}^{\star}}\disc^{\star}(\tilde{\bm{\xi}}^{\star})\|_2 - 1)^2\bigr]
\end{aligned}
\end{equation}
In eq.~\eqref{eq:p2-critic-single}, the index $\star \in \{\text{pca}, \text{std}\}$ selects the space, where $\bm{\xi}^{\text{pca}}$ are samples in PCA coordinates obtained via eq.~\eqref{eq:pca-scaler} and $\bm{\xi}^{\text{std}} = (\bm{\xi} - \bm{\mu}) \oslash \bm{\sigma}$ are samples in raw standardised coordinates. The coefficient $\lambda_{\text{gp}}$ weights the gradient penalty and is held fixed at the standard value $\lambda_{\text{gp}} = 10$ throughout training~\citep{Gulrajani2017}. The point $\tilde{\bm{\xi}}^{\star} = \alpha \bm{\xi}^{\star} + (1-\alpha)\hat{\bm{\xi}}^{\star}$, with mixing weight $\alpha \sim U[0,1]$, is a random interpolation on the segment joining a real sample and a generated sample, and the penalty term pushes the norm of the critic gradient at these interpolated points towards $1$. The gradient penalty replaces weight clipping by soft-constraining each critic to be $1$-Lipschitz along these interpolated samples, which stabilises training and prevents mode collapse~\citep{Gulrajani2017}.

The adversarial critic gives a strong distributional signal but a noisy gradient, because it depends on a separately optimised network. We therefore add a maximum mean discrepancy (MMD), a kernel two-sample statistic that compares the real and generated samples directly, with no auxiliary network, and that vanishes exactly when the two distributions coincide for a characteristic kernel. In empirical form it is the sum of the within-sample and between-sample kernel averages, as shown by eq.~\eqref{eq:p2-mmd}.
\begin{equation}\label{eq:p2-mmd}
\begin{aligned}
    \loss_{\text{data-MMD}} &= \MMD^2_k\!\bigl(\{\bm{\xi}^{(i)}\}, \{\hat{\bm{\xi}}^{(j)}\}\bigr) \\
    &= \tfrac{1}{n^2}\sum_{i,i'} k(\bm{\xi}^{(i)}, \bm{\xi}^{(i')})
    + \tfrac{1}{n^2}\sum_{j,j'} k(\hat{\bm{\xi}}^{(j)}, \hat{\bm{\xi}}^{(j')}) \\
    &\quad - \tfrac{2}{n^2}\sum_{i,j} k(\bm{\xi}^{(i)}, \hat{\bm{\xi}}^{(j)})
\end{aligned}
\end{equation}
where $\{\bm{\xi}^{(i)}\}$ and $\{\hat{\bm{\xi}}^{(j)}\}$ are the $n$ real and $n$ generated samples in a minibatch and $k$ is an inverse multiquadric (IMQ) or RBF kernel. Both are characteristic kernels, making the MMD a proper distance between distributions~\citep{Tolstikhin2018}. We default to the IMQ for its heavier tails, which preserve a usable gradient when the two empirical distributions are far apart.

We optionally add a moment-matching term on the first four marginal moments between real and generated samples, as shown by eq.~\eqref{eq:p2-mom}.
\begin{equation}\label{eq:p2-mom}
\begin{aligned}
    \loss_{\text{data-mom}} = {} & \|\bm{\mu}_{\bm{\xi}} - \bm{\mu}_{\hat{\bm{\xi}}}\|^2 + \|\bm{\sigma}^2_{\bm{\xi}} - \bm{\sigma}^2_{\hat{\bm{\xi}}}\|^2 \\
    & + \|\bm{s}_{\bm{\xi}} - \bm{s}_{\hat{\bm{\xi}}}\|^2 + \|\bm{\kappa}_{\bm{\xi}} - \bm{\kappa}_{\hat{\bm{\xi}}}\|^2
\end{aligned}
\end{equation}
where $\bm{\mu}$, $\bm{\sigma}^2$, $\bm{s}$, and $\bm{\kappa}$ are the marginal mean, variance, skewness, and kurtosis vectors, respectively.

\subsubsection{Point 3: Latent Regularity}
\label{sec:p3}

The calibration of $\Zz$ (Section~\ref{sec:calibration}) assumes the encoded data distribution $q(\bm{z})$ is close to the prior $p(\bm{z}) = \N(\bm{0}, \bm{I})$. This assumption is critical: if it is violated, the quantile-based bounds in eqs.~\eqref{eq:z-box}--\eqref{eq:z-ball} do not have their intended coverage, and the uncertainty set $\Uset_{\dec}$ is miscalibrated.

It is important to distinguish the \emph{per-sample} posterior $q(\bm{z}|\bm{\xi})$ from the \emph{aggregated} posterior $q(\bm{z})$ introduced above. The per-sample KL divergence $\mathrm{KL}(q(\bm{z}|\bm{\xi}^{(i)}) \| p(\bm{z}))$, used in variational autoencoders, regularises each individual encoding to be close to the prior $\N(\bm{0},\bm{I})$.

However, even when every per-sample KL is small, the aggregated posterior can deviate substantially from $\N(\bm{0},\bm{I})$, so the quantile-based bounds in eqs.~\eqref{eq:z-box}--\eqref{eq:z-ball} underestimate or overestimate along the directions of deviation.

We regularise the aggregated posterior towards $\N(\bm{0}, \bm{I})$ with two losses of the same form as their data-space counterparts (eqs.~\eqref{eq:p2-mmd}--\eqref{eq:p2-mom}): a latent MMD loss, as defined by eq.~\eqref{eq:p3-mmd},
\begin{equation}\label{eq:p3-mmd}
    \loss_{\text{lat-MMD}} = \MMD^2_k\bigl( \{\enc(\bm{\xi}^{(i)})\},\, \{\bm{z}_j\}\bigr), \quad \bm{z}_j \sim \N(\bm{0}, \bm{I})
\end{equation}
and a latent moment-matching loss, as defined by eq.~\eqref{eq:p3-mom}.
\begin{equation}\label{eq:p3-mom}
\begin{aligned}
    \loss_{\text{lat-mom}} = \tfrac{1}{n_z}\bigl(\,
        & \|\bm{\mu}_z\|^2 + \|\diag(\bm{\Sigma}_z) - \bm{1}\|^2 \\
        & {} + \|\bm{s}_z\|^2 + \|\bm{\kappa}_z - 3 \cdot \bm{1}\|^2 \,\bigr)
\end{aligned}
\end{equation}
A stronger alternative that decomposes latent matching into $K$ per-component sub-conditions via a data-space Gaussian mixture is presented in Section~\ref{sec:waae-gmm}.

To evaluate the latent regularity, we employ marginal normality tests. Specifically, we report:
\begin{itemize}[nosep]
    \item $n_{\text{normal}}$: number of latent dimensions passing the Anderson-Darling and Shapiro-Wilk normality tests at $\alpha = 0.01$, a strict significance level chosen so that $n_{\text{normal}}$ only counts dimensions both tests confidently accept as Gaussian
    \item Eigenvalue spread: $\lambda_{\max}/\lambda_{\min}$ of $\mathrm{Cov}(\bm{z})$ (target: $1.0$, attained when the latent covariance is isotropic so $\mathrm{Cov}(\bm{z}) = \bm{I}$ matches the prior)
    \item Latent MMD: $\MMD(\{\enc(\bm{\xi}^{(i)})\}, \N(\bm{0},\bm{I}))$.
\end{itemize}
A model with $n_{\text{normal}} = n_z$ (every latent marginal passes both normality tests) and eigenvalue spread near $1$ (covariance approximately isotropic) is jointly close to $\N(\bm{0}, \bm{I})$, so the quantile-based calibration of Section~\ref{sec:calibration} is statistically meaningful.

\subsubsection{Point 4: Robust Relevance}
\label{sec:p4-intro}

(P1)--(P3) ensure the generative model produces high-quality samples and a well-calibrated latent space. In theory, a model with perfect (P1)--(P3) would capture the full distribution including tails, making (P4) unnecessary. In practice, however, finite data and the limits of a given architecture mean that few generative models simultaneously shape the aggregated latent into a true $\N(\bm{0},\bm{I})$ and reconstruct the data space faithfully, so (P1)--(P3) alone may not capture the specific scenarios that govern worst-case feasibility in the downstream optimisation. For a decision $\bm{x}$, define the worst-case violation:
\begin{align}
    v_{\text{gen}}(\bm{x}) &= \max_{\bm{z} \in \Zz}\; g\bigl(\bm{x}, \dec(\bm{z})\bigr),
    \label{eq:v-gen} \\
    v_{\text{ref}}(\bm{x}) &= \max_{n = 1,\dots,N}\; g(\bm{x}, \bm{\xi}^{(n)})
    \label{eq:v-ref}
\end{align}
The gap between the decoded worst-case of eq.~\eqref{eq:v-gen} and the reference worst-case of eq.~\eqref{eq:v-ref}, expressed as a pool-aggregated relative error, quantifies the discrepancy:
\begin{equation}\label{eq:delta-wc}
    \Delta_{\text{wc}} = \frac{\sum_{k=1}^{K_{\text{pool}}} \bigl|\,v_{\text{gen}}(\bm{x}_k) - v_{\text{ref}}(\bm{x}_k)\,\bigr|}{\sum_{k=1}^{K_{\text{pool}}} \bigl|\,v_{\text{ref}}(\bm{x}_k)\,\bigr|}
\end{equation}
Eq.~\eqref{eq:delta-wc} measures the pool-averaged relative error between generated and reference worst-case violations, a dimensionless ratio of total discrepancy over total reference violation magnitude.

In combination with the \emph{conservatism ratio}, abbreviated \emph{conserv} and defined as $\bar{v}_{\text{gen}}/\bar{v}_{\text{ref}}$ (bars denote averages over the $K_{\text{pool}}$ pool decisions), $\Delta_{\text{wc}}$ characterises the magnitude of the gap once the decision is on the safe side. The conservatism ratio is what determines that safe side: \emph{conserv} $\ge 1$ means the decoded worst-case meets or exceeds the reference, so the resulting decision remains robust-feasible. A value of \emph{conserv} $<1$ is the unsafe case, since the decoder underestimates the true worst case and may certify infeasible decisions as feasible.

(P4) addresses this gap by directly measuring and training for constraint-violation accuracy in the distributional regions that govern worst-case violations. These regions may correspond to tails in specific dimensions rather than the joint tail of all dimensions, depending on the structure of the constraint function $g$. Reducing $\Delta_{\text{wc}}$ towards zero while keeping \emph{conserv} $\ge 1$ (the safe regime for RO) is the primary objective of (P4) training.

(P4) is problem-driven: it requires knowledge of the constraint function $g$ and a pool of candidate decisions. We detail the (P4) training procedure in Section~\ref{sec:waae-ro}.

\subsubsection{Point 5: Computational Tractability}
\label{sec:p5}

Even with perfect (P1)--(P4) scores, a generative model is not useful for robust optimisation unless its decoder can be embedded in the worst-case problem~\eqref{eq:oracle-milp}. (P5) requires:
\begin{itemize}[nosep]
    \item The decoder $\dec$ can be encoded \emph{exactly} as algebraic constraints, not merely approximated.
    \item The resulting formulation is a tractable class (MILP), rather than a mixed-integer nonlinear program (MINLP).
    \item The model size yields practical solution times (number of binary variables, LP relaxation tightness).
\end{itemize}

(P5) concerns the existence of an \emph{exact}, solver-friendly embedding, not worst-case complexity. A mixed-integer linear program is NP-hard in general, but for the decoder sizes used here a branch-and-bound solver verifies the worst-case oracle to global optimality in practical time. An MINLP or a non-exact (approximate) encoding provides neither this exact verification nor reliable global solves, so a model that forces one fails (P5) even when it scores well on (P1)--(P4).

\subsubsection{Cross-Model Assessment}

Table~\ref{tab:five-point} evaluates representative generative model families against the five-point framework. Each family fails on at least one point, which motivates the Wasserstein Adversarial Autoencoder (WAAE) design (Section~\ref{sec:waae}) that targets all five simultaneously. Each row highlights the model's primary failure mode, and a dash in the (P4) column indicates no built-in (P4) mechanism. The MIP embedding column gives the number of binary variables and the MIP type (MILP or MINLP) of each decoder's embedding. NF and DDIM attain strong distribution matching (P2) through their multi-step structure, which is also what makes (P5) intractable.

\begin{table}[tbp]
\centering
\caption{Qualitative five-point assessment. In the MIP-embedding column, the integer is the number of binary variables in the model. VAE = variational autoencoder, GAN = generative adversarial network, NF = normalising flow (RealNVP), DDIM = denoising diffusion implicit model.}
\label{tab:five-point}
\scriptsize
\begin{tabular}{lccccccc}
\toprule
\textbf{Model} & \textbf{(P1)} & \textbf{(P2)} & \textbf{(P3)} & \textbf{(P4)} & \textbf{(P5)} & \textbf{MIP embedding} & \textbf{Primary failure} \\
\midrule
VAE (KL) & \checkmark & $\sim$ & $\sim$ & --- & \checkmark & 128 (MILP) &
    (P3): KL $\not\Rightarrow q(\bm{z}) \approx \N$ \\[2pt]
GAN & --- & \checkmark & $\times$ & --- & \checkmark & 128 (MILP) &
    (P3): no encoder \\[2pt]
NF (RealNVP) & \checkmark & \checkmark & \checkmark & --- & $\times$ & 256/layer (MINLP) &
    (P5): multistep, MINLP \\[2pt]
DDIM & $\sim$ & \checkmark & $\sim$ & --- & $\times$ & 128/step (MILP) &
    (P5): multistep \\[2pt]
\midrule
WAAE-RO (ours) & \checkmark & \checkmark & \checkmark & \checkmark & \checkmark & 128 (MILP) &
    Satisfies all five \\[2pt]
WAAE-GMM-RO (ours) & \checkmark & \checkmark & \checkmark & \checkmark & \checkmark & 128 (MILP) &
    GMM-guided (P3) \\
\bottomrule
\end{tabular}
\end{table}

We identify the operational requirements for which the five-point framework is necessary. They come from the robust counterpart of Section~\ref{sec:problem}, stated independently of the five points.

\begin{definition}[Deployable generative uncertainty set]
\label{def:deployable}
Let $\Uset_{\dec} = \dec(\Zz)$ be the decoded uncertainty set, with encoder $\enc$, aggregated posterior $q(\bm{z})$, calibrated latent set $\Zz$ at marginal level $p$, and decision pool $\{\bm{x}_k\}_{k=1}^{K_{\text{pool}}}$. The set $\Uset_{\dec}$ is \emph{deployable at level $p$} if:
\begin{itemize}[nosep]
    \item[\textbf{(I)}] \emph{Coverage.} On each latent dimension $j$ the encoded data attains marginal coverage at least $p$ of the calibrated interval, and the canonical encoder--decoder path is faithful: with $\hat{\bm{\xi}} = \dec(\enc(\bm{\xi}))$, a $p$-fraction of $P^{*}$ on each marginal is both captured by $\Zz$ through $\enc$ and returned by $\dec$ to within reconstruction tolerance.
    \item[\textbf{(II)}] \emph{Tractability.} The worst-case oracle $\max_{\bm{z} \in \Zz}\, g(\bm{x}, \dec(\bm{z}))$ is an exact MILP (box latent) or MISOCP (ball latent), solvable per query within the regime reported in Section~\ref{sec:results}.
    \item[\textbf{(III)}] \emph{Robust safety.} The decoder oracle does not underestimate the empirical worst case on the pool: $v_{\text{gen}}(\bm{x}_k) \ge v_{\text{ref}}(\bm{x}_k)$ for every $\bm{x}_k$.
\end{itemize}
\end{definition}

The necessity result holds under three standing assumptions.

\begin{assumption}[Canonical-path coverage]
\label{ass:canonical}
A scenario $\bm{\xi}$ counts as covered by $\Uset_{\dec}$ through the single encoder map $\bm{z} = \enc(\bm{\xi})$, the same map used to calibrate $\Zz$ in Section~\ref{sec:calibration}.
\end{assumption}

\begin{assumption}[Marginal-coverage convention]
\label{ass:marginal}
Coverage in (I) is marginal, matching the calibration radius $\gamma_{\text{marg}} = \Phi^{-1}((1+p)/2)$ of eq.~\eqref{eq:gamma-per-dim}.
\end{assumption}

\begin{assumption}[Violation-exposed pool]
\label{ass:pool}
The decision pool is sampled near the robust boundary, so $v_{\text{ref}}(\bm{x}_k) > 0$ for each $\bm{x}_k$, and the safety comparison in (III) is then well-posed.
\end{assumption}

Each assumption plays a role in the necessity result below. Assumption~\ref{ass:canonical} makes coverage an encoder test, so the argument never needs to invert the decoder. Assumption~\ref{ass:marginal} reads requirement~(I) as marginal rather than joint: even with a perfectly Gaussian aggregated posterior $q(\bm{z}) = \N(\bm{0}, \bm{I})$ the calibrated box holds only the joint mass $p^{n_z} < p$ (eq.~\eqref{eq:joint-prob}), so a joint-$p$ reading would be unattainable by construction.

\begin{proposition}[Necessity of (P1)--(P5)]
\label{prop:necessity}
Under Assumptions~\ref{ass:canonical}--\ref{ass:pool}, the decoded set $\Uset_{\dec} = \dec(\Zz)$ is deployable at level $p$ on the pool only if the underlying model $(\enc, \dec)$ satisfies all five points (P1)--(P5). Points (P1) and (P2) are each necessary for (I) Coverage, as is (P3) against variance-inflation deviations of the latent law. (P4) is necessary for (III) Robust safety, and (P5) is necessary for (II) Tractability.
\end{proposition}

The proof (Appendix~\ref{ec:proof-necessity}) argues five contrapositives. (P1) failure breaks the canonical encoder--decoder coverage path. (P2) failure opens a total-variation gap between $P^{*}$ and $\dec\#\N$ that removes the level-$p$ coverage certificate (rather than forcing a realised coverage drop). (P3) failure inherits the calibration gap of Theorem~\ref{thm:cal-gap} in the variance-inflation regime. (P4) failure produces a pool decision whose decoder oracle underestimates the empirical worst case. (P5) failure puts the oracle outside the exact-MILP class (smooth activations, Lemma~\ref{lem:activation}) or outside the per-query regime (depth scaling).

\begin{theorem}[Calibration gap]
\label{thm:cal-gap}
Let $j \in \{1, \ldots, n_z\}$ and $p \in (0, 1)$. Let $q_j$ denote the marginal of the aggregated posterior on dimension $j$, with CDF $F_j$, and let $\gamma_{\text{marg}} = \Phi^{-1}((1 + p)/2)$ be the marginal-mode calibration radius. The realised marginal coverage $C_j(p) = \Pr_{z_j \sim q_j}(|z_j| \le \gamma_{\text{marg}}) = F_j(\gamma_{\text{marg}}) - F_j(-\gamma_{\text{marg}})$ satisfies
\begin{equation}\label{eq:cal-gap}
    |C_j(p) - p| \;\le\; 2\, d_{\mathrm{KS}}(q_j,\,\N(0,1))
\end{equation}
where $d_{\mathrm{KS}}$ is the Kolmogorov--Smirnov distance. When $q_j = \N(0, 1)$, $C_j(p) = p$ for every $p$.
\end{theorem}

The proof (Appendix~\ref{ec:proof-calgap}) follows the conformal-prediction calibration setup~\citep{Angelopoulos2023}. The bound is marginal and tight up to a factor of two. The converse $C_j(p) = p \Rightarrow q_j = \N(0,1)$ does not hold, since even-perturbation deviations preserve $C_j(p) = p$ at every $p$. Empirically, $\Delta_{\text{wc}}$ increases when (P3) is not satisfied (Table~\ref{tab:full-comparison}).

\begin{lemma}[Activation requirement]
\label{lem:activation}
Let $f : \R^n \to \R^m$ be the function computed by a feedforward neural network. Call a real-analytic activation a \emph{non-degenerate component} of $f$ when its nonlinearity affects $f$ on a set of positive Lebesgue measure. An exact, finite mixed-integer linear (MILP) encoding of the input--output graph of $f$ exists only if $f$ is piecewise-affine on a finite polyhedral subdivision of $\R^n$. Hence, if any non-degenerate component is a real-analytic non-affine activation ($\tanh$, $\sigma$, $\exp$, $\sin$, $\cos$), no exact MILP encoding exists, and an exact mixed-integer encoding must instead be a mixed-integer nonlinear program (MINLP).
\end{lemma}

The proof (Appendix~\ref{ec:proof-activation}) builds on the mixed-integer-formulation framework of \citet[Definition~1]{Anderson2020}. Piecewise-affine activations such as ReLU admit exact Big-M encodings~\citep[Eq.~6]{Tjeng2019}.
Empirical support is given by the MILP-size column of Tables~\ref{tab:five-point} and~\ref{tab:full-comparison}: normalising flows with $\exp\circ\tanh$ coupling layers produce a nonconvex MINLP rather than a MILP, and DDIM-style diffusion models unroll their denoiser $T$ times so that the binary-variable count grows linearly in $T$ (e.g.\ ${\sim}6400$ binaries at $T = 50$ versus $128$ for a single-pass ReLU decoder).

This intractability is a structural property of eq.~\eqref{eq:oracle-milp}, not a solver limitation, and it compounds when multiple uncertain constraints are present: each requires its own MILP oracle per cutting-plane iteration, so the per-iteration cost scales linearly in the number of uncertain constraints and can dominate even when any single oracle is fast.

Figure~\ref{fig:latent-data-grid} shows the five-point assessment on the six PP distributions (rows) across all five model families (columns). Both the data and latent spaces are six-dimensional ($n_{\xi} = n_z = 6$). Each sub-panel shows a 2D projection onto $(\xi_1, \xi_2)$ for data sub-panels and $(z_1, z_2)$ for latent sub-panels. Each cell pairs a latent-space sub-panel with a data-space sub-panel. Four concentric marginal quantile levels $p \in \{0.3, 0.5, 0.7, 0.9\}$ ($\gamma_{\text{marg}} \in \{0.39, 0.67, 1.04, 1.64\}$) are drawn as same-coloured rings in the latent sub-panel. The corresponding data-space contour is the image of that ring under the decoder $\dec$.

(P3) failure is visible on the VAE for the independent mixture case, where the latent is multimodal rather than $\mathcal{N}(\bm{0}, \bm{I})$. The other baselines show latent or data-space distortion across the six rows. WAAE-GMM-RO maintains a stable latent space and an accurate data space.

\begin{figure}[tbp]
\centering
\includegraphics[width=\textwidth]{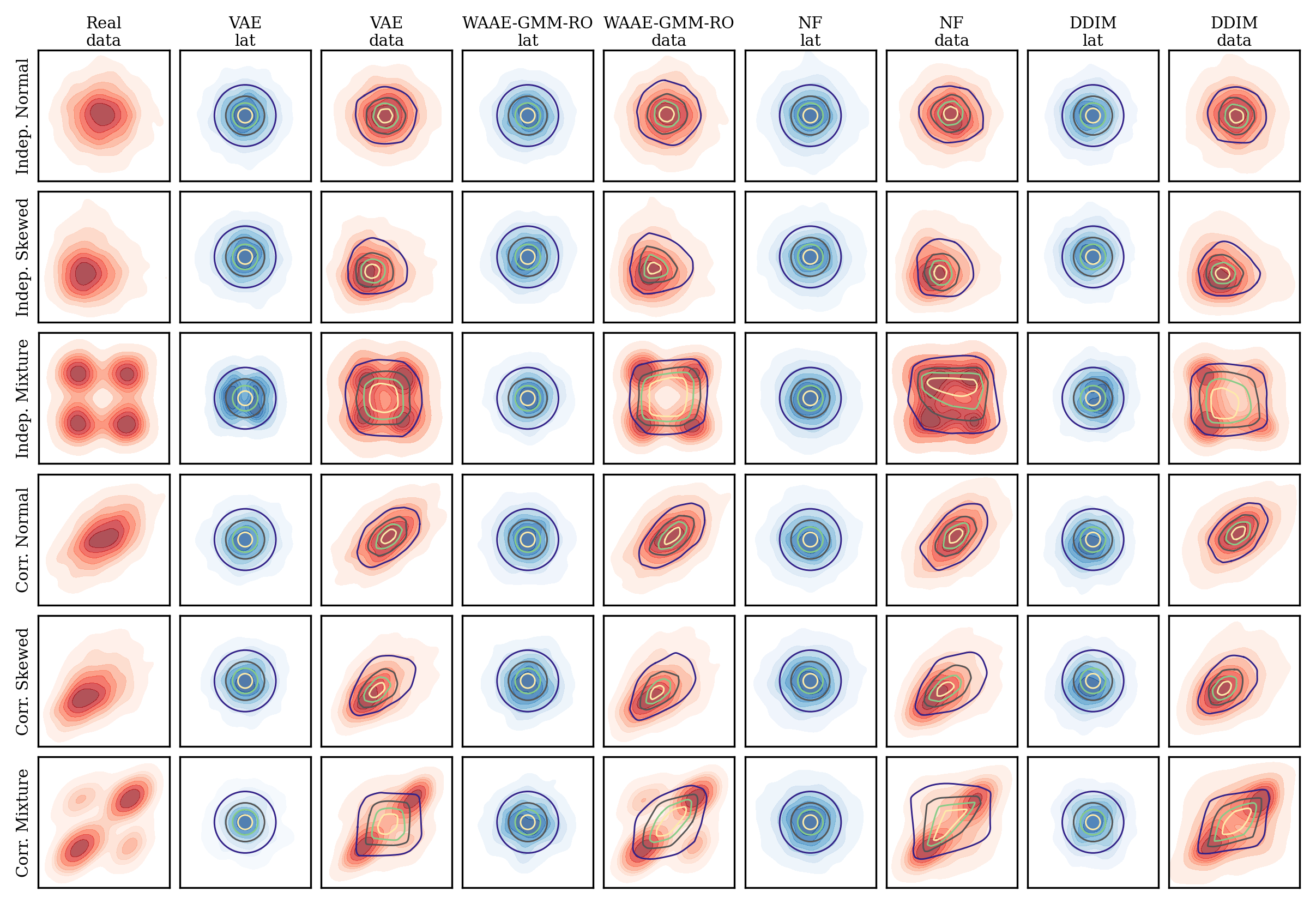}
\caption{Latent and data-space distribution fits. The four trained generative models against the ground-truth distribution $P^{*}$ across the six PP uncertainty distributions (rows). Rings at the marginal-quantile levels $p \in \{0.3,\, 0.5,\, 0.7,\, 0.9\}$ trace through each model's decoder.}
\label{fig:latent-data-grid}
\end{figure}

\subsection{WAAE: Wasserstein Adversarial Autoencoder}
\label{sec:waae}

The WAAE backbone is configured so that each component addresses one or more of the five evaluation points. It is not a novel generative architecture per se.

Although only the decoder $\dec$ appears in the MILP oracle at optimisation time, the architecture includes a jointly trained encoder $\enc$. A decoder-only generator (e.g.\ GAN~\citep{Goodfellow2014}) cannot guarantee that the latent geometry reflects the data geometry, which is what the calibration of Section~\ref{sec:calibration} requires: if $\Zz_q$ does not correspond to a meaningful data-space region, the decoded set $\Uset_{\dec}$ cannot be calibrated to a desired coverage level. The encoder supplies the data-anchored inverse mapping that resolves this, provided Points~1 and~3 both hold: $\dec(\enc(\bm{\xi})) \approx \bm{\xi}$ (P1) ensures local consistency between latent and data space, and $\enc$ mapping $P^*$ close to $\N(\bm{0},\bm{I})$ (P3) ensures that a $q$-quantile bound computed from encoded training data corresponds to a $q$-quantile region of $P^*$. The calibrated $\Zz_q$ is therefore \emph{data-supported}: its boundaries come from actual encoded observations, not an assumed prior. Encoder--decoder pairings of this kind trace back to VAEs~\citep{KingmaWelling2014}, adversarial autoencoders~\citep{Makhzani2015}, and Wasserstein autoencoders~\citep{Tolstikhin2018}. Our WAAE reuses their structure and adds the ReLU-only, PCA-warm-started decoder required for Points~4 and~5.

\subsubsection{Architecture}

The WAAE consists of four neural networks, all multilayer perceptrons (MLPs):
\begin{itemize}[nosep]
    \item \textbf{Encoder} $\enc: \R^{n_{\xi}} \to \R^{n_z}$, deterministic, ReLU activations.
    \item \textbf{Decoder} $\dec: \R^{n_z} \to \R^{n_{\xi}}$, composed of a ReLU MLP body $b: \R^{n_z} \to \R^{d_{\text{pca}}}$ (with optional skip connection) followed by a linear head $h: \R^{d_{\text{pca}}} \to \R^{n_{\xi}}$ (see PCA warm-start below), so $\dec = h \circ b$.
    \item \textbf{PCA-space critic} $\disc^{\text{pca}}: \R^{d_{\text{pca}}} \to \R$, LeakyReLU, WGAN-GP critic operating on PCA-standardised samples.
    \item \textbf{Std-space critic} $\disc^{\text{std}}: \R^{n_{\xi}} \to \R$, LeakyReLU, WGAN-GP critic operating on standardised raw samples.
\end{itemize}

When $n_z = d_{\text{pca}}$, the decoder body includes a residual (skip) connection,
\begin{equation}\label{eq:skip}
    b(\bm{z}) = \text{NN}(\bm{z}) + \bm{z}
\end{equation}
Since the encoder is trained to produce $\enc(\bm{\xi}) \approx \bm{z}$ close to the standardised data (P3), the identity mapping $\bm{z} \mapsto \bm{z}$ is already a reasonable reconstruction, and the network $\text{NN}(\bm{z})$ learns only the \emph{correction}. In the MILP formulation, the skip adds only $d_{\text{pca}}$ linear constraints.

The decoder also carries a PCA warm-start. The body $b$ produces outputs in the PCA-whitened coordinate system of eq.~\eqref{eq:pca-scaler}, and the linear head $h$ of eq.~\eqref{eq:dec-head} maps them back to raw units. The head's parameters $\{V, \bm{\sigma}, \bm{\mu}\}$ are initialised from the fitted PCA scaler so that at $t = 0$, $h(b(\enc(\bm{\xi}_{\text{pca}}))) = \bm{\xi}$ whenever the body reconstructs the identity, making the warm-start reconstruction-trivial.

The head serves two roles: (i) it lets the body train in PCA-whitened coordinates, so the loss scales are balanced across coordinates, and (ii) it gives the final linear layer a principled initialisation. The head is held fixed during (P1)--(P3) training.

\subsubsection{Training Objective}

The total loss decomposes as:
\begin{equation}\label{eq:total-loss}
    \loss = \lambda_{\text{recon}} \loss_{\text{recon}}
    + \lambda_{\text{data}} \loss_{\text{data}}
    + \lambda_{\text{latent}} \loss_{\text{latent}}
\end{equation}
where $\loss_{\text{recon}}$ targets (P1) via eq.~\eqref{eq:p1}, $\loss_{\text{data}}$ targets (P2) via eqs.~\eqref{eq:p2-critic}--\eqref{eq:p2-mom}, and $\loss_{\text{latent}}$ targets (P3) via eqs.~\eqref{eq:p3-mmd}--\eqref{eq:p3-mom}.

The WAAE addresses (P1)--(P3) and (P5) (ReLU decoder) by construction. The complete training procedure, including (P4) extensions, is given in Algorithms~\ref{alg:waae-ro-p1} and~\ref{alg:waae-ro-p2}. (P4) requires the extension described in Section~\ref{sec:waae-ro}.

\subsection{GMM-Guided Training (WAAE-GMM)}
\label{sec:waae-gmm}
The WAAE presented in Section~\ref{sec:waae} satisfies (P1)--(P3) empirically through MMD and moment penalties. However, these penalties only penalise deviations in the aggregated posterior: they do not decompose the matching problem into verifiable sub-conditions, so the encoder receives no structural guidance about the data distribution and must learn to transport a potentially multimodal distribution to $\N(\bm{0},\bm{I})$ in a single pass. Combined with the \emph{hole problem} of latent variable models with a fixed unimodal prior~\citep{Rezende2018Taming}, where $q(\bm{z})$ may not cover all regions of high prior density (so samples from those regions decode to unrealistic outputs), this motivates the GMM-guided approach below.

\subsubsection{Data-Space Gaussian Mixture}

By the universal approximation property of Gaussian mixtures, any continuous data distribution can be written as
\begin{equation}\label{eq:gmm-data}
  p(\bm{\xi}) = \sum_{k=1}^{K} \pi_k\,
    \N\!\bigl(\bm{\xi}\mid\bm{\mu}_k,\,\diag(\bm{\sigma}_k^{2})\bigr)
\end{equation}
for sufficiently large $K$. This mixture (eq.~\eqref{eq:gmm-data}) is explicitly modelled in \emph{data space} with trainable parameters $\{\bm{\sigma}_k, \pi_k\}_{k=1}^K$ and component means $\{\bm{\mu}_k\}$ held fixed at their warm-start values. Freezing the means is necessary here, because without it the components drift during adversarial training and can collapse onto a single mode~\citep{Lu2025GMM}. The encoder is trained to transport the $K$ data-space components into $K$ latent sub-distributions whose $\pi$-weighted moments aggregate to those of $\N(\bm{0},\bm{I})$ on every dimension, an explicit condition that is given in eq.~\eqref{eq:gmm-moments} below.

The GMM is a \emph{training-time guide only}. At inference, generation proceeds as $\bm{z}\sim\N(\bm{0},\bm{I}),\;\bm{\xi}=\dec(\bm{z})$. The latent prior remains $\N(\bm{0},\bm{I})$, so calibration and the ball constraint $\|\bm{z}\|_2\le\gamma$ are unchanged, and the decoder MILP is identical to that of the standard WAAE.

This contrasts with the latent-prior mixtures used by VaDE~\citep{Jiang2017VaDE}, VampPrior~\citep{Tomczak2018VampPrior}, GMVAE~\citep{Dilokthanakul2016GMVAE}, and AAE+GMM~\citep{Makhzani2015}, which require two-stage sampling at inference, $k\sim\mathrm{Cat}(\bm{\pi})$ followed by $\bm{z}\sim p_k$: the latent prior is no longer $\N(\bm{0},\bm{I})$, so calibration of $\Zz_q$ breaks (P3), and MILP embedding requires one-hot binaries to multiplex between the $K$ component branches (P5).

\subsubsection{Training Losses}\label{sec:waae-gmm-losses}

The GMM parameters $\{\bm{\sigma}_k, \pi_k\}_{k=1}^K$ are fitted to the data distribution by minimising the negative log-likelihood:
\begin{equation}\label{eq:gmm-nll}
  \loss_{\text{GMM}} = -\frac{1}{N}\sum_{n=1}^{N}\log \sum_{k=1}^{K}
  \pi_k\,\N\bigl(\bm{\xi}^{(n)}\mid\bm{\mu}_k,\,\diag(\bm{\sigma}_k^{2})\bigr)
\end{equation}
Eq.~\eqref{eq:gmm-nll} is computed using the log-sum-exp trick for numerical stability. Gradients flow only to $\{\bm{\sigma}_k, \pi_k\}$ (the input $\bm{\xi}^{(n)}$ is real data, not encoder output). The GMM is warm-started before adversarial training, to avoid component collapse, by either expectation--maximisation with fixed $K$ or variational inference with a Dirichlet-process prior that auto-prunes redundant components. On correlated mixture data the variational variant typically prunes $K{=}12$ down to $\approx 9$ effective components.

We next construct the component moment constraint. Given a data point $\bm{\xi}$ and the current GMM parameters, the \emph{responsibility} of component $k$ \citep{Dempster1977EM} is
\begin{equation}\label{eq:gmm-resp}
  r_k(\bm{\xi}) = \frac{\pi_k\,p_k(\bm{\xi})}
    {\sum_{j=1}^K \pi_j\,p_j(\bm{\xi})}
\end{equation}
with $p_k(\bm{\xi}) = \N(\bm{\xi}\mid\bm{\mu}_k,\,\diag(\bm{\sigma}_k^{2}))$. Responsibilities are treated as constants in subsequent gradient computations: the encoder loss does not update the GMM parameters, which are trained only by the negative log-likelihood loss above. Using the responsibilities of eq.~\eqref{eq:gmm-resp}, we compute per-component latent statistics. Let $\bm{z}_n=\enc(\bm{\xi}^{(n)})$ and $r_{nk}=r_k(\bm{\xi}^{(n)})$. The weighted mean and variance per component $k$ and latent dimension $j$ are:
\begin{align}
  \mu_{k,j}^z &= \frac{\sum_n r_{nk}\,z_{nj}}{\sum_n r_{nk}},
    \label{eq:gmm-mu-kz}\\
  v_{k,j}^z   &= \frac{\sum_n r_{nk}\,(z_{nj}-\mu_{k,j}^z)^2}
                       {\sum_n r_{nk}}
    \label{eq:gmm-var-kz}
\end{align}

The component moment loss ensures that the $K$ latent sub-distributions aggregate to $\N(\bm{0},\bm{I})$ up to four marginal moments. Since the $m$-th moment of a mixture is the $\pi$-weighted sum of its components' $m$-th moments, and the moments of $\N(\mu_k, v_k)$ are known in closed form, a 1-D mixture with component placeholders $(\mu_k, v_k)$ matches the first four moments of $\N(0,1)$ if and only if:
\begin{equation}\label{eq:gmm-moments}
\begin{aligned}
  \sum_k \pi_k\,\mu_k &= 0, & \sum_k \pi_k\,(v_k + \mu_k^2) &= 1, \\
  \sum_k \pi_k\,(\mu_k^3 + 3\mu_k\,v_k) &= 0, & \sum_k \pi_k\,(\mu_k^4 + 6\mu_k^2\,v_k + 3v_k^2) &= 3
\end{aligned}
\end{equation}
The component moment loss penalises deviations from these targets, applied to each latent dimension $j$ with the placeholders $(\mu_k, v_k)$ instantiated as the per-component latent statistics $(\mu_{k,j}^z, v_{k,j}^z)$ from eqs.~\eqref{eq:gmm-mu-kz}--\eqref{eq:gmm-var-kz}:
\begin{equation}\label{eq:gmm-comp-loss}
  \loss_{\text{comp}} = \frac{1}{n_z}\sum_{j=1}^{n_z}\Bigl( M_{1,j}^2 + (M_{2,j}-1)^2
    + M_{3,j}^2 + (M_{4,j}-3)^2 \Bigr)
\end{equation}
where $M_{m,j}$ is the left-hand side of the $m$-th line of eq.~\eqref{eq:gmm-moments} evaluated with $\mu_k \leftarrow \mu_{k,j}^z$ and $v_k \leftarrow v_{k,j}^z$.

\begin{corollary}[GMM-decomposed moment condition for (P3)]
\label{cor:gmm-p3}
Suppose (i) the per-component conditional latent distributions are univariate Gaussian, $q(z_j \mid \text{component } k) = \N(\mu_{k,j}^z, v_{k,j}^z)$ for each $k \in \{1,\ldots,K\}$ and $j \in \{1,\ldots,n_z\}$, and (ii) the aggregated-posterior mass of each component equals its GMM weight, $N^{-1}\sum_n r_{nk} = \pi_k$ (exact at an expectation--maximisation fixed point). If the component moment loss $\loss_{\text{comp}} = 0$ (eq.~\eqref{eq:gmm-comp-loss}), then for every dimension $j$ the aggregated marginal $q(z_j)$ matches $\N(0,1)$ in moments $1$ through $4$. The GMM thus decomposes the global matching problem into $K$ per-component conditions on the arrangement of latent sub-distributions.
\end{corollary}

The proof is in Appendix~\ref{ec:proof-gmm}. The match is marginal. The joint dependence across dimensions is not controlled, and agreement in four moments is not distributional equality. Hypotheses (i) and (ii) are not enforced by any training loss.

\begin{remark}[Applicability to arbitrary distributions]
By the universal-approximation property invoked above, the component moment constraint applies regardless of the true data distribution (multimodal, skewed, and heavy-tailed distributions included).
\end{remark}

Combining these terms with the WAAE loss of eq.~\eqref{eq:total-loss} gives the total WAAE-GMM objective:
\begin{equation}\label{eq:gmm-total}
  \loss_{\text{WAAE-GMM}} = \loss
    + \lambda_{\text{gmm}}\,\loss_{\text{GMM}}
    + \lambda_{\text{comp}}\,\loss_{\text{comp}}
\end{equation}
The GMM parameters are optimised with a separate learning rate ($\eta_{\text{gmm}}=10^{-4}$, lower than $\eta_{\text{ae}}=10^{-3}$).

Figure~\ref{fig:gmm-components} visualises the GMM-guided latent decomposition on independent mixture distribution ($K = 12$). The left panel shows the raw data space $(\xi_1, \xi_2)$ with points coloured by their dominant GMM component. The right panel shows how those same components cover the latent space $(z_1, z_2)$ without holes after training. Figure~\ref{fig:gmm-components} illustrates Corollary~\ref{cor:gmm-p3}: the $K$ per-component moment conditions of eq.~\eqref{eq:gmm-comp-loss} jointly recover the target prior.

\begin{figure}[tbp]
\centering
\includegraphics[width=\textwidth]{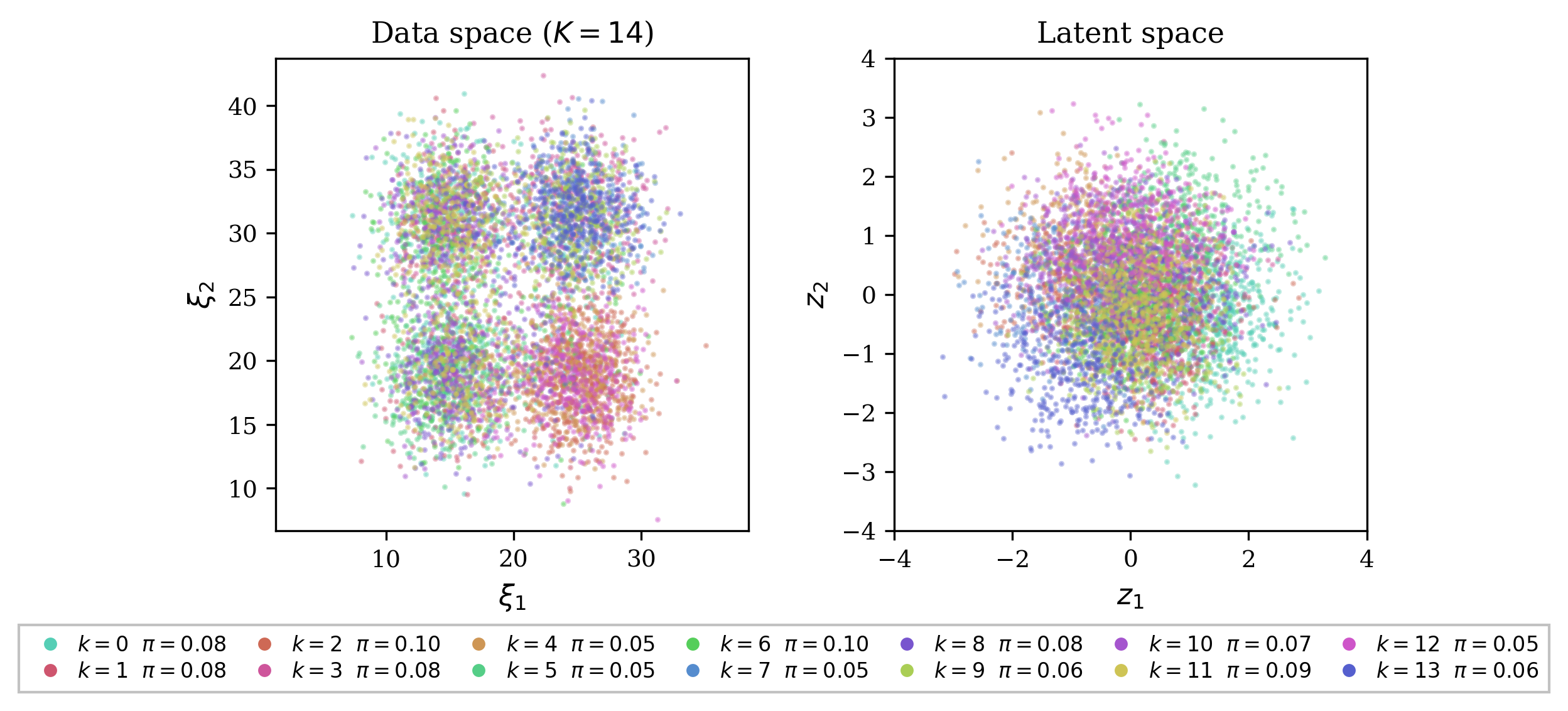}
\caption{WAAE-GMM component membership ($K{=}12$) on independent mixture distribution: component colouring in data space (left) and latent space (right).}
\label{fig:gmm-components}
\end{figure}

Having addressed (P3) through the component moment constraint, we now turn to Point~4: ensuring the trained decoder captures the constraint-violation structure needed for robust decisions.

\subsection{Robust Optimisation-Aware Training}
\label{sec:waae-ro}

Standard WAAE training optimises (P1)--(P3) but is agnostic to the downstream optimisation problem. To address (P4), we extend the training with robust-optimisation-aware (RO-aware) objectives. The scenario-selection step below shares motivation with \emph{problem-driven scenario generation} in stochastic programming~\citep{Fairbrother2022}: both let the downstream problem inform the data side, but the mechanism and output are different. SP scenario generation chooses scenarios to best approximate the optimisation \emph{objective}, producing a discrete weighted scenario set that is then used as a sample-average. Here a pool of feasible decisions is sampled, the \emph{constraint function} is evaluated over the resulting (decision, scenario) grid to identify infeasibility-inducing scenarios, and these scenarios upweight the training loss of a continuous generative model whose decoder is later embedded in a MIP oracle for worst-case verification. No auxiliary optimisation problem is solved during this step.

\subsubsection{Problem Setting}

The (P4) procedure operates per constraint of the robust optimisation~\eqref{eq:ro}: for a fixed index $i \in \{1, \ldots, m\}$, $g_i(\bm{x}, \bm{\xi})$ is the violation signal. A pool of feasible decisions $\{\bm{x}_k\}_{k=1}^{K_{\text{pool}}}$ is sampled from the decision space and $N$ historical scenarios $\{\bm{\xi}^{(n)}\}_{n=1}^{N}$ are drawn i.i.d.\ from $P^*$. Training then upweights pairs $(\bm{x}_k, \bm{\xi}^{(n)})$ with $g_i > 0$ so the decoded uncertainty set covers the infeasibility-inducing tail of $P^*$. Case-study specifics (decision dimensions, sampling procedure, tolerance value) are presented with the numerical results.

\subsubsection{Violation-Aware Training Signal}

Given a decision pool of size $K_{\text{pool}}$ and $N$ historical scenarios, the \emph{violation matrix} is precomputed as
\begin{equation}\label{eq:violation-matrix}
    M_{kn} = g_{i}(\bm{x}_k, \bm{\xi}^{(n)}),
    \quad k = 1, \dots, K_{\text{pool}}, \quad n = 1, \dots, N
\end{equation}
Each entry $M_{kn}$ is the value of constraint $i$ at scenario $\bm{\xi}^{(n)}$ under decision $\bm{x}_k$, with $M_{kn} > 0$ marking an infeasibility and $M_{kn} \le 0$ marking feasibility. Let $\varepsilon_{\text{vio}} > 0$ denote a small violation tolerance below which a constraint value is treated as numerically zero.

For each scenario $\bm{\xi}^{(n)}$, the worst violation it causes across the decision pool, clipped at the tolerance, defines the tail score
\begin{equation}\label{eq:tail-score}
    \rho_n = \max_{k}\; \max\!\bigl(M_{kn} - \varepsilon_{\text{vio}},\; 0\bigr)
\end{equation}
A positive tail score from eq.~\eqref{eq:tail-score}, $\rho_n > 0$, identifies $\bm{\xi}^{(n)}$ as infeasibility-inducing for at least one decision in the pool, while $\rho_n = 0$ leaves it benign.

The tail score is mapped to a per-scenario training weight by a sigmoid:
\begin{equation}\label{eq:tail-weight}
    w_n = 1 + \alpha \cdot \sigma\!\bigl(\kappa(\rho_n - \rho_0)\bigr)
\end{equation}
where $\sigma$ is the sigmoid function, $\alpha$ controls the maximum upweighting, $\kappa$ controls sharpness, and $\rho_0$ is the activation threshold. Each scenario $n$ receives a weight $w_n \in [1, 1{+}\alpha]$ that multiplies its reconstruction-loss contribution during Phase-2 training: benign scenarios ($\rho_n \leq \rho_0$) keep $w_n \approx 1$ and contribute as usual, while infeasibility-inducing scenarios ($\rho_n > \rho_0$) are upweighted by up to a factor of $(1{+}\alpha)$, so the decoder's gradient signal emphasises violation-critical regions of $P^*$ without altering the distributional properties of the training batch (Remark~\ref{rem:mixed-batch}).

The decision--scenario pairs whose constraint value exceeds the tolerance form the \emph{dangerous-pair set} $\mathcal{H} = \{(k, n) : M_{kn} > \varepsilon_{\text{vio}},\ k \in \{1, \dots, K_{\text{pool}}\},\ n \in \{1, \dots, N\}\}$.
Only pairs in $\mathcal{H}$ contribute to the constraint-consistency loss below, so the decoder is corrected only at decision--scenario pairs where infeasibility is triggered.

(P4) training may slightly affect (P1)--(P3) performance. To make this trade-off controllable, training data and scenario data are mixed in each batch:
\begin{equation}\label{eq:mixing}
\begin{aligned}
    \mathcal{B}_{\text{mix}} &= \underbrace{\{\bm{\xi}^{(i)}\}_{i=1}^{N_t}}_{\text{training data}} \cup \underbrace{\{\bm{\xi}^{(n)}\}_{n=1}^{N_s}}_{\text{scenarios}}, \\
    N_t &= \lfloor(1-\beta)B\rfloor, \quad N_s = B - N_t
\end{aligned}
\end{equation}

\begin{remark}[Mixed-batch input distribution]
\label{rem:mixed-batch}
Since the training data $\{\bm{\xi}^{(i)}\}$ and scenario data $\{\bm{\xi}^{(n)}\}$ are both drawn i.i.d.\ from $P^*$, their concatenation $\mathcal{B}_{\text{mix}}$ is itself an i.i.d.\ sample from $P^*$ for any mixing ratio $\beta \in [0, 1]$. The violation weights $w_n$ enter the loss as gradient coefficients only. The unweighted (P2) and (P3) estimators are therefore unbiased on $\mathcal{B}_{\text{mix}}$. Training under the reweighted loss converges to a different stationary point, by design. That shift injects the (P4) signal.
\end{remark}

Training is split into two phases. Phase~1 ($E_1$ epochs) optimises the (P1)--(P3) losses together with the GMM guide under a frozen decoder head. (P4) losses are inactive. Once Phase~1 ends the decoder learning rate drops to $\text{lr}_2$ (Table~\ref{tab:hyperparams}) to prevent gradient conflicts between (P1)--(P3) and (P4) objectives, while the encoder retains the Phase-1 learning rate. Phase~2 ($E_2$ epochs) introduces the (P4) losses through a linear annealing schedule~\citep{Bowman2016} over $E_a$ epochs, with anneal factor $a \in [0,1]$ of eq.~\eqref{eq:weighted-recon}, letting the network adapt to the new objective without disrupting the existing (P1)--(P3) fit.

\subsubsection{Point 4 Loss Components}

The (P4) signal first enters reconstruction through a weighted loss that unifies (P1) and (P4):
\begin{equation}\label{eq:weighted-recon}
    \loss_{\text{w-recon}} = \frac{1}{W} \Bigl( \sum_{i \in \text{train}} \|\bm{\xi}^{(i)} - \hat{\bm{\xi}}^{(i)}\|^2
    + a \sum_{n \in \text{scen}} w_n \|\bm{\xi}^{(n)} - \hat{\bm{\xi}}^{(n)}\|^2 \Bigr)
\end{equation}
where $W = N_t + a \sum_n w_n$ is the total sample weight and $a \in [0,1]$ is the anneal factor ($a = 0$ during Phase~1 so only the training-data term is active, and $a$ ramps to $1$ linearly over the first $E_a$ epochs of Phase~2).

The second (P4) component is a constraint-consistency regularisation. While (P1) reconstruction ensures $\dec(\enc(\bm{\xi})) \approx \bm{\xi}$ in data space, it does not guarantee that the \emph{constraint function value} is preserved.
Two scenarios $\bm{\xi}$ and $\hat{\bm{\xi}}$ may have small data-space distance $\|\bm{\xi} - \hat{\bm{\xi}}\|$ but produce different violation values $g_{i}(\bm{x}, \bm{\xi}) \neq g_{i}(\bm{x}, \hat{\bm{\xi}})$ if they differ along the direction aligned with the decision $\bm{x}$.

The constraint-consistency loss penalises violation mismatches on infeasibility-inducing decision--scenario pairs $\mathcal{H}$, using the constraint function $g$ as the training signal:
\begin{equation}\label{eq:cc}
    \loss_{\text{cc}} = \frac{1}{|\mathcal{H}_{\text{sub}}|} \sum_{(k,n) \in \mathcal{H}_{\text{sub}}} \ell\bigl( g_{i}(\bm{x}_k, \bm{\xi}^{(n)}), g_{i}(\bm{x}_k, \hat{\bm{\xi}}^{(n)})\bigr)
\end{equation}
where $\hat{\bm{\xi}}^{(n)} = \mathcal{S}_{\text{pre}}^{-1}\!\bigl(\dec(\enc(\bm{\xi}^{(n)}))\bigr)$ and $\mathcal{H}_{\text{sub}} \subset \mathcal{H}$ is a random subsample per batch.
This can be viewed as a form of \emph{physics-informed} regularisation: the physical constraint function $g_{i}(\bm{x}, \bm{\xi})$ serves as the loss signal, ensuring the decoder preserves the constraint values across the uncertainty distribution. For problems with multiple constraints ($m > 1$), the same procedure is applied independently to each $g_i$.

Three loss functions $\ell(\cdot, \cdot)$ are considered:
\begin{itemize}[nosep]
    \item \textbf{L1}: $\ell(g, \hat{g}) = |g - \hat{g}|$, symmetric, penalising both over- and underestimation.
    \item \textbf{Hinge}: $\ell(g, \hat{g}) = \max(g - \hat{g}, 0)$, penalising only underestimation (conservative reconstruction incurs zero loss).
    \item \textbf{MMD}: $\ell = \text{MMD}_{1\text{D}}\!\bigl(\{g_{i}(\bm{x}_k, \bm{\xi}^{(n)})\}_n, \{g_{i}(\bm{x}_k, \hat{\bm{\xi}}^{(n)})\}_n\bigr)$, matching the per-decision \emph{distribution} of violations rather than individual values. This variant targets the distributional quantity that the MIP oracle optimises.
\end{itemize}

A \emph{latent-violation} loss reinforces the coupling between infeasibility-inducing scenarios and the latent space by combining (i) a weighted latent MMD between encoded infeasibility-inducing scenarios and the prior, with weights $w_n$ from eq.~\eqref{eq:tail-weight}, and (ii) an encode-decode roundtrip penalty $\|\bm{z}_n - \enc(\dec(\bm{z}_n))\|^2$ on the infeasibility-inducing subsample:
\begin{equation}\label{eq:lat-vio}
    \loss_{\text{lat-vio}} = \MMD^2_k\!\bigl(\{\enc(\bm{\xi}^{(n)})\}, \{\bm{z}_j\}; w_n\bigr)
        + \mathbb{E}_{\bm{z} \sim \enc(\bm{\Xi}_{\text{dang}})}\,\bigl\|\bm{z} - \enc(\dec(\bm{z}))\bigr\|^2
\end{equation}
with $\bm{\Xi}_{\text{dang}}$ the scenarios appearing in $\mathcal{H}_{\text{sub}}$ and $\{\bm{z}_j\} \sim \N(\bm{0}, \bm{I})$. This loss keeps infeasibility-inducing scenarios well-represented in latent space (the weighted MMD) and ensures the decoder is inference-consistent on them (the roundtrip), with weight $\lambda_{\text{lat-vio}}$.

The decoder head can optionally be unfrozen in Phase 2. The decoder head $h$ is frozen during Phase 1 (P1--P3 training), serving only as a fixed inverse of the PCA standardisation. When Phase 2 begins we optionally allow the head to adapt. Three modes are supported: (i) keep the head frozen throughout, (ii) unfreeze only the decoder head $h$, and (iii) unfreeze both $h$ and the encoder pre-head $V$ (a linear pre-encoder with the same structure as the PCA preprocessor of Eq.~\eqref{eq:pca-scaler}, with $\bm{\mu}, \bm{\sigma}$ held fixed at their PCA-fitted values and $V$ made trainable), so that $V$ can co-adapt on both sides of the encoder--decoder pipeline. The weighted-reconstruction loss of eq.~\eqref{eq:weighted-recon} and the constraint-consistency loss of eq.~\eqref{eq:cc} continue to provide the training signal through the now partially trainable decoder. No new reconstruction term is introduced. Allowing $V$ to drift only in Phase 2 lets the head align with the constraint-violation geometry captured by $\loss_{\text{cc}}$ without disturbing the distribution fit learned under the frozen warm-start.
\subsubsection{Unified Training Algorithm}

Algorithms~\ref{alg:waae-ro-p1} and~\ref{alg:waae-ro-p2} give the complete WAAE-GMM-RO training procedure. Phase~1 (Algorithm~\ref{alg:waae-ro-p1}) establishes (P1)--(P3) together with the GMM guide under a frozen decoder head. When Phase~2 begins the decoder learning rate drops to $\text{lr}_2$, and the head (optionally together with the encoder pre-head $V$) becomes trainable. Phase~2 (Algorithm~\ref{alg:waae-ro-p2}) continues all Phase-1 losses while adding the tail-weighted mixed batch, constraint-consistency, and latent-violation losses.

\begin{algorithm}[tbp]
\small
\caption{WAAE-GMM-RO Training: Phase 1 (P1--P3 with GMM guide, frozen head)}\label{alg:waae-ro-p1}
\begin{algorithmic}[1]
\REQUIRE Training data $\bm{X} \in \R^{N \times n_{\xi}}$; phase-1 epochs $E_1$.
\ENSURE Warm-started $(\enc, b, h, \text{GMM})$ with $h$ frozen.
\STATE Fit PCA on $\bm{X}$ to obtain $(\bm{\mu}, \bm{\sigma}, V)$; set $\bm{X}_{\text{pca}} \gets V^{\top}(\bm{X} - \bm{\mu}) \oslash \bm{\sigma}$.
\STATE Initialise head $h$ with weight $\mathrm{diag}(\bm{\sigma})\,V$, bias $\bm{\mu}$; freeze $h$.
\STATE Warm-start GMM $\{(\pi_k, \bm{\mu}_k, \bm{\sigma}_k)\}_{k=1}^{K_{\text{gmm}}}$ via expectation--maximisation (or variational inference with a Dirichlet-process prior) on $\bm{X}$.
\FOR{$e = 1, \dots, E_1$}
    \FOR{each mini-batch $\bm{x} \subset \bm{X}_{\text{pca}}$}
        \STATE $\loss_1 \gets \lambda_{\text{recon}}\loss_{\text{recon}} + \lambda_{\text{data}}\loss_{\text{data}} + \lambda_{\text{latent}}\loss_{\text{latent}} + \lambda_{\text{gmm}}\loss_{\text{GMM}} + \lambda_{\text{comp}}\loss_{\text{comp}}$, i.e.\ $\loss_{\text{WAAE-GMM}}$ of eq.~\eqref{eq:gmm-total}; update $\enc, b, \{\disc^{\text{pca}}, \disc^{\text{std}}\}$, GMM.
    \ENDFOR
\ENDFOR
\RETURN $(\enc, b, h, \text{GMM})$.
\end{algorithmic}
\end{algorithm}

\begin{algorithm}[tbp]
\small
\caption{WAAE-GMM-RO Training: Phase 2 (P4 on tail-weighted mixed batch)}\label{alg:waae-ro-p2}
\begin{algorithmic}[1]
\REQUIRE Warm-started $(\enc, b, h, \text{GMM})$ from Algorithm~\ref{alg:waae-ro-p1}; scenarios $\bm{\Xi}$; decision pool $\mathcal{D}$ with $|\mathcal{D}| = K_{\text{pool}}$; phase-2 epochs $E_2$; anneal period $E_a$; mixing ratio $\beta$; unfreeze mode $\in \{\text{none}, \text{head only}, \text{head}+\text{pre-head}\}$.
\ENSURE Trained $(\enc, b, h)$.
\STATE Precompute violation matrix $M$ (eq.~\eqref{eq:violation-matrix}) and tail weights $\bm{w}$ (eq.~\eqref{eq:tail-weight}).
\STATE Drop decoder learning rate to $\text{lr}_2$. If unfreeze mode $\neq$ none, unfreeze $h$ (and encoder pre-head $V$ when mode $=$ head$+$pre-head) and add their parameters to the AE optimiser.
\FOR{$e = 1, \dots, E_2$}
    \STATE $a \gets \min(1,\; e / E_a)$ \quad\COMMENT{anneal factor}
    \FOR{each mini-batch}
        \STATE Build mixed batch $\mathcal{B}_{\text{mix}}$ (eq.~\eqref{eq:mixing}) and attach tail weights $\bm{w}_{\text{mix}}$.
        \STATE $\loss_2 \gets \lambda_{\text{recon}}\loss_{\text{w-recon}} + \lambda_{\text{data}}\loss_{\text{data}} + \lambda_{\text{latent}}\loss_{\text{latent}} + \lambda_{\text{gmm}}\loss_{\text{GMM}} + 0.1\,\lambda_{\text{comp}}\loss_{\text{comp}} + a\,\lambda_{\text{cc}}\loss_{\text{cc}} + \lambda_{\text{lat-vio}}\loss_{\text{lat-vio}}$ (eqs.~\eqref{eq:weighted-recon},~\eqref{eq:cc},~\eqref{eq:lat-vio}; $\loss_{\text{data}}, \loss_{\text{latent}}, \loss_{\text{GMM}}, \loss_{\text{comp}}$ as in Phase 1; $\lambda_{\text{comp}}$ is attenuated by $0.1$ so that (P4) terms dominate late in training); update $\enc, b, h$ ($+V$ if unfrozen), GMM.
    \ENDFOR
\ENDFOR
\RETURN $(\enc, b, h)$.
\end{algorithmic}
\end{algorithm}


\begin{table}[tbp]
\centering
\footnotesize
\caption{WAAE-GMM-RO hyperparameters.}
\label{tab:hyperparams}
\begin{tabular}{llcl}
\toprule
\textbf{Parameter} & \textbf{Symbol} & \textbf{Default} & \textbf{Description} \\
\midrule
Mixing ratio & $\beta$ & 0.2 & Fraction of scenarios in mixed batch \\
Phase-1 epochs & $E_1$ & 1000 & Epochs of (P1)--(P3) + GMM training (frozen head) \\
Anneal period & $E_a$ & 100 & Linear annealing duration \\
Phase-2 epochs & $E_2$ & 500 & Epochs of (P1)--(P4) training with mixed batch \\
Phase-2 learning rate & $\text{lr}_2$ & $5{\times}10^{-4}$ & Decoder LR during Phase 2 \\
GMM components & $K_{\text{gmm}}$ & 12 & Mixture components in the data-space GMM guide \\
Tail upweighting & $\alpha$ & 3.0 & Maximum tail weight factor \\
Sigmoid sharpness & $\kappa$ & 10.0 & Sigmoid steepness \\
Activation threshold & $\rho_0$ & 0.1 & Violation threshold for upweighting \\
Reconstruction weight & $\lambda_{\text{recon}}$ & 2.0 & (P1) loss weight \\
Constraint consistency & $\lambda_{\text{cc}}$ & 1.0 & (P4) constraint-consistency weight \\
CC loss type & $\ell$ & hinge & Constraint-consistency loss ($\ell\in\{\text{L1, hinge, MMD}\}$) \\
Latent violation weight & $\lambda_{\text{lat-vio}}$ & 1.0 & (P4) latent loss weight \\
Critic steps & $n_{\text{critic}}$ & 5 & Discriminator steps per generator step \\
Gradient penalty & $\lambda_{\text{gp}}$ & 10.0 & WGAN-GP penalty weight \\
CC subsample size & $|\mathcal{H}_{\text{sub}}|$ & 256 & Dangerous pairs per batch \\
Gradient clipping & --- & 1.0 & Max gradient norm (optional) \\
Refresh interval & $R$ & 50 & Epochs between (P4) data refresh \\
Head unfreeze mode & --- & none & \{none, head only, head + pre-head\} in Phase 2 \\
\bottomrule
\end{tabular}
\end{table}

With the encoder--decoder trained to satisfy (P1)--(P4), we now embed the ReLU decoder in a mixed-integer linear program for exact worst-case verification (P5).

\subsection{MILP Integration and Worst-Case Oracle}
\label{sec:omlt}

The central premise of using a ReLU decoder is exact worst-case verification: given a candidate decision~$\bm{x}$, we can certify whether the generative uncertainty set contains any scenario that violates a given constraint.
Building on the MILP representation of Section~\ref{sec:decoded-set}, the oracle is embedded in a cutting-plane outer loop that solves the GRO problem.

\subsubsection{Decoder Export and MILP Encoding}
\label{sec:omlt-export}
After training, the decoder is exported via ONNX and lifted to a Pyomo MILP block using the OMLT library~\citep{Ceccon2022}. Each hidden ReLU is encoded by eq.~\eqref{eq:bigm}, yielding one binary variable per hidden neuron ($128$ in our $[6, 32, 64, 32, 6]$ decoder) and reproducing eq.~\eqref{eq:oracle-milp} exactly. Two operational choices complete the encoding. First, once the decoder is exported we tighten its Big-M relaxations offline with optimality-based bound tightening (OBBT), which accelerates the subsequent branch-and-bound. Second, the inverse data-scaling is folded into the block itself: under the PCA warm-start pipeline (Section~\ref{sec:waae-gmm}) the decoder head $h$ implements the PCA inverse $h(\bm{y}) = \diag(\bm{\sigma})\,V\,\bm{y} + \bm{\mu}$, so the composed network $h \circ b$ outputs raw-unit scenarios directly, with no separate inverse-scaling step after the MILP block. The decoder body also carries an input-output skip, $b(\bm{z}) = \text{net}(\bm{z}) + \bm{z}$ (eq.~\eqref{eq:skip}). Because the skip is linear in $\bm{z}$, it is included in the ONNX export and appears as a standard feedforward operation in the MILP block, adding $n_z$ linear terms with no new binary variables.

\subsubsection{Worst-Case Oracle}
\label{sec:oracle}
Given a candidate decision $\bm{x}$ and an uncertain constraint $g(\bm{x}, \bm{\xi})$ from eq.~\eqref{eq:ro}, the worst-case violation is computed by solving
\begin{equation}\label{eq:oracle-full}
    v^{\star}(\bm{x}) = \max_{\bm{z},\,\bm{\xi}} \; g(\bm{x}, \bm{\xi}) \quad \text{s.t.} \quad \bm{\xi} = \dec(\bm{z}), \; \bm{z} \in \Zz,
\end{equation}
where $\Zz$ is the calibrated latent set (eq.~\eqref{eq:z-box} or eq.~\eqref{eq:z-ball}).
The oracle returns:
\begin{itemize}[nosep]
    \item the worst-case violation $v^{\star}$,
    \item the optimal latent vector $\bm{z}^{\star}$,
    \item the worst-case scenario $\bm{\xi}^{\star} = \dec(\bm{z}^{\star})$ in original-unit space,
    \item the MIP optimality gap and solution time.
\end{itemize}

A convex feasibility check precedes the exact oracle. Solving the full oracle of eq.~\eqref{eq:oracle-full} requires branch-and-bound over the 128 ReLU binaries and is the dominant cost of the cutting-plane loop. The oracle is a MILP under the box latent constraint of eq.~\eqref{eq:z-box} and a MISOCP under the ball constraint of eq.~\eqref{eq:z-ball}. In either case its continuous relaxation (ReLU binaries relaxed to $[0,1]$, latent constraint retained as-is) is convex (an LP or an SOCP respectively) and provides an upper bound $\bar v$ on the worst-case violation. The feasibility check uses this relaxation with the OBBT-tightened Big-M bounds in place. When $\bar v \le \varepsilon_{\text{vio}}$, feasibility is certified. Otherwise the full oracle is invoked.

When the check does not certify feasibility, the exact oracle returns $v^{\star}$ and $\bm{z}^{\star}$ at global optimality. If $v^{\star} \leq \varepsilon_{\text{vio}}$, the decision is certified robust-feasible within the calibrated latent set. Otherwise the worst-case scenario $\bm{\xi}^{\star} = \dec(\bm{z}^{\star})$ is a constructive certificate of infeasibility.

This exact verification is what distinguishes the oracle from sampling. Monte Carlo sampling of $\bm{z}$ from the calibrated set $\Zz$ and evaluating $g(\bm{x}, \dec(\bm{z}))$ provides only a \emph{lower bound} on the worst-case violation: if the worst-case latent lies in a narrow region missed by the sample, the bound is loose. The oracle of eq.~\eqref{eq:oracle-full} guarantees global optimality (gap = 0), providing the \emph{exact} worst-case, which gradient-based heuristics such as PGA~\citep{Brenner2025} do not provide.

\subsubsection{GRO Cutting-Plane Algorithm}
\label{sec:ro-cutting-plane}

The trained decoder and the oracle of Section~\ref{sec:oracle} close the GRO loop. The master problem inherits the form of the application's deterministic counterpart: an LP for the production-planning case study, a MILP for facility location where binary opening decisions make the master integer. Algorithm~\ref{alg:ro-cutting-plane} treats this master as a black-box and is agnostic to its type. The master starts with only the deterministic constraints. Each iteration solves the current master, queries the oracle (with the feasibility check of Section~\ref{sec:oracle} in front of the exact oracle) on every uncertain constraint, and either certifies the current decision as robust-feasible or adds one robust cut per violated constraint to the master. The preprocessing (MILP export and OBBT) from Section~\ref{sec:omlt-export} is done once, before the loop, and is reused across all iterations.
\begin{algorithm}[tbp]
\small
\caption{GRO Cutting-Plane Algorithm}\label{alg:ro-cutting-plane}
\begin{algorithmic}[1]
\REQUIRE Trained decoder pipeline $\dec: \Zz \to \R^{n_{\xi}}$ mapping latent $\bm{z}$ to original-unit outputs (Section~\ref{sec:omlt-export}); calibrated latent set $\Zz$; nominal master-LP builder (no uncertainty constraints); set of uncertain constraints $\mathcal{C}$ with per-constraint violation function $g_c(\bm{x}, \bm{\xi})$; violation tolerance $\varepsilon_{\text{vio}}$; maximum master iterations $k_{\max}$.
\ENSURE Robust-feasible decision $\bm{x}^{\star}$, or max-iter with best-so-far $\bm{x}$.
\STATE \textbf{[Preprocessing]}
\STATE Export $\dec$ to ONNX; build the OMLT MILP block once (Section~\ref{sec:omlt-export}).
\STATE Run OBBT once (one LP per hidden variable) to tighten Big-M constants.
\STATE \textbf{[Outer master loop]}
\STATE Build master LP; $k \gets 0$.
\WHILE{$k < k_{\max}$}
    \STATE Solve master LP $\to$ candidate decision $\bm{x}$.
    \STATE \textbf{[Phase A]:} feasibility screen on every constraint
    \FOR{\textbf{each} $c \in \mathcal{C}$}
        \STATE \textbf{Feasibility check} (Section~\ref{sec:oracle}): solve the convex relaxation of the oracle for $c$ at $\bm{x}$ $\to$ screen value $\bar v_c$.
    \ENDFOR
    \STATE $\mathcal{C}_{\text{cand}} \gets \{\, c \in \mathcal{C} : \bar v_c > \varepsilon_{\text{vio}} \,\}$ \quad // MIP candidates
    \IF{$\mathcal{C}_{\text{cand}} = \emptyset$}
        \RETURN $\bm{x}^{\star} \gets \bm{x}$ \quad // every constraint certified by the screen
    \ENDIF
    \STATE \textbf{[Phase B]:} exact MIPs on the candidate list
    \FOR{\textbf{each} $c \in \mathcal{C}_{\text{cand}}$ \textbf{in parallel}}
        \STATE \textbf{Exact oracle}: solve the full MIP~\eqref{eq:oracle-full} for $c$ at $\bm{x}$ $\to (v^{*}_c, \bm{\xi}^{*}_c)$.
    \ENDFOR
    \STATE $\mathcal{C}_{\text{vio}} \gets \{\, c \in \mathcal{C}_{\text{cand}} : v^{*}_c > \varepsilon_{\text{vio}} \,\}$ \quad // confirmed violators
    \IF{$\mathcal{C}_{\text{vio}} = \emptyset$}
        \RETURN $\bm{x}^{\star} \gets \bm{x}$
    \ENDIF
    \STATE \textbf{[Phase C]:} one robust cut per confirmed violator
    \FOR{\textbf{each} $c \in \mathcal{C}_{\text{vio}}$}
        \STATE Add robust cut $g_c(\cdot, \bm{\xi}^{*}_c) \leq 0$ to the master LP.
    \ENDFOR
    \STATE $k \gets k + 1$.
\ENDWHILE
\RETURN max-iter reached; current $\bm{x}$ with worst violation $\max_{c \in \mathcal{C}_{\text{vio}}} v^{*}_c$.
\end{algorithmic}
\end{algorithm}

Two properties justify this scheme. First, finite convergence: the master polytope is bounded and the exact oracle returns an optimal $\bm{\xi}^{\star}$ each iteration. The loop either terminates by certification or produces a monotonically tightening sequence of master feasibility sets until $k_{\max}$. Second, the guarantee carries through: if the loop returns $\bm{x}^{\star}$ via the certification branch, $\bm{x}^{\star}$ is feasible against every scenario in the decoded uncertainty set $\Uset_{\dec}$. The three-phase inner structure of Algorithm~\ref{alg:ro-cutting-plane} accelerates the loop without compromising this guarantee. Phase~A applies the feasibility check (Section~\ref{sec:oracle}) to each constraint and prunes to the candidate list $\mathcal{C}_{\text{cand}}$. Phase~B solves the exact oracle on each candidate in parallel. The candidates share the master decision $\bm{x}$, the MILP block, and the OBBT-tightened Big-M bounds, differing only in the constraint $g_c$. Phase~C adds one robust cut per confirmed violator to the master, then the loop advances.

\section{Case Studies}
\label{sec:results}

\subsection{Multi-Period Production Planning}
\label{sec:cs-case-study}

We evaluate the GRO framework on a 6-period production planning (PP) problem.

\subsubsection{Problem formulation}
\label{sec:cs-pp-formulation}

A company plans production, storage, and sales over $T=6$ periods of 2~months each to maximise sales revenue subject to a cost budget.
For each period $t = 1, \ldots, 6$, the decision variables are
$x_t$ (production, tons),
$y_t$ (end-of-period inventory, tons), and
$z_t$ (sales, tons).
The deterministic formulation is:
\begin{align}
    \max_{x, y, z} \quad & \sum_{t=1}^{6} P_t \, z_t \label{eq:cs-obj} \\
    \text{s.t.} \quad & \sum_{t=1}^{6} \tilde{C}_t \, x_t + \sum_{t=1}^{6} V_t \, y_t \leq B \label{eq:cs-budget} \\
    & y_{t-1} + x_t = y_t + z_t, \quad t = 1, \ldots, 6, \;\; y_0 := I_0 \label{eq:cs-inv} \\
    & y_6 = I_{\text{final}} \label{eq:cs-invfinal} \\
    & 0 \leq x_t \leq U_t, \quad 0 \leq z_t \leq D_t, \quad y_t \geq 0 \label{eq:cs-bounds}
\end{align}
where $P_t$ are selling prices, $\tilde{C}_t$ are uncertain production costs, $V_t$ are storage costs, $B = 400{,}000$ is the budget, and $I_0 = I_{\text{final}} = 500$~tons.
The uncertain costs $\tilde{C}_t$ are the only source of uncertainty. All other parameters are known.
Table~\ref{tab:cs-data} lists the nominal parameters.

\begin{table}[tbp]
\centering
\footnotesize
\caption{Production planning problem data (\citealp{Ravindran2008}; robust cost-uncertainty version after \citealp{Li2012}).}
\label{tab:cs-data}
\begin{tabular}{ccccccc}
\toprule
Period $t$ & 1 & 2 & 3 & 4 & 5 & 6 \\
\midrule
Price $P_t$ (\$/ton) & 180 & 180 & 250 & 270 & 300 & 320 \\
Nominal cost $C_t$ (\$/ton) & 20 & 25 & 30 & 40 & 50 & 60 \\
Storage $V_t$ (\$/ton) & 2 & 2 & 2 & 2 & 2 & 2 \\
Capacity $U_t$ (tons) & 1500 & 2000 & 2200 & 3000 & 2700 & 2500 \\
Demand $D_t$ (tons) & 1100 & 1500 & 1800 & 1600 & 2300 & 2500 \\
\bottomrule
\end{tabular}
\end{table}

Training data consist of $N = 6{,}000$ scenarios drawn from six synthetic cost distributions: three independent variants (normal, skewed, mixture) and three correlated variants with shared latent factor ($\rho = 0.7$). Each marginal is centred at the nominal cost $C_t$ of Table~\ref{tab:cs-data} with a fixed coefficient of variation $\sigma_t / C_t = 0.25$, so the per-period standard deviations $\sigma_t$ range from $5.00$ to $15.00$ \$/ton across the six periods. The mixture variants additionally use an inner-component standard deviation of $\sigma_t / 2$ around each of two component means.

This problem instantiates the abstract robust formulation eq.~\eqref{eq:ro} of Section~\ref{sec:problem} as follows. The decision vector is $\bm{x} = (\bm{x}^{\text{prod}},\, \bm{y},\, \bm{z}^{\text{sale}}) \in \R^{18}$ with $\bm{x}^{\text{prod}} = (x_1,\ldots,x_6)$, $\bm{y} = (y_1,\ldots,y_6)$, and $\bm{z}^{\text{sale}} = (z_1,\ldots,z_6)$. The uncertainty vector is $\bm{\xi} = (\tilde{C}_1,\ldots,\tilde{C}_6) \in \R^6$. The budget eq.~\eqref{eq:cs-budget} is the sole uncertain constraint and takes the form $g(\bm{x}, \bm{\xi}) = \bm{\xi}^{\top}\bm{x}^{\text{prod}} + \bm{V}^{\top}\bm{y} - B \leq 0$, where $\bm{V} = (V_1,\ldots,V_6)$. The remaining constraints, eqs.~\eqref{eq:cs-inv}--\eqref{eq:cs-bounds}, are deterministic and define the feasible set $\mathcal{X}$ of eq.~\eqref{eq:ro}. Under the uncertain costs, the robust counterpart maximises the sales revenue eq.~\eqref{eq:cs-obj} subject to the budget eq.~\eqref{eq:cs-budget} holding for every cost realisation in the calibrated uncertainty set. Its optimal value is the robust objective $\tau$ (in M\$).

\subsubsection{Trained decoder evaluation}
\label{sec:cs-pp-decoder-eval}

The trained decoders are evaluated against the empirical production-planning distribution before deployment in the cutting-plane algorithm of Section~\ref{sec:ro-cutting-plane}. Table~\ref{tab:full-comparison} reports the five-point cross-model comparison. Figures~\ref{fig:worst-case-fixed-d} and~\ref{fig:oracle-validation} report the worst-case oracle accuracy at a representative decision and across a pool of $K_{\text{pool}} = 2000$ feasible decisions, respectively.

Table columns implement the five-point metric definitions of Section~\ref{sec:five-points} (cf.\ eqs.~\eqref{eq:p1}, \eqref{eq:p2-mmd}, \eqref{eq:delta-wc}). The Point-3 cells count latent dimensions passing the Anderson--Darling and Shapiro--Wilk normality tests at $\alpha = 0.01$. The Point-4 columns ($\Delta_{\text{wc}}$, conserv, corr) are evaluated on each model's encoder--decoder reconstruction of the $2000$ empirical reference scenarios, which probes whether the decoder preserves the violation geometry of the training distribution. NF is bijective, so its reconstruction is exact and the Point-4 cells reduce to a trivial identity (reported as ``---'').

Point-4 training is essential for producing a deployable decoder, irrespective of architecture: the tail-weighted (P4) loss keeps WAAE-GMM-RO's conservatism ratio within a tight band around the safe-but-tight target ($1.002$--$1.023$ across the six distributions, all on the safe side), the best of any model in the table and the only model with conserv ${\geq} 1.00$ everywhere.

\begin{table}[tbp]
\scriptsize
\centering
\caption{Cross-model comparison across six distribution families ($d = 6$, $N = 6000$). Metric definitions are given in the body. Bold marks the best entry per distribution and metric, excluding NF for the Point-4 columns. }
\label{tab:full-comparison}
\scriptsize
\setlength{\tabcolsep}{2pt}
\begin{tabular}{llcccccccl}
\toprule
& & \textbf{(P1)} & \textbf{(P2)} & \multicolumn{2}{c}{\textbf{(P3)}} & \textbf{(P4)} & & & \textbf{(P5)} \\
\cmidrule(lr){3-3} \cmidrule(lr){4-4} \cmidrule(lr){5-6} \cmidrule(lr){7-7} \cmidrule(lr){8-9} \cmidrule(lr){10-10}
\textbf{Dist.} & \textbf{Model} & $r_{\text{mse}}$ & MMD & AD & SW & $\Delta_{\text{wc}}$ & conserv & corr & MILP \\
\midrule
\multirow{5}{*}{\rotatebox{90}{\scriptsize\shortstack{Indep.\\ Normal}}}
& VAE        & 0.0210 & \textbf{0.0002} & \textbf{6} & 6; 57/60 & 0.083 & 0.917 & 0.999 & 128 \\
& WAAE       & 0.0020 & \textbf{0.0002} & \textbf{6} & \textbf{6; 59/60} & \textbf{0.020} & 1.012 & \textbf{1.000} & 128 \\
& WAAE-GMM-RO & \textbf{0.0013} & 0.0006 & \textbf{6} & \textbf{6; 59/60} & \textbf{0.020} & \textbf{1.008} & 0.999 & 128 \\
& NF         & --- & \textbf{0.0002} & 5 & \textbf{6; 59/60} & --- & --- & --- & MINLP \\
& DDIM       & 0.0282 & 0.0067 & 5 & \textbf{6; 60/60} & 0.138 & 0.862 & \textbf{1.000} & ${\sim}$6400 \\
\midrule
\multirow{5}{*}{\rotatebox{90}{\scriptsize\shortstack{Indep.\\ Skewed}}}
& VAE        & 0.0241 & \textbf{0.0004} & 0 & 1; 25/60 & 0.044 & 1.041 & 0.998 & 128 \\
& WAAE       & 0.0084 & \textbf{0.0004} & 2 & 6; 51/60 & 0.042 & 0.958 & 0.997 & 128 \\
& WAAE-GMM-RO & \textbf{0.0043} & 0.0009 & \textbf{6} & \textbf{6; 58/60} & \textbf{0.016} & \textbf{1.015} & \textbf{1.000} & 128 \\
& NF         & --- & \textbf{0.0002} & \textbf{6} & \textbf{6; 59/60} & --- & --- & --- & MINLP \\
& DDIM       & 0.0309 & 0.0078 & 5 & 6; 57/60 & 0.041 & 0.959 & 0.999 & ${\sim}$6400 \\
\midrule
\multirow{5}{*}{\rotatebox{90}{\scriptsize\shortstack{Indep.\\ Mixture}}}
& VAE        & 0.0375 & 0.0017 & 1 & 1; 24/60 & 0.127 & 0.873 & \textbf{1.000} & 128 \\
& WAAE       & 0.0167 & 0.0012 & 3 & 5; 52/60 & \textbf{0.035} & 0.979 & 0.999 & 128 \\
& WAAE-GMM-RO & 0.0044 & \textbf{0.0006} & \textbf{6} & \textbf{6; 59/60} & \textbf{0.037} & \textbf{1.023} & 0.999 & 128 \\
& NF         & --- & 0.0003 & 0 & 6; 51/60 & --- & --- & --- & MINLP \\
& DDIM       & 0.0202 & 0.0041 & 0 & 3; 34/60 & 0.135 & 0.865 & \textbf{1.000} & ${\sim}$6400 \\
\midrule
\multirow{5}{*}{\rotatebox{90}{\scriptsize\shortstack{Corr.\\ Normal}}}
& VAE        & 0.0324 & 0.0007 & \textbf{6} & \textbf{6; 58/60} & 0.104 & 0.896 & 0.998 & 128 \\
& WAAE       & 0.0193 & 0.0018 & \textbf{6} & \textbf{6; 58/60} & \textbf{0.007} & 0.994 & \textbf{1.000} & 128 \\
& WAAE-GMM-RO & \textbf{0.0022} & 0.0012 & \textbf{6} & \textbf{6; 58/60} & \textbf{0.010} & \textbf{1.005} & 0.999 & 128 \\
& NF         & --- & \textbf{0.0001} & \textbf{6} & \textbf{6; 58/60} & --- & --- & --- & MINLP \\
& DDIM       & 0.0303 & 0.0057 & \textbf{6} & \textbf{6; 58/60} & 0.034 & 0.966 & \textbf{1.000} & ${\sim}$6400 \\
\midrule
\multirow{5}{*}{\rotatebox{90}{\scriptsize\shortstack{Corr.\\ Skewed}}}
& VAE        & 0.0399 & 0.0007 & 4 & \textbf{6; 56/60} & 0.009 & 0.994 & 0.999 & 128 \\
& WAAE       & 0.0180 & 0.0025 & 1 & 5; 48/60 & 0.023 & 1.005 & 0.994 & 128 \\
& WAAE-GMM-RO & \textbf{0.0050} & \textbf{0.0008} & \textbf{6} & \textbf{6; 58/60} & \textbf{0.005} & \textbf{1.002} & \textbf{1.000} & 128 \\
& NF         & --- & \textbf{0.0002} & \textbf{6} & \textbf{6; 56/60} & --- & --- & --- & MINLP \\
& DDIM       & 0.0298 & 0.0055 & 3 & \textbf{6; 53/60} & 0.051 & 0.949 & \textbf{1.000} & ${\sim}$6400 \\
\midrule
\multirow{5}{*}{\rotatebox{90}{\scriptsize\shortstack{Corr.\\ Mixture}}}
& VAE        & 0.0859 & 0.0020 & 0 & 1; 33/60 & 0.063 & 0.938 & 0.999 & 128 \\
& WAAE       & 0.0137 & 0.0021 & 0 & 0; 23/60 & 0.115 & 0.888 & 0.986 & 128 \\
& WAAE-GMM-RO & \textbf{0.0035} & \textbf{0.0008} & \textbf{5} & \textbf{6; 57/60} & \textbf{0.029} & \textbf{1.018} & 0.998 & 128 \\
& NF         & --- & 0.0006 & 4 & \textbf{6; 60/60} & --- & --- & --- & MINLP \\
& DDIM       & 0.0257 & 0.0125 & 0 & 1; 22/60 & 0.071 & 0.929 & \textbf{1.000} & ${\sim}$6400 \\
\bottomrule
\end{tabular}
\end{table}

Several conclusions emerge from Table~\ref{tab:full-comparison}:
\begin{itemize}[nosep]
\item \textbf{VAE} fails (P3) on non-Gaussian data: KL regularisation alone does not deliver a well-calibrated aggregated posterior, which propagates into a (P4) underestimation on the skewed and mixture distributions.

\item \textbf{WAAE} matches the data distribution well (P1)--(P3) but underestimates the worst-case violation on the harder skewed and mixture cells, since the loss carries no (P4) signal.

\item \textbf{NF} matches the data distribution well, but its multi-step structure (stacked coupling layers with $\exp(\tanh(\cdot))$ activations) pushes the verification problem to MINLP rather than MILP, ruling out exact oracle solves: a (P5) failure.

\item \textbf{DDIM} is also a multi-step model: unrolling its $T$-step denoising chain inflates the binary-variable count an order of magnitude beyond the single-pass decoders, so (P5) fails. It additionally underestimates the worst-case (conserv ${<} 1$ on all six distributions), failing (P4) as well.

\item \textbf{WAAE-GMM-RO} is the only model that satisfies all five points: an accurate distribution match, a regular latent, a (P4)-safe oracle, and a tractable MILP-embeddable decoder.
\end{itemize}

Figure~\ref{fig:worst-case-fixed-d} provides a deployment-side snapshot at a single decision $\bm{x}_{\text{med}}$ (median worst-case violation in the pool) on the independent mixture distribution. The top row plots the $(\xi_2, \xi_5)$ data-space slice: the ground-truth column shows the empirical scatter coloured by $g(\bm{x}_{\text{med}}, \bm{\xi})/B$, and each model column overlays the $2000$ encoder--decoder reconstructions of those reference scenarios on the same rainbow scale. The bottom row plots the latent $(z_1, z_2)$ density of the encoded reference scenarios, with a dashed reference circle at the calibrated radius $\gamma_{\text{marg}}(p = 0.9) \approx 1.64$. Black $\bigstar$ marks the empirical worst-case scenario $\bm{\xi}^{*}_{\text{ref}}$ (shared across columns). Model-coloured $\bigstar$ marks the decoder-side worst-case reconstruction $\bm{\xi}^{*}_{\text{gen}}$ on top and the matching latent $\bm{z}^{*}_{\text{gen}}$ on bottom. A well-calibrated decoder places its worst-case reconstruction close to the empirical reference and its latent-space worst-case inside the calibrated ball. The VAE column shows both failure modes simultaneously: its worst-case falls outside the calibrated ball of the latent space, reflecting the Point-3 failure of Section~\ref{sec:p3}, and its data-space reconstruction underestimates the empirical worst-case violation, leaving the resulting oracle on the unsafe side of (P4). A standalone VAE is therefore structurally insufficient as a generative uncertainty set.

\begin{figure}[tbp]
\centering
\includegraphics[width=\textwidth]{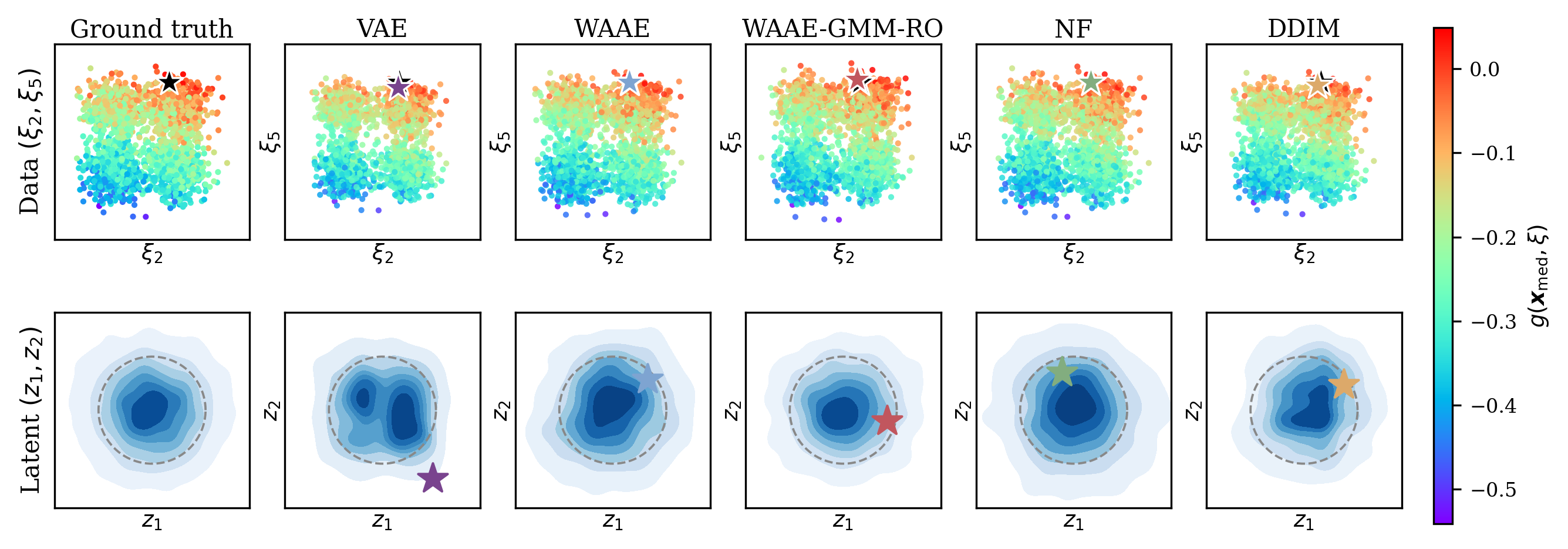}
\caption{Worst-case oracle at a fixed median decision $\bm{x}_{\text{med}}$ on the independent mixture distribution. Top row: data-space slice $(\xi_2, \xi_5)$. Bottom row: latent slice $(z_1, z_2)$.}
\label{fig:worst-case-fixed-d}
\end{figure}

Figure~\ref{fig:oracle-validation} provides a per-decision view of the oracle across the six PP distributions for VAE, WAAE, WAAE-GMM-RO, and DDIM. NF is omitted, since its bijective decoder makes $v_{\text{gen}} \equiv v_{\text{ref}}$ by construction. Each point is one of $K_{\text{pool}} = 2000$ pool decisions $\bm{x}_k$: the $x$-axis is the empirical worst-case violation $v_{\text{ref}}(\bm{x}_k)$, the $y$-axis is its encoder--decoder counterpart $v_{\text{gen}}(\bm{x}_k)$, both as defined in Section~\ref{sec:p4-intro} and normalised here by the budget $B$. Per-panel annotations report the Pearson correlation \emph{corr}, the pool-averaged relative error $\Delta_{\text{wc}}$ of eq.~\eqref{eq:delta-wc}, and the conservatism ratio \emph{conserv}. Across the six distributions WAAE-GMM-RO keeps its points on or above $y = x$ (\emph{conserv} $\ge 1$ everywhere), whereas the VAE and DDIM points fall below it on the unsafe side for most distributions. The per-decision view therefore confirms the aggregate safety reported in Table~\ref{tab:full-comparison}.

\begin{figure}[tbp]
\centering
\includegraphics[width=0.98\textwidth]{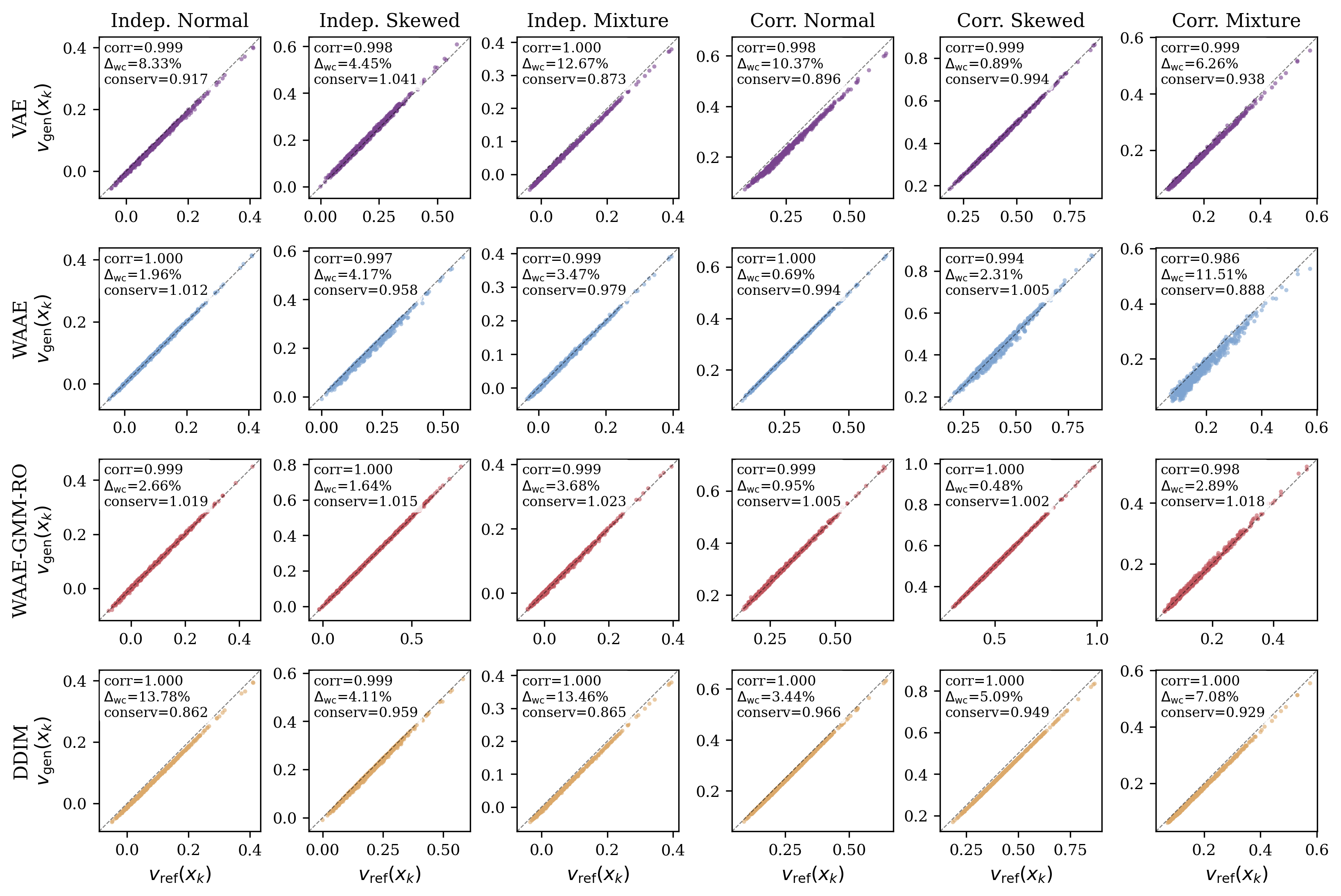}
\caption{Per-decision oracle validation across the six distributions (columns) and four generative models (rows). Each point is one of $K_{\text{pool}} = 2000$ pool decisions $\bm{x}_k$. The dashed $y = x$ marks a perfectly calibrated oracle. A safe oracle keeps every point on or above this line and as close to the line as possible to limit over-conservatism.}
\label{fig:oracle-validation}
\end{figure}

\subsubsection{Algorithm~\ref{alg:ro-cutting-plane} specialisation and comparison baselines}
\label{sec:cs-pp-algorithm}

We compare seven uncertainty set approaches at the same marginal coverage $p$ (eq.~\eqref{eq:gamma-per-dim}): two classical robust baselines, three distributionally-robust baselines, and two decoder-based methods. The WAAE and WAAE-RO ablation rows are reported in Table~\ref{tab:full-comparison} above and are not repeated here.
\begin{itemize}[nosep]
    \item \textbf{Box LP.} Marginal empirical quantile box $[\xi^{\text{lo}}_t, \xi^{\text{hi}}_t] = [Q_{(1-p)/2},\, Q_{(1+p)/2}]$, with the worst case at the upper bound since $x_t \geq 0$, solved as LP.
    \item \textbf{IE SOCP.} Interval+ellipsoidal counterpart~\citep{BenTalNemirovski1999} with $\Omega = \Phi^{-1}((1+p)/2)$, solved as SOCP.
    \item \textbf{DDRO SOCP.} Data-driven distributionally robust counterpart~\citep{Bertsimas2018} using the empirical covariance $\hat{\Sigma}$ and a coverage-$p$ confidence set, solved as SOCP.
    \item \textbf{W-DRO.} Wasserstein-1 ambiguity ball about the empirical cost distribution ($L_2$ ground metric)~\citep[Thm~4.2]{MohajerinEsfahani2018}. For the affine budget, its SOCP counterpart adds one cone constraint $\widehat{\bm{\xi}}^{\top}\bm{x} + \sum_t V_t y_t + \varepsilon\,\|\bm{x}\|_2 \leq B$ to the master, with $\widehat{\bm{\xi}}$ the empirical mean cost and $\varepsilon$ the radius.
    \item \textbf{OT-GMM-DRO.} Optimal-transport DRO over the weight simplex of a fixed $K = 8$ Gaussian mixture fit by EM~\citep[Model~20]{Kammammettu2024}. The inner supremum over weight transports, bounded in squared Bures distance, $\sum_{i,j} \pi_{ij} d_{ij}^2 \leq \varepsilon^2$, admits an LP dual added to the master.
    \item \textbf{VAE.} VAE decoder with the MIP oracle of Section~\ref{sec:oracle}.
    \item \textbf{WAAE-GMM-RO.} WAAE-GMM-RO decoder with the same MIP oracle as VAE.
\end{itemize}
Computations use Gurobi~12 with OBBT preprocessing on an NVIDIA RTX~2000~Ada GPU.

The PP case specialises Algorithm~\ref{alg:ro-cutting-plane} as follows: $|\mathcal{C}| = 1$ (the budget eq.~\eqref{eq:cs-budget} is the sole uncertain constraint), and the decision vector $\bm{x} \in \R^{18}$ is the consolidated stack of Section~\ref{sec:cs-pp-formulation}. Each outer iteration therefore queries one MIP oracle.

\subsubsection{Numerical validation and quantile sweep}
\label{sec:cs-pp-evidence}

We first validate the oracle on independent normal costs, the one case with an analytical reference. For costs $\xi_t \sim \N(\mu_t, \sigma_t^2)$, the analytical $q$-quantile robust counterpart is
\begin{equation}\label{eq:true-var-socp}
    \sum_t \mu_t x_t + \Phi^{-1}(q) \sqrt{\sum_t \sigma_t^2 x_t^2} + \sum_t V_t y_t \leq B
\end{equation}
which an affine decoder on $\|\bm{z}\|_2 \leq \gamma_{\text{marg}}$ reproduces exactly under the calibration of eq.~\eqref{eq:gamma-per-dim}. On Gaussian data the decoder therefore cannot outperform the analytical set, and we measure how closely the \emph{trained} decoder and MIP oracle approximate it.

Across the coverage sweep $p \in \{0.1, 0.3, 0.5, 0.7, 0.9\}$ under the ball latent constraint, the two decoder objective curves overlap the analytical SOCP, while the Box~LP, IE~SOCP and DDRO~SOCP baselines lie progressively below, reflecting set-shape conservatism (Fig.~\ref{fig:validation}, left panel). The middle panel resolves the decoder gap on a $\pm 8\%$ scale: WAAE-GMM-RO remains on the safe side throughout, with the gap growing from $0$ at $p = 0.1$ to approximately $-0.44\%$ at $p \geq 0.7$, since the trained-decoder nonlinearity admits a slightly larger image than the analytical ball. VAE sits on the unsafe side, with the gap reaching $+0.28\%$ at $p = 0.9$, consistent with its (P3) failure (Section~\ref{sec:p3}). Both gaps remain within half a percent across the full $p$ range. The right panel verifies decision-level agreement at $p = 0.9$: both decoder schedules coincide with SOCP period by period. On data where an exact reference exists, the trained decoder and MIP oracle therefore reproduce the analytical robust counterpart to within half a percent, validating the pipeline before the non-Gaussian distributions where no closed form is available.

\begin{figure}[tbp]
\centering
\includegraphics[width=0.98\textwidth]{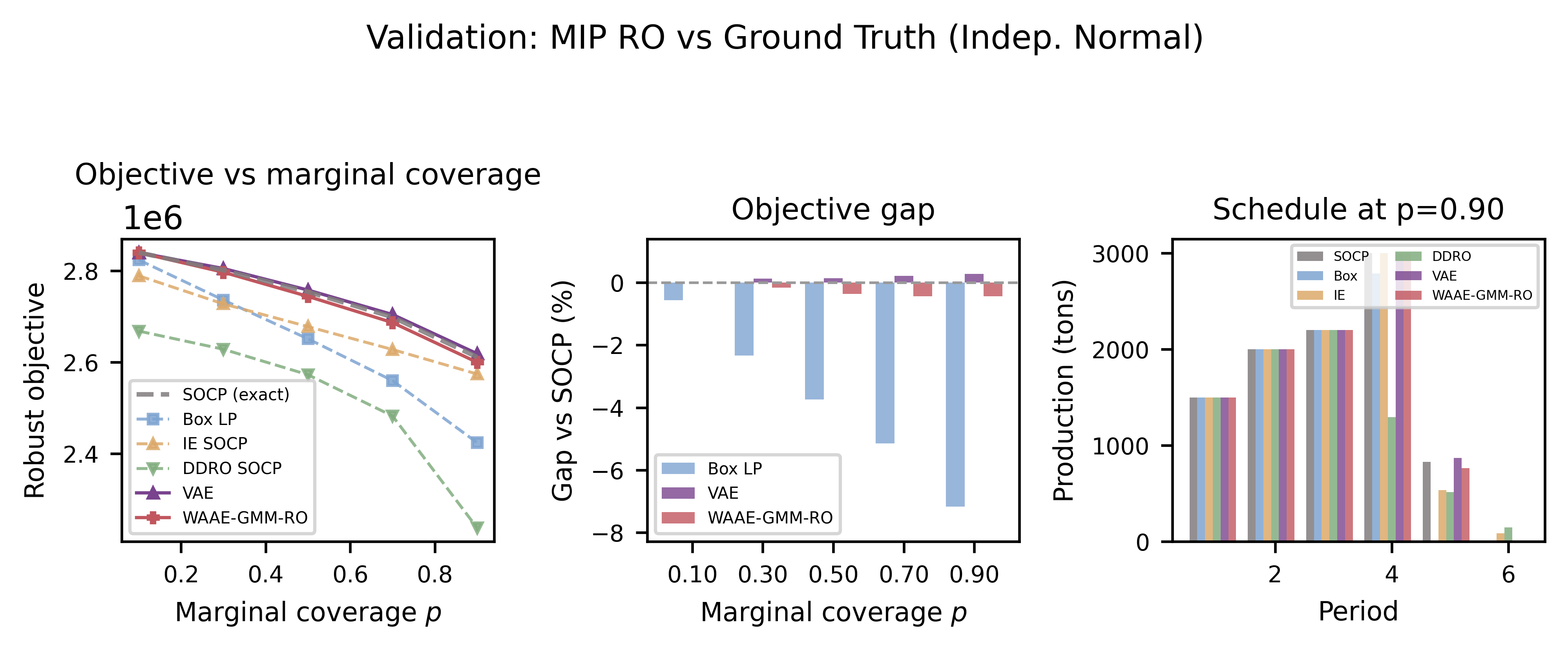}
\caption{Validation on the independent normal distribution under the ball latent constraint. Left: $\tau$ vs $p$. Middle: relative gap to SOCP. Right: 6-period schedule at $p = 0.9$. Lines: SOCP-exact (gray dashed), Box~LP (blue), IE~SOCP (tan), DDRO~SOCP (sage), VAE~MIP (purple), WAAE-GMM-RO~MIP (red).}
\label{fig:validation}
\end{figure}

The $(\xi_1, \xi_2)$ slice at $p = 0.8$ shows the same comparison geometrically (Fig.~\ref{fig:uncertainty-regions}, independent normal on top, independent mixture on the bottom). On Gaussian data all methods nearly coincide, and both decoders reduce to a single Gaussian peak at the data mean. On bimodal data the two decoders separate: the WAAE-GMM-RO decoded density splits into four separated mode peaks that track the empirical mixture, whereas VAE recovers the four modes but places noticeable mass between them. The Box~LP and IE~SOCP regions cover the same area but flat-fill the interior, carrying no information about where the data concentrates, and the two distributionally-robust sets are no better here: W-DRO reduces, for the affine budget, to an isotropic ball about the empirical mean, and OT-GMM-DRO to a frozen set of Gaussian components. Neither adapts to the bimodal structure that the decoders capture above.

\begin{figure}[tbp]
\centering
\includegraphics[width=0.98\textwidth]{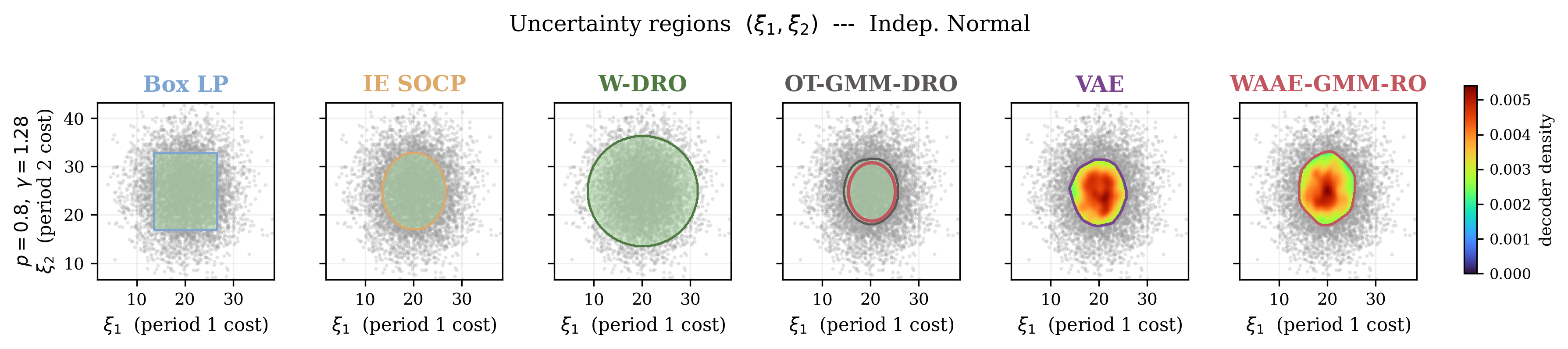}\\[3pt]
\includegraphics[width=0.98\textwidth]{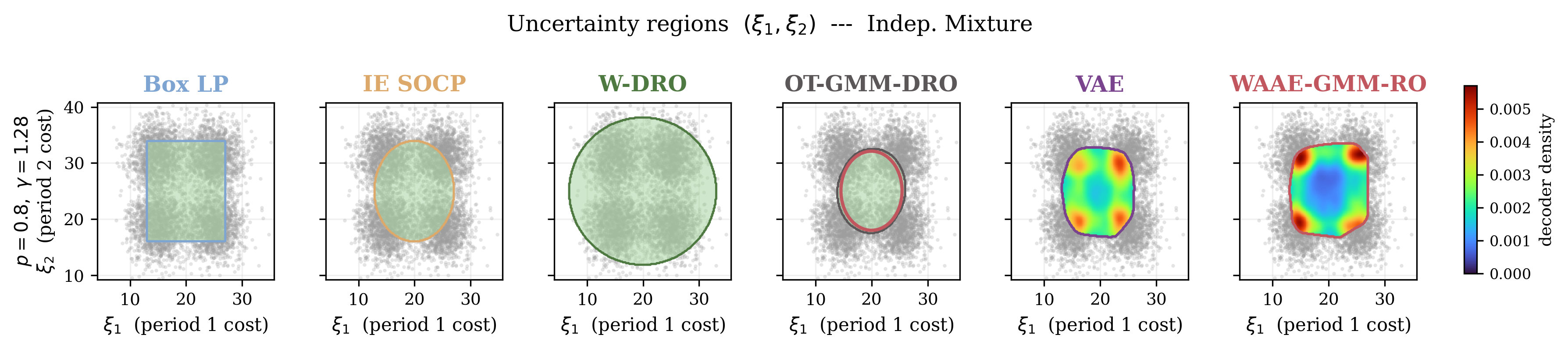}
\caption{Per-method uncertainty regions in the slice $(\xi_1, \xi_2)$ at $p = 0.8$. W-DRO is an isotropic ball about the empirical mean. OT-GMM-DRO is the $K = 8$ fixed mixture components, with the worst-cost component outlined. Upper: independent normal. Lower: independent mixture.}
\label{fig:uncertainty-regions}
\end{figure}

The robust objective $\tau$ is swept over marginal coverage $q \in [0.1, 0.9]$ for the two MILP-embeddable decoders on all six distributions (Fig.~\ref{fig:cs-sweep}). Each decoder is plotted in the two latent-shape variants (the ``box'' eq.~\eqref{eq:z-box} and the inscribed ``ball'' eq.~\eqref{eq:z-ball}), at the calibrated marginal radius $\gamma_{\text{marg}}$ of eq.~\eqref{eq:gamma-per-dim}. Throughout, lower $\tau$ corresponds to a larger uncertainty set (more conservative robust decision) and higher $\tau$ to a smaller set. The two analytical baselines (Box~LP and IE~SOCP) are independent of the decoder's latent shape.

On the non-Gaussian distributions the two analytical baselines, Box~LP and IE~SOCP, have no closed-form reference and hold a fixed geometry. Figure~\ref{fig:cs-sweep} also overlays two distributionally-robust baselines, W-DRO and OT-GMM-DRO, each calibrated to the same one-sided feasibility $(1+p)/2$ as the decoders, so a higher $\tau$ marks a less conservative decision at equal safety. Both are shape-blind (Fig.~\ref{fig:uncertainty-regions}): W-DRO reaches the requested feasibility on every distribution at a profit cost that grows with coverage, and OT-GMM-DRO saturates at a distribution-dependent coverage ceiling that under-covers at high $p$ (the hollow-marker segments).

The decoder adapts. Under the box latent both decoders track Box~LP within a few percent and sit ${\sim}6\%$ below SOCP-exact on independent normal, since the decoded box exceeds the inscribed ball (WAAE-GMM-RO runs slightly more conservative than VAE on the multimodal and correlated cells). Under the ball latent they coincide with SOCP-exact on independent normal (Fig.~\ref{fig:validation}) and sit above Box~LP and IE~SOCP on the other five distributions, a less conservative set than any analytical baseline. Where the data is multimodal or correlated, the learned set splits into the right modes and follows the correlation (Fig.~\ref{fig:uncertainty-regions}), capturing structure that no isotropic ball, ellipsoid, or frozen mixture can, and it carries no coverage ceiling. Off the Gaussian case no analytical reference ranks conservatism exactly, and on this single affine budget the decoder offers no lower robust cost. Its advantage is a shape-adaptive set delivered under one calibration. The classical baselines, conversely, need no training, solve in under a second, and carry finite-sample out-of-sample guarantees the decoder lacks.

Together, the two latent variants show the practical advantage of the formulation: one trained decoder captures the real data distribution, and the uncertainty set is then chosen by the latent constraint rather than re-derived for each case. This decouples distribution learning from set specification in two ways: the same decoder spans different set specifications without retraining (the box and ball geometries, marginal or joint coverage, at any level $q$), and it embeds unchanged into a different \emph{type} of robust problem, as the facility-location study of Section~\ref{sec:fl-case-study} shows. The advantage is greatest on the non-Gaussian, multimodal, and correlated distributions, where no fixed geometric primitive captures both the data shape and the desired coverage.

\begin{figure}[tbp]
\centering
\includegraphics[width=0.98\textwidth]{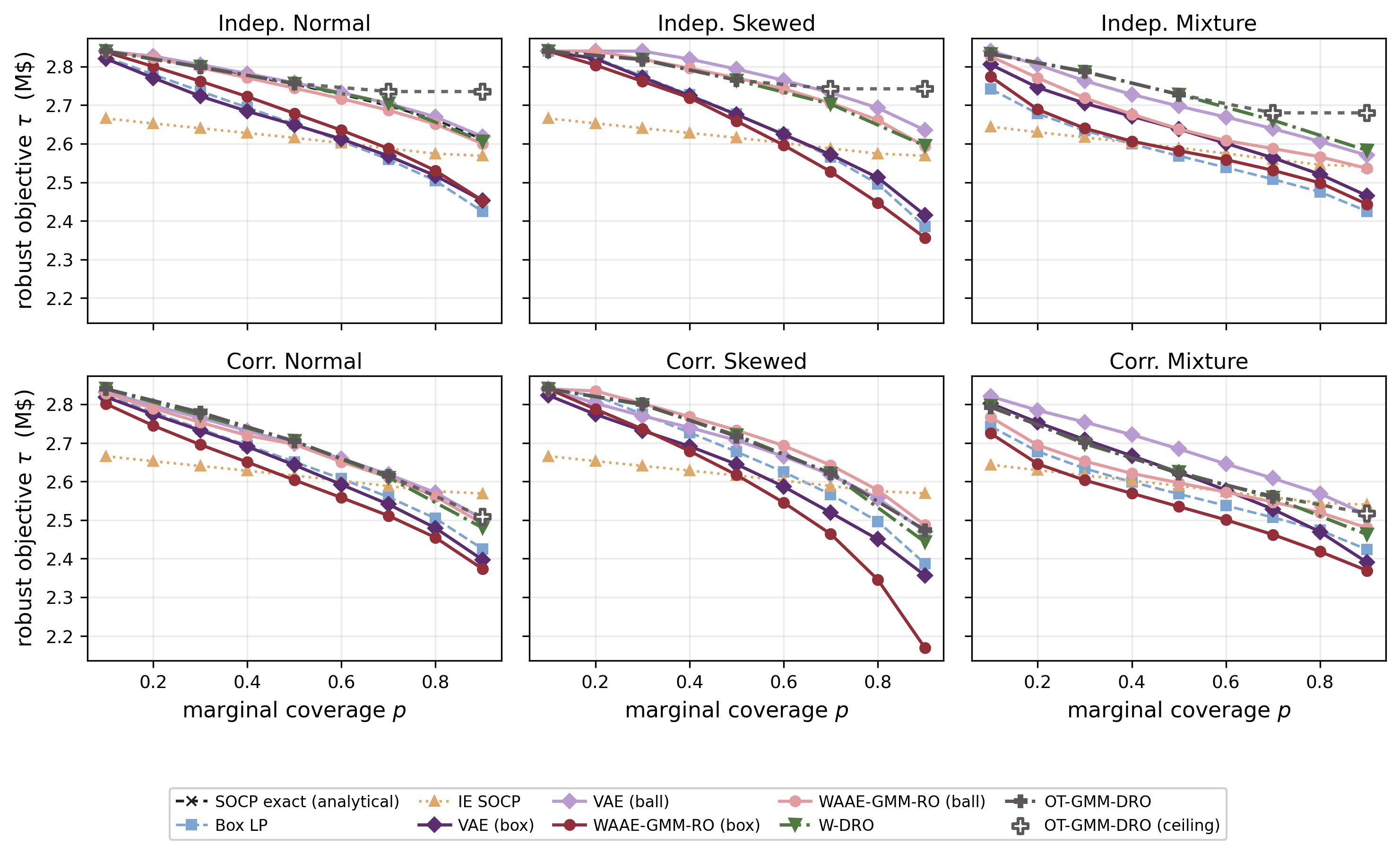}
\caption{PP quantile sweep across the six uncertainty distributions: robust objective $\tau$ (M\$) vs marginal coverage $q \in [0.1, 0.9]$. Row 1: independent variants. Row 2: correlated variants. Each decoder appears in two latent-shape variants (``box'' = $|z_j| \leq \gamma$, ``ball'' = $\|\bm{z}\|_2 \leq \gamma$), with the W-DRO and OT-GMM-DRO baselines overlaid. OT-GMM-DRO markers are hollow where it under-covers.}
\label{fig:cs-sweep}
\end{figure}

\subsubsection{Computational cost}
\label{sec:cs-pp-compute}

Computational cost has two components: a one-off per-model training cost (Table~\ref{tab:training-cost}) and a per-query cutting-plane cost (Table~\ref{tab:cs-cost}).

Table~\ref{tab:training-cost} reports the median wall-clock across the hyperparameter-search trials. WAAE-GMM-RO is the most expensive, since its adversarial critics, GMM guidance, and Point-4 RO-aware stage add per-epoch overhead beyond a plain autoencoder. This cost is paid once per distribution before deployment, so even the $431$\,s figure is amortised across all subsequent queries.

\begin{table}[tbp]
\centering
\footnotesize
\caption{Per-model training cost on the six PP distributions ($d = 6$, $N = 6000$). Median wall-clock training time per distribution on an NVIDIA RTX 2000 Ada GPU.}
\label{tab:training-cost}
\begin{tabular}{lc}
\toprule
\textbf{Model} & \textbf{Training time per distribution (s)} \\
\midrule
VAE          & 18 \\
WAAE         & 100 \\
WAAE-GMM-RO  & 431 \\
NF           & 30 \\
DDIM         & 15 \\
\bottomrule
\end{tabular}
\end{table}

The non-generative baselines, including the two distributionally-robust ones, solve in under $0.5$\,s and are omitted from Table~\ref{tab:cs-cost}. For the two MILP-embeddable decoders ($128$ binary variables), Gurobi solves each worst-case oracle to global optimality, and the cutting-plane loop converges in a few outer iterations. Table~\ref{tab:cs-cost} reports the outer-iteration count, per-iteration time, and total wall-clock at $q = 0.5$ and $q = 0.9$ as the median and $[\text{min}, \text{max}]$ over four retrain seeds. The four seeds yield density-equivalent decoders, yet their wall-clock varies sharply, spanning $60$ to $2167$\,s on the hardest cell (independent skewed at $q = 0.9$). The variation traces to the branch-and-bound geometry of each trained decoder. The bottleneck is the big-M relaxation, so a tighter ReLU encoding~\citep{Anderson2020} is a promising, as-yet-untested route to a uniform speed-up, while the correlated distributions are inexpensive throughout because the PCA rotation aligns the latent ball with the correlation structure. OBBT preprocessing (Section~\ref{sec:omlt-export}) adds ${\sim}30$--$60$\,s, amortised across queries.

\begin{table}[tbp]
\centering
\scriptsize
\caption{End-to-end cutting-plane cost on the $q$-sweep for the two MILP-embeddable decoders across the six distributions: outer-iteration count, per-iteration time, and total wall-clock at $q = 0.5$ and $q = 0.9$ (median~[min--max] over four retrain seeds).}
\label{tab:cs-cost}
\begin{tabular}{llcccccc}
\toprule
& & \multicolumn{3}{c}{$q = 0.5$} & \multicolumn{3}{c}{$q = 0.9$} \\
\cmidrule(lr){3-5} \cmidrule(lr){6-8}
\textbf{Distribution} & \textbf{Method} & iters & s/iter & total (s) & iters & s/iter & total (s) \\
\midrule
\multirow{2}{*}{Indep. Normal}  & VAE         & 4 & 2.7  & 13 [6--22]    & 7  & 11.0 & 69 [40--112]   \\
                                & WAAE-GMM-RO & 4 & 1.3  & 6 [1--405]    & 7  & 14.5 & 98 [5--1083]   \\
\midrule
\multirow{2}{*}{Indep. Skewed}  & VAE         & 4 & 4.2  & 17 [6--72]    & 12 & 11.6 & 107 [35--429]  \\
                                & WAAE-GMM-RO & 4 & 6.1  & 24 [2--81]    & 13 & 49.3 & 876 [60--2167] \\
\midrule
\multirow{2}{*}{Indep. Mixture} & VAE         & 4 & 2.6  & 12 [3--50]    & 12 & 8.5  & 98 [37--161]   \\
                                & WAAE-GMM-RO & 6 & 6.4  & 35 [14--646]  & 5  & 68.8 & 337 [51--1013] \\
\midrule
\multirow{2}{*}{Corr. Normal}   & VAE         & 4 & 2.6  & 10 [3--37]    & 4  & 8.6  & 34 [13--44]    \\
                                & WAAE-GMM-RO & 4 & 2.0  & 7 [2--23]     & 4  & 6.5  & 26 [3--129]    \\
\midrule
\multirow{2}{*}{Corr. Skewed}   & VAE         & 4 & 3.6  & 14 [6--39]    & 4  & 8.7  & 35 [18--53]    \\
                                & WAAE-GMM-RO & 4 & 4.0  & 16 [2--25]    & 4  & 7.1  & 35 [9--367]    \\
\midrule
\multirow{2}{*}{Corr. Mixture}  & VAE         & 4 & 10.0 & 41 [21--63]   & 4  & 23.9 & 94 [37--343]   \\
                                & WAAE-GMM-RO & 4 & 4.6  & 19 [9--32]    & 4  & 11.8 & 47 [14--109]   \\
\bottomrule
\end{tabular}
\end{table}


\subsection{Robust Facility Location with Demand Uncertainty}
\label{sec:fl-case-study}

We apply the generative uncertainty set framework to a multi-period robust facility location problem with uncertain demand, following the formulation of~\citep{Baron2011}. This case study extends the framework in two directions: (i)~integration with combinatorial optimisation, since the master is a MILP with binary location decisions and endogenous capacity sizing, and (ii)~a per-period latent radius schedule that adapts the spatial decoder set to Baron's deterministic temporal-growth recurrence on the uncertainty band.

\subsubsection{Problem formulation}
\label{sec:fl-formulation}

Consider a network of $N$ nodes on a two-dimensional plane, each serving as both a potential facility location and a demand point, over $T$ scheduling periods.
Node~$j$ has uncertain demand $\tilde{\xi}_{jt}$ in period~$t$.

The strategic decisions, made once before demand is observed, are $I_i \in \{0,1\}$ (open a facility at node~$i$) and $Z_{i0} \geq 0$ (installed capacity). The operational decisions, made per period, are the allocation fractions $X_{ijt} \in [0,1]$ and the production levels $Z_{it} \geq 0$. The nominal problem maximises profit:
\begin{align}
\max \quad & \tau \label{eq:fl-obj} \\
\text{s.t.} \quad
& \sum_{t,i,j} (r - d_{ij})\, \xi_{jt}\, X_{ijt} \nonumber \\
& \qquad - \sum_{t,i} c_i\, Z_{it} - \sum_{i} (C_{i0}\, Z_{i0} + K\, I_i) \geq \tau, \label{eq:fl-profit} \\
& \sum_{j} \xi_{jt}\, X_{ijt} \leq Z_{it}, && \forall\, i,\, t, \label{eq:fl-production} \\
& Z_{it} \leq Z_{i0}, && \forall\, i,\, t, \label{eq:fl-cap-link} \\
& Z_{i0} \leq M\, I_i, && \forall\, i, \label{eq:fl-open-link} \\
& \sum_{i} X_{ijt} \leq 1, && \forall\, j,\, t, \label{eq:fl-service} \\
& X_{ijt} \geq 0,\; I_i \in \{0,1\},\; Z_{i0}, Z_{it} \geq 0 && \label{eq:fl-domain}
\end{align}
Parameters follow Baron et al.\ Table~1: $N = 15$, $T = 20$, unit revenue $r = 1.0$, $K = 50\,000$, $C_{i0} = 0.1$, $c_i = 0.1$, $d_{ij} = \|\bm{p}_i - \bm{p}_j\|_2$, $\bar{\xi}_j \sim \mathrm{Uniform}[17\,500,\, 22\,500]$.
The capacity $Z_{i0}$ is a decision variable: the firm chooses capacity before demand is realised.

\subsubsection{Uncertainty model and training data}
\label{sec:fl-uncertainty}

Baron et al.\ model demand as $\tilde{\xi}_{jt} = \bar{\xi}_j (1 + \varepsilon_{jt})$ with a deterministic growth schedule for the uncertainty-band half-width,
\begin{equation}\label{eq:fl-eps-schedule}
\varepsilon_0 = 0, \qquad \varepsilon_t = \gamma + (1-\gamma)\,\varepsilon_{t-1}, \quad \gamma = 0.15
\end{equation}
(giving $\varepsilon_1 = 0.15$, $\varepsilon_{19} \approx 0.95$), and a per-period per-node box $U^B_{jt} = \bar{\xi}_j [1 - \varepsilon_t, 1 + \varepsilon_t]$. Since the allocation $X_{ijt} \geq 0$, the worst case of production constraint~\eqref{eq:fl-production} at $(i,t)$ is attained analytically at $\bar{\xi}_j (1 + \varepsilon_t)$, so the $NT$ production worst-cases reduce to closed-form linear constraints in a single master solve. Baron's ellipsoid counterpart admits an analogous SOC form.

We replace Baron's box spatial geometry with a learned set $\hat{\bm{\xi}}_t = \dec(\bm{z}_t)$ where $\bm{z}_t \in \Zz^{(\gamma_t)}$ is constrained to a ball of radius $\gamma_t$. The decoder output is the demand directly, so no $\varepsilon_t$ mixing with nominal is needed. The per-period latent radius is computed from Baron's schedule via the single mapping $\gamma_t = \Phi^{-1}\!\bigl((1 + \min(\varepsilon_t, \varepsilon_{\mathrm{cap}})) / 2\bigr)$. Here $\varepsilon_t$ plays the role of the marginal coverage parameter in eq.~\eqref{eq:gamma-per-dim}, capped at $\varepsilon_{\mathrm{cap}}$ so that $\gamma_t$ stays within a bounded radius. The $T = 20$ experiment uses $\varepsilon_{\mathrm{cap}} = 0.75$ ($\gamma_{\max} \approx 1.15$). At $T = 5$ no cap binds and $\gamma_0 = 0$, $\gamma_4 \approx 0.64$.

\begin{remark}[Per-period latent ball]\label{rem:per-period-bounds}
The schedule $\{\gamma_t\}_{t=0}^{T-1}$ is computed once from Baron's $\{\varepsilon_t\}$ at oracle construction time and stored in the oracle. Period $t = 0$ has zero radius ($\gamma_0 = 0$): the latent ball reduces to the single point $\bm{z} = \bm{0}$, which decodes to the nominal demand, so there is no worst case to search and the oracle evaluates the constraint at that point directly, without building a MISOCP. OBBT (Section~\ref{sec:omlt-export}) runs once at the widest radius $\gamma_{T-1}$ (${\sim}124$\,s for the FL decoder) so the tightened Big-M bounds are valid for every $(i, t)$ query. Individual queries at smaller $\gamma_t$ solve faster because the ball contracts but the bounds do not loosen.
\end{remark}

The decoder is trained on a multimodal demand distribution that replaces Baron's uniform $\pm \varepsilon_t$ box perturbation. For each scenario and \emph{independently} for each node $j$, a regime is drawn from a three-component categorical distribution and the demand $\xi_j$ is then sampled from a truncated Gaussian conditional on the chosen regime:
\begin{itemize}[nosep]
    \item \emph{Low:} $\mu = 0.7\, \bar{\xi}_j$, $\sigma = 0.05\, \bar{\xi}_j$, weight $= 0.3$.
    \item \emph{Normal:} $\mu = 1.0\, \bar{\xi}_j$, $\sigma = 0.10\, \bar{\xi}_j$, weight $= 0.4$.
    \item \emph{High:} $\mu = 1.3\, \bar{\xi}_j$, $\sigma = 0.15\, \bar{\xi}_j$, weight $= 0.3$.
\end{itemize}
Drawing the regime per node induces a joint distribution with $3^N$ modes (Section~\ref{sec:waae-gmm}). This multimodal spatial structure is what a box or ellipsoid cannot capture, motivating the learned decoder set used here.

The decoder is trained using the (P4) procedure of Section~\ref{sec:waae-ro}, with two FL-specific adaptations. First, feasible decisions are generated by opening a random subset of facilities, deterministically allocating each demand node to its nearest open facility (one-hot $X_{ij}$), setting $Z_{it} = \sum_j \bar{\xi}_j X_{ij}$, and installing capacity $Z_{i0} = \max_t Z_{it} \cdot (1 + s_i)$ with per-facility slack $s_i \sim U[0, 1.5]$. The utilisation of facility $i$ is then $u_i = \max_t Z_{it} / Z_{i0} = 1/(1 + s_i)$, its peak production as a fraction of installed capacity (the FL analogue of the PP budget utilisation). Decisions with $\max_i u_i < u_{\min} = 0.60$ are discarded, keeping the training set near the capacity boundary, where demand can plausibly exceed capacity and the Point-4 violation signal is informative.

Second, although the problem has $N \times T$ production constraints, two properties of our training data reduce the violation to a per-facility quantity: the training scenarios $\bm{\xi}^{(s)} \in \R^N$, each a demand vector over the $N$ nodes indexed by scenario $s$, are period-agnostic patterns (temporal growth enters only at oracle time via $\gamma_t$), and the sampler uses $X_{ijt} = X_{ij}$ and $Z_{i0}$ time-invariant by construction. Under these conditions the per-$(i,t)$ violation is independent of $t$, so we work with the per-facility
\begin{equation}\label{eq:fl-viol-per-i}
g_i(\bm{x}, \bm{\xi}) = \max\!\left(0,\; \frac{\sum_{j=1}^N X_{ij}\, \xi_j - Z_{i0}}{Z_{i0}}\right)
\end{equation}
matching the multi-constraint vector form of Section~\ref{sec:waae-ro} (eq.~\eqref{eq:cc}) with $N$ components and violation data of shape $(K, S, N)$ instead of $(K, S, N, T)$. The (P4) Constraint-consistency loss of Section~\ref{sec:waae-ro} is applied elementwise across the $N$ facilities. Pairs are flagged as infeasibility-inducing when $\max_i g_i > \epsilon_{\mathrm{vio}} = 10^{-4}$. The profit bound~\eqref{eq:fl-profit} is tracked diagnostically but, being effectively always satisfied in our parameter regime, contributes no training signal.

\subsubsection{Algorithm~\ref{alg:ro-cutting-plane} specialisation}
\label{sec:fl-solution}

The FL case specialises Algorithm~\ref{alg:ro-cutting-plane} in three ways. First, the latent radius $\gamma_t$ schedules with the period, so the constraint set $\mathcal{C} = \{(i, t)\}$ has $|\mathcal{C}| = N T$ and each outer iteration queries up to $N T$ oracles rather than one. Second, within each period only facilities with $I^{*}_{i} = 1$ in the current master solution require an oracle (active-facility pruning). Third, each per-$(i, t)$ exact oracle is a MISOCP rather than a pure MILP because the latent constraint is a ball. The decoder oracle for constraint $(i, t)$ searches the latent ball $\Zz^{(\gamma_t)}$ for the demand realisation maximising the overload,
\begin{equation}\label{eq:fl-oracle-prod}
v^*_{it} = \max_{\bm{z}_t \in \Zz^{(\gamma_t)}} \; \sum_{j=1}^N \dec(\bm{z}_t)_j \, X^*_{ijt} - Z^*_{it}
\end{equation}
which is a MISOCP over the composed decoder $h \circ b$ (Section~\ref{sec:omlt-export}) with $d_z = 15$ latent variables, $587$ hidden ReLU neurons, and one convex quadratic ball constraint $\|\bm{z}_t\|_2 \leq \gamma_t$. The feasibility check of Section~\ref{sec:oracle} is applied per $(i, t)$ before the exact MISOCP is solved (Section~\ref{sec:fl-evidence} characterises the screen's prune rate on this geometry). The feasibility cut added to the master at each violated $(i, t)$ is the per-period production cut $\sum_{j} \hat{\xi}_{jt}\, X_{ijt} \leq Z_{it}$, where $\hat{\xi}_{jt} = \dec(\bm{z}^*_t)_j$ is the worst-case demand returned by the oracle eq.~\eqref{eq:fl-oracle-prod}.

\subsubsection{Results and computational cost}
\label{sec:fl-evidence}

\begin{table}[tbp]
\centering
\footnotesize
\caption{Structural comparison of uncertainty models for robust facility location.}
\label{tab:fl-structural}
\begin{tabular}{lcccc}
\toprule
\textbf{Method} & \textbf{Spatial} & \textbf{Temporal} & \textbf{Multimodal} & \textbf{Oracles/iter} \\
\midrule
Box             & Independent     & $\varepsilon_t$ schedule & No  & 0 (analytical)  \\
Ellipsoid       & Joint (linear)  & $\varepsilon_t$ schedule & No  & 0 (analytical)  \\
Decoder (ours)  & Joint (learned) & $\gamma_t(\varepsilon_t)$ & Yes & $\leq NT$ MISOCPs \\
\bottomrule
\end{tabular}
\end{table}

All three methods share the deterministic temporal growth of eq.~\eqref{eq:fl-eps-schedule} but parameterise the spatial geometry differently (Table~\ref{tab:fl-structural}): Box treats nodes independently and Ellipsoid captures linear correlations, while the decoder is the only one capturing nonlinear multimodal spatial structure and the only one without a closed-form worst case, so it alone calls a MISOCP oracle per $(i, t)$. This added spatial expressiveness is what the profit comparison below isolates.

On the full Baron $T = 20$, $N = 15$ instance, the parallel cutting-plane driver is applied to the FL decoder (WAAE-GMM-RO, $d_z = 15$, $K = 12$ GMM components, ball latent constraint, $\varepsilon_{\mathrm{cap}} = 0.75$, $\gamma_{\max} \approx 1.15$). Table~\ref{tab:fl-t20-results} compares it against the nominal LP and Baron's box counterpart. Only WAAE-GMM-RO satisfies all five points (Table~\ref{tab:full-comparison}) and is therefore the only model embeddable here as an exact MISOCP.

\begin{table}[tbp]
\centering
\footnotesize
\caption{Facility location, $T = 20$: WAAE-GMM-RO decoder vs.\ nominal LP vs.\ Baron's box counterpart.}
\label{tab:fl-t20-results}
\begin{tabular}{lrrrr}
\toprule
\textbf{Method} & $\tau$ (profit) & Facilities open & Iterations & Wall-clock \\
\midrule
Nominal LP                   & $4{,}694{,}256$ & 11 & 1 (no oracle) & --- \\
Box (Baron)                  & $705{,}952$     & 4  & 1 (analytical) & --- \\
WAAE-GMM-RO--MISOCP (ours)   & \textbf{$4{,}555{,}060$} & 11 & 6 & 10.4\,h (6-worker parallel) \\
\bottomrule
\end{tabular}
\end{table}

The decoder set recovers $97.0\%$ of the nominal optimum, while Baron's box counterpart reaches only $15.0\%$, a $6.4\times$ profit gap. The cutting-plane loop converges in six iterations, the number of violating $(i, t)$ pairs falling $208 \to 101 \to 52 \to 15 \to 1 \to 0$. The per-iteration MISOCP cost is analysed below. The learned multimodal set therefore protects against the same demand realisations the box guards against while shedding the spatial over-conservatism behind the box's profit loss.

The strategic decisions diverge (Fig.~\ref{fig:fl-spatial-decision}): the decoder opens the same eleven facilities as Nominal and sizes each ${\sim}25\%$ larger, whereas Box consolidates into only four, each ${\sim}4\times$ the nominal-LP's average per-facility capacity (largest is $f_2$ at $166$\,k vs.\ nominal max $43$\,k). The decoder instead spreads its additional capacity across all eleven sites, which is why it pays no comparable profit penalty.

\begin{figure}[tbp]
\centering
\includegraphics[width=0.98\textwidth]{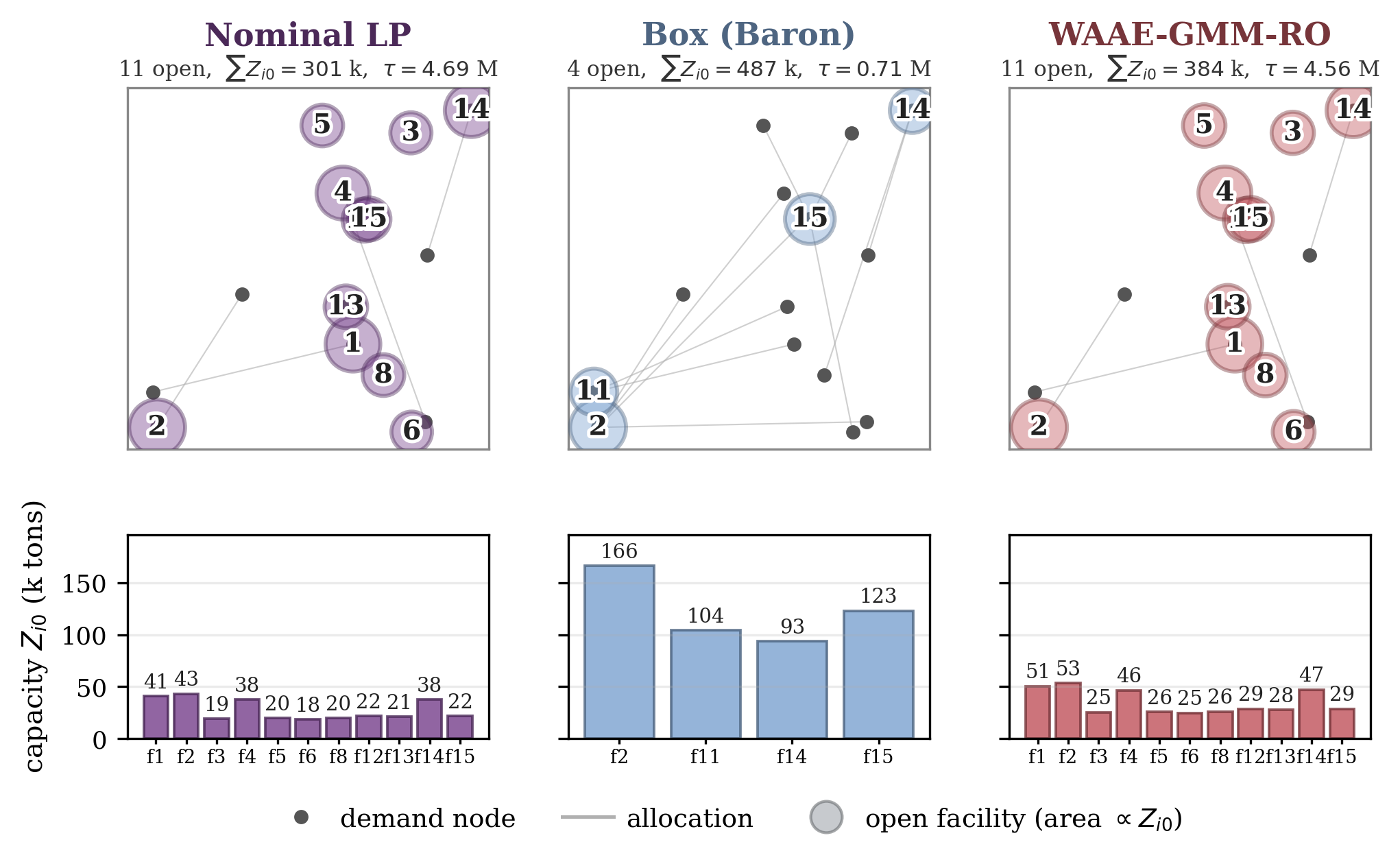}
\caption{Facility-location spatial decisions on the Baron 15-node instance ($T = 20$). Columns left to right: Nominal LP (purple), Box (Baron, blue), and WAAE-GMM-RO (red). Top: 2D node layout with open facilities sized by $\sqrt{Z_{i0}}$ (area $\propto Z_{i0}$), each panel annotated with its open count, total capacity, and profit~$\tau$. Bottom: installed capacity $Z_{i0}$ per open facility on a shared $y$-axis.}
\label{fig:fl-spatial-decision}
\end{figure}

The $T = 20$ instance has $N T = 300$ production constraints in total. Period $t = 0$ is deterministic (Remark~\ref{rem:per-period-bounds}), leaving up to $N(T-1) = 285$ stochastic constraints, and active-facility pruning further reduces this to $11 \times 19 = 209$ MISOCPs dispatched per outer iteration in our run. The per-iteration MISOCP count dominates wall-clock: the convex LP feasibility screen (Section~\ref{sec:oracle}) prunes few queries on the FL geometry, so every outer iteration solves essentially all $209$, even the final one that adds no cut, for ${\sim}1{,}254$ MISOCPs in total ($6 \times 209$). MISOCP solve time scales with the latent radius: the small early-period balls solve in seconds, while the widest radii hit the $300$\,s solver limit, averaging ${\sim}180$\,s overall. Aggregate wall-clock is reported in Table~\ref{tab:fl-t20-results}, with one-off OBBT preprocessing of ${\sim}124$\,s amortised across all queries. The $T = 5$ instance is an order of magnitude cheaper ($N(T-1) = 60$ stochastic constraints, $\gamma_t \in [0, 0.64]$, ${\sim}11$\,min per outer iteration sequentially). Tighter LP screening or MISOCP warm-start strategies are open directions for reducing wall-clock without sacrificing exactness. At the months-to-years strategic-planning horizon, however, the current solution times are well within practical budgets.

\section{Conclusions}
\label{sec:conclusion}

This paper develops Generative Robust Optimisation (GRO), a framework for robust optimisation based on encoder--decoder neural networks. We propose a five-point evaluation framework (reconstruction fidelity, distribution matching, latent regularity, robust relevance, computational tractability) that identifies the jointly necessary conditions any generative model must satisfy to serve as a reliable uncertainty set, and we instantiate an architecture, WAAE-GMM-RO, that satisfies all five points simultaneously through a PCA-preprocessed Wasserstein adversarial autoencoder, a Gaussian-mixture-guided latent regulariser, a constraint-consistency Point-4 loss, and a single-pass ReLU decoder. The trained decoder is integrated into an exact mixed-integer worst-case oracle via OMLT with optimality-based bound tightening, and embedded in a parallel GRO cutting-plane algorithm. Numerical experiments on a six-period production planning problem and a multi-period robust facility location problem demonstrate that the framework reproduces analytical SOCP ground truth on Gaussian data within $\pm 0.5\%$, yields more data-adaptive and less conservative uncertainty sets than the fixed geometric counterparts on the non-Gaussian distributions, where no analytical reference is available to calibrate conservatism exactly, and recovers $97\%$ of the nominal optimum in a $T = 20$ facility location instance where Baron's box counterpart reaches only $15\%$.

Beyond the specific results above, the framework offers two structural advantages. First, the decoder is shape-adaptive: a single ReLU MLP trained on data captures heavy-tailed, skewed, multimodal, and correlated structure that no fixed geometric primitive (box, ellipsoid, polyhedron) can express in one shape, removing the need for problem-specific set engineering. Second, the calibration procedure is uniform: each uncertainty set is the image under the decoder of a calibrated latent set $\Zz$ (a ball or box) in $\N(\bm{0}, \bm{I})$, so a single coverage level controls conservatism across all six production-planning distributions and the facility-location instance, replacing the per-problem derivations that fixed geometric counterparts require. Because the decoder is a learned distribution-to-distribution mapping, the trained model also serves as a portable computational primitive: exported as a single ONNX block, it is reusable without retraining across stochastic programming (by direct $\N(\bm{0}, \bm{I})$ sampling), scenario generation, and distributionally robust counterparts whose ambiguity geometry inherits the decoder's shape rather than being prescribed in data space.

Future work includes more efficient and expressive distribution-mapping architectures that preserve (P5) tractability, faster worst-case oracles through tighter feasibility screens and MISOCP warm-starts, and extensions to higher-dimensional uncertainty and to multi-stage robust optimisation with recourse.

\section*{Data Availability}
The data underlying this article will be shared on reasonable request to the
corresponding author.

\section*{Acknowledgments}
Financial support from EPSRC grant EP/W003317/1 (ADOPT) and UCL's Research
Excellence Scholarship is gratefully acknowledged.

\appendix

\section{Proofs}
\label{ec:proofs}

\subsection{Proof of Proposition~\ref{prop:necessity}}
\label{ec:proof-necessity}
We prove the contrapositive in each case, under
Assumptions~\ref{ass:canonical}--\ref{ass:pool}.

\paragraph{Necessity of (P1) for (I).}
By Assumptions~\ref{ass:canonical} and~\ref{ass:marginal}, coverage is
certified through the canonical encoder path and required on each marginal
$j$: a $p$-fraction of $P^*$ must be both captured by $\Zz$ through $\enc$
and returned by $\dec$ to within the reconstruction tolerance $\tau$. Write
$R_j = \{\bm{\xi} : |\xi_j - \dec(\enc(\bm{\xi}))_j| \le \tau\}$ for the
within-tolerance return set on marginal $j$. The return leg of~(I) is
$\Pr_{P^*}(R_j) \ge p$ for every $j$. If (P1) fails, so that on some
marginal $j$ the reconstruction keeps fewer than a $p$-fraction within
tolerance, then $\Pr_{P^*}(R_j) < p$ and~(I) fails on that marginal.
Faithful reconstruction (P1) is therefore necessary for~(I).

\paragraph{Necessity of (P2) for (I).}
Requirement~(I) is coverage of the true $P^*$, whereas the latent set
$\Zz$ is calibrated against the model's own generated law $\dec\#\N$, the
distribution the calibration of Section~\ref{sec:calibration} actually
sees. Treating the decoded set $\dec(\Zz)$ as covering a $p$-fraction of
$P^*$ is the inheritance premise of that calibration. (P2) carries this
certified coverage from $\dec\#\N$ to $P^*$: any fixed region is
assigned masses by $P^*$ and by $\dec\#\N$ that differ by at most
$\mathrm{TV}(P^*, \dec\#\N)$, so the level-$p$ coverage the calibration
certifies under $\dec\#\N$ transfers to $P^*$ only up to
$\mathrm{TV}(P^*, \dec\#\N)$. A (P2) failure
$\mathrm{TV}(P^*, \dec\#\N) = \delta_2 > 0$ degrades the guaranteed
$P^*$-coverage by up to $\delta_2$ and removes the level-$p$ certificate
the premise needs. The bound is one-sided: (P2) failure removes the
certificate rather than forcing realised coverage below $p$.

\paragraph{Necessity of (P3) for (I).}
(P3) failure means $q(\bm{z}) \neq \N(\bm{0}, \bm{I})$ on some marginal
$j$, so $d_{\mathrm{KS}}(q_j, \N(0,1)) > 0$. By
Theorem~\ref{thm:cal-gap}, $|C_j(p) - p| \le 2\, d_{\mathrm{KS}}(q_j, \N(0,1))$.
In the variance-inflation regime $\sigma_j > 1$, the shift-and-scale
corollary (Corollary~\ref{cor:shift-scale}) gives coverage strictly below
the level $p$, $C_j(p) = 2\Phi(\gamma_{\text{marg}}/\sigma_j) - 1 < p$, so
$\Zz$ captures less than a $p$-fraction of the encoded data on marginal
$j$ and (I) fails. Two other deviation classes do not break~(I): a
deflation $\sigma_j < 1$ gives coverage above $p$, and by the converse of
Theorem~\ref{thm:cal-gap} even-symmetric deviations can preserve
$C_j(p) = p$. (P3) is therefore necessary for~(I) against the
variance-inflation class.

\paragraph{Necessity of (P4) for (III).}
Define (P4) failure by the signed mean gap
$\bar v_{\text{gen}} - \bar v_{\text{ref}}
= K_{\text{pool}}^{-1} \sum_k (v_{\text{gen}}(\bm{x}_k) - v_{\text{ref}}(\bm{x}_k)) < 0$.
Then at least one pool decision $\bm{x}_k$ has
$v_{\text{gen}}(\bm{x}_k) < v_{\text{ref}}(\bm{x}_k)$, which negates
(III). By Assumption~\ref{ass:pool}, $v_{\text{ref}}(\bm{x}_k) > 0$, so
this is a genuine underestimation: the oracle can return
$v_{\text{gen}}(\bm{x}_k) \le 0$ and certify $\bm{x}_k$ robust-feasible
while $v_{\text{ref}}(\bm{x}_k) > 0$ exposes a training scenario
$\bm{\xi}^{(n)} \sim P^*$ with $g(\bm{x}_k, \bm{\xi}^{(n)}) > 0$.

\paragraph{Necessity of (P5) for (II).}
Sub-case A (smooth activations): by Lemma~\ref{lem:activation} no exact
MILP encoding exists, so the oracle falls outside class~(II). Sub-case B
(multi-step depth): for $\dec = f_T \circ \cdots \circ f_1$ with
$T \gg 1$, where layer $f_t$ admits a tight Big-M ReLU encoding with
$b_t$ binaries, the standard compositional encoding uses $\sum_t b_t$
binaries, an order of magnitude beyond the single-pass case for moderate
$T$ and outside the per-query regime reported in
Table~\ref{tab:cs-cost}. \hfill$\square$

\subsection{Proof of Theorem~\ref{thm:cal-gap}}
\label{ec:proof-calgap}
Fix $j$ and $p$. Let $F_j(t) = \Pr_{z_j \sim q_j}(z_j \le t)$, $\Phi$ the
standard normal CDF, and $\gamma_{\text{marg}} = \Phi^{-1}((1+p)/2)$.

\emph{Identity under $q_j = \N(0,1)$.}
$C_j(p) = \Phi(\gamma_{\text{marg}}) - \Phi(-\gamma_{\text{marg}})
= 2\Phi(\gamma_{\text{marg}}) - 1 = p$, using $\Phi(-x) = 1 - \Phi(x)$
and the definition of $\gamma_{\text{marg}}$.

\emph{KS bound.}
$C_j(p) - p = [F_j(\gamma_{\text{marg}}) - \Phi(\gamma_{\text{marg}})]
- [F_j(-\gamma_{\text{marg}}) - \Phi(-\gamma_{\text{marg}})]$. Bounding
each of the two bracketed terms by
$\sup_t |F_j(t) - \Phi(t)| = d_{\mathrm{KS}}(q_j, \N(0,1))$ and adding
gives eq.~\eqref{eq:cal-gap}.

\emph{Converse fails.}
Take $F_j(t) = \Phi(t) + \epsilon\, t^2 e^{-t^2}$. Since $t^2 e^{-t^2}$ vanishes at $\pm\infty$, $F_j(-\infty) = 0$ and
$F_j(+\infty) = 1$. Equivalently, the density perturbation
$\epsilon\, \tfrac{d}{dt}(t^2 e^{-t^2})$ integrates to zero. The perturbed
density $\phi(t) + \epsilon\, \tfrac{d}{dt}(t^2 e^{-t^2})$ is non-negative
for $|\epsilon|$ below a positive threshold, because the
perturbation-to-$\phi$ ratio is bounded on $\R$. Hence $F_j$ is a
monotone, proper CDF.
The perturbation is even, so
$F_j(\gamma_{\text{marg}}) - F_j(-\gamma_{\text{marg}}) =
\Phi(\gamma_{\text{marg}}) - \Phi(-\gamma_{\text{marg}})$, giving
$C_j(p) = p$ for every $p$ while $q_j$ is not Gaussian. \hfill$\square$

\begin{corollary}[Shift-and-scale deviation]
\label{cor:shift-scale}
If $q_j = \N(\mu_j, \sigma_j^2)$, then
$C_j(p) = \Phi\!\left((\gamma_{\text{marg}} - \mu_j)/\sigma_j\right)
- \Phi\!\left((-\gamma_{\text{marg}} - \mu_j)/\sigma_j\right)$. Pure variance
deviation $\mu_j = 0,\ \sigma_j > 1$ gives
$C_j(p) = 2\Phi(\gamma_{\text{marg}}/\sigma_j) - 1 < p$ (coverage below $p$),
whereas $\sigma_j < 1$ gives coverage above $p$. Pure shift ($\sigma_j = 1$) gives
$|C_j(p) - p|$ increasing in $|\mu_j|$.
\end{corollary}

\noindent\textit{Proof.}
For $q_j = \N(\mu_j, \sigma_j^2)$ the CDF is
$F_j(t) = \Phi((t - \mu_j)/\sigma_j)$. Substituting into
$C_j(p) = F_j(\gamma_{\text{marg}}) - F_j(-\gamma_{\text{marg}})$ gives the
stated expression. With $\mu_j = 0$ it reduces to
$2\Phi(\gamma_{\text{marg}}/\sigma_j) - 1$. Since $\Phi$ is strictly
increasing and $\gamma_{\text{marg}} > 0$, this is below
$2\Phi(\gamma_{\text{marg}}) - 1 = p$ when $\sigma_j > 1$ and above it when
$\sigma_j < 1$. For $\sigma_j = 1$, by the symmetry and unimodality of the normal density
the captured mass $C_j(p)$ over the fixed interval $[-\gamma_{\text{marg}},
\gamma_{\text{marg}}]$ is maximised at $\mu_j = 0$ and decreases as
$|\mu_j|$ grows, so $|C_j(p) - p|$ increases in $|\mu_j|$. \hfill$\square$

\subsection{Proof of Lemma~\ref{lem:activation}}
\label{ec:proof-activation}
\emph{Step 1.}
A mixed-integer linear formulation of
$\Gamma_f = \{(\bm{\zeta}, \bm{\eta}) : \bm{\eta} = f(\bm{\zeta})\}$
represents $\Gamma_f$ as the projection onto $(\bm{\zeta}, \bm{\eta})$ of
the feasible region of a finite linear system with continuous
auxiliaries $\bm{u} \in \R^q$ and binaries $\bm{a} \in \{0,1\}^b$
\citep[Definition~1]{Anderson2020}.

\emph{Step 2.}
For each fixed $\bm{a} = \bm{a}^{(\ell)}$ the system reduces to a
polyhedron $P_\ell$ in $(\bm{\zeta}, \bm{\eta}, \bm{u})$. Projecting out
the continuous auxiliaries $\bm{u}$ while retaining
$(\bm{\zeta}, \bm{\eta})$ gives a polyhedron $Q_\ell$,
and $\Gamma_f = \bigcup_\ell Q_\ell$. On the common refinement of the
$\bm{\zeta}$-images $\mathrm{proj}_{\bm{\zeta}}\, Q_\ell$ by their
active-constraint sets (finitely many full-dimensional cells) the active
set is constant and the binding equalities express $\bm{\eta}$ as an
affine function of $\bm{\zeta}$, single-valued since $f$ is a function.

\emph{Step 3.}
The refined cells are finite (at most $2^b$ binary assignments, each
refined into finitely many active sets) and cover $\R^n$ since $f$ is
total, so $f$ is affine cell-by-cell on a finite polyhedral subdivision
of $\R^n$.

\emph{Step 4.}
A real-analytic $g : \R^n \to \R^m$ that agrees with an affine map on a
non-empty open subset $O \subseteq \R^n$ agrees with it on all of $\R^n$:
the difference $\psi$ is real-analytic and vanishes on $O$, so $\psi$ and
all its partial derivatives vanish throughout $O$. The set on which
$\psi$ and all its partial derivatives vanish is closed by continuity and
open by analyticity (there $\psi$ equals its identically-zero Taylor
series locally), hence all of $\R^n$ by connectedness, and
$\psi \equiv 0$.

\emph{Step 5.}
For the activations of interest ($\tanh$, $\sigma$, $\exp$, $\sin$,
$\cos$) the scalar map $\phi$ is real-analytic on all of $\R$, so $f$ is
real-analytic on $\R^n$. Were $f$ affine on a full-dimensional cell of
Step~3 (a non-empty open set), Step~4 would force $f$ to be affine on all
of $\R^n$, contradicting that the analytic activation is a non-degenerate
component (it influences $f$ on a set of positive Lebesgue measure, so
$f$ is not globally affine). Hence $f$ is non-affine on every
full-dimensional cell on which the component is active. By Step~3, $f$ is
then not piecewise-affine on any finite polyhedral subdivision, so no
exact MILP encoding exists and an exact reformulation requires nonlinear
(MINLP) constraints. Piecewise-affine activations such as ReLU instead
admit exact Big-M encodings~\citep[Eq.~6]{Tjeng2019}. \hfill$\square$

\subsection{Proof of Corollary~\ref{cor:gmm-p3}}
\label{ec:proof-gmm}
Fix $j$. By hypotheses (i) and (ii),
$q(z_j) = \sum_{k=1}^K \pi_k\, \N(\mu_{k,j}^z, v_{k,j}^z)$. Linearity of
the integral over a finite sum gives, for any $m \ge 0$,
$\E_{q(z_j)}[z_j^m] = \sum_k \pi_k\, \E_{\N(\mu_{k,j}^z, v_{k,j}^z)}[z_j^m]$.
For $Z \sim \N(\mu, v)$, $\E[Z] = \mu$, $\E[Z^2] = v + \mu^2$,
$\E[Z^3] = \mu^3 + 3\mu v$, $\E[Z^4] = \mu^4 + 6\mu^2 v + 3 v^2$.
Substituting yields the four aggregated moments $M_{1,j}, \ldots, M_{4,j}$
of eq.~\eqref{eq:gmm-comp-loss}. Since
$\loss_{\text{comp}} = \frac{1}{n_z} \sum_{j=1}^{n_z} (M_{1,j}^2 + (M_{2,j} - 1)^2
+ M_{3,j}^2 + (M_{4,j} - 3)^2)$ is a sum of non-negative squares,
$\loss_{\text{comp}} = 0$ forces $M_{1,j} = 0$, $M_{2,j} = 1$,
$M_{3,j} = 0$, $M_{4,j} = 3$, the first four moments of $\N(0,1)$.
\hfill$\square$

\bibliographystyle{plainnat}
\bibliography{references}

\begin{thebibliography}{46}
\providecommand{\natexlab}[1]{#1}
\providecommand{\url}[1]{\texttt{#1}}
\expandafter\ifx\csname urlstyle\endcsname\relax
  \providecommand{\doi}[1]{doi: #1}\else
  \providecommand{\doi}{doi: \begingroup \urlstyle{rm}\Url}\fi

\bibitem[Anderson et~al.(2020)Anderson, Huchette, Ma, Tjandraatmadja, and
  Vielma]{Anderson2020}
Ross Anderson, Joey Huchette, Will Ma, Christian Tjandraatmadja, and Juan~Pablo
  Vielma.
\newblock Strong mixed-integer programming formulations for trained neural
  networks.
\newblock \emph{Mathematical Programming}, 183\penalty0 (1):\penalty0 3--39,
  2020.
\newblock \doi{10.1007/s10107-020-01474-5}.

\bibitem[Andrianesis et~al.(2024)Andrianesis, Bertsimas, Koukouvinos, and
  Koulouras]{Andrianesis2024}
Panagiotis Andrianesis, Dimitris Bertsimas, Thodoris Koukouvinos, and
  Angelos~Georgios Koulouras.
\newblock Ensembling wind forecasting models to construct data-driven
  uncertainty sets in robust optimization.
\newblock In \emph{IEEE Power \& Energy Society General Meeting (PESGM)}, 2024.

\bibitem[Angelopoulos and Bates(2023)]{Angelopoulos2023}
Anastasios~N. Angelopoulos and Stephen Bates.
\newblock Conformal prediction: A gentle introduction.
\newblock \emph{Foundations and Trends in Machine Learning}, 16\penalty0
  (4):\penalty0 494--591, 2023.
\newblock \doi{10.1561/2200000101}.

\bibitem[Baron et~al.(2011)Baron, Milner, and Naseraldin]{Baron2011}
Opher Baron, Joseph Milner, and Hussein Naseraldin.
\newblock Facility location: A robust optimization approach.
\newblock \emph{Production and Operations Management}, 20\penalty0
  (5):\penalty0 772--785, 2011.
\newblock \doi{10.1111/j.1937-5956.2010.01194.x}.

\bibitem[Ben-Tal and Nemirovski(1998)]{BenTalNemirovski1998}
Aharon Ben-Tal and Arkadi Nemirovski.
\newblock Robust convex optimization.
\newblock \emph{Mathematics of Operations Research}, 23\penalty0 (4):\penalty0
  769--805, 1998.

\bibitem[Ben-Tal and Nemirovski(1999)]{BenTalNemirovski1999}
Aharon Ben-Tal and Arkadi Nemirovski.
\newblock Robust solutions of uncertain linear programs.
\newblock \emph{Operations Research Letters}, 25\penalty0 (1):\penalty0 1--13,
  1999.
\newblock \doi{10.1016/S0167-6377(99)00016-4}.

\bibitem[Ben-Tal et~al.(2009)Ben-Tal, El~Ghaoui, and Nemirovski]{BenTal2009}
Aharon Ben-Tal, Laurent El~Ghaoui, and Arkadi Nemirovski.
\newblock \emph{Robust Optimization}.
\newblock Princeton University Press, 2009.

\bibitem[Bertsimas and Sim(2004)]{BertsimasSim2004}
Dimitris Bertsimas and Melvyn Sim.
\newblock The price of robustness.
\newblock \emph{Operations Research}, 52\penalty0 (1):\penalty0 35--53, 2004.

\bibitem[Bertsimas et~al.(2011)Bertsimas, Brown, and Caramanis]{Bertsimas2011}
Dimitris Bertsimas, David~B. Brown, and Constantine Caramanis.
\newblock Theory and applications of robust optimization.
\newblock \emph{SIAM Review}, 53\penalty0 (3):\penalty0 464--501, 2011.

\bibitem[Bertsimas et~al.(2018)Bertsimas, Gupta, and Kallus]{Bertsimas2018}
Dimitris Bertsimas, Vishal Gupta, and Nathan Kallus.
\newblock Data-driven robust optimization.
\newblock \emph{Mathematical Programming}, 167:\penalty0 235--292, 2018.

\bibitem[Bowman et~al.(2016)Bowman, Vilnis, Vinyals, Dai, Jozefowicz, and
  Bengio]{Bowman2016}
Samuel~R. Bowman, Luke Vilnis, Oriol Vinyals, Andrew~M. Dai, Rafal Jozefowicz,
  and Samy Bengio.
\newblock Generating sentences from a continuous space.
\newblock In \emph{Proceedings of the 20th SIGNLL Conference on Computational
  Natural Language Learning (CoNLL)}, pages 10--21, 2016.

\bibitem[Brenner et~al.(2025)Brenner, Khorramfar, Sun, and Amin]{Brenner2025}
Aron Brenner, Rahman Khorramfar, Jennifer Sun, and Saurabh Amin.
\newblock A deep generative learning approach for two-stage adaptive robust
  optimization.
\newblock In \emph{International Conference on Learning Representations
  (ICLR)}, 2025.

\bibitem[Ceccon et~al.(2022)Ceccon, Jalving, Haddad, Thebelt, Tsay, Laird, and
  Misener]{Ceccon2022}
Francesco Ceccon, Jordan Jalving, Joshua Haddad, Alexander Thebelt, Calvin
  Tsay, Carl~D. Laird, and Ruth Misener.
\newblock {OMLT}: {O}ptimization \& {M}achine {L}earning {T}oolkit.
\newblock \emph{Journal of Machine Learning Research}, 23\penalty0
  (349):\penalty0 1--8, 2022.

\bibitem[Chenreddy et~al.(2022)Chenreddy, Bandi, and Delage]{Chenreddy2022}
Abhilash~Reddy Chenreddy, Nymisha Bandi, and Erick Delage.
\newblock Data-driven conditional robust optimization.
\newblock In \emph{Advances in Neural Information Processing Systems
  (NeurIPS)}, 2022.

\bibitem[Dempster et~al.(1977)Dempster, Laird, and Rubin]{Dempster1977EM}
Arthur~P. Dempster, Nan~M. Laird, and Donald~B. Rubin.
\newblock Maximum likelihood from incomplete data via the {EM} algorithm.
\newblock \emph{Journal of the Royal Statistical Society: Series B},
  39\penalty0 (1):\penalty0 1--38, 1977.

\bibitem[Dilokthanakul et~al.(2016)Dilokthanakul, Mediano, Garnelo, Lee,
  Salimbeni, Arulkumaran, and Shanahan]{Dilokthanakul2016GMVAE}
Nat Dilokthanakul, Pedro A.~M. Mediano, Marta Garnelo, Matthew C.~H. Lee, Hugh
  Salimbeni, Kai Arulkumaran, and Murray Shanahan.
\newblock Deep unsupervised clustering with {G}aussian mixture variational
  autoencoders.
\newblock \emph{Preprint, arXiv:1611.02648}, 2016.
\newblock URL \url{https://arxiv.org/abs/1611.02648}.

\bibitem[Fairbrother et~al.(2022)Fairbrother, Turner, and
  Wallace]{Fairbrother2022}
Jamie Fairbrother, Amanda Turner, and Stein~W. Wallace.
\newblock Problem-driven scenario generation: an analytical approach for
  stochastic programs with tail risk measure.
\newblock \emph{Mathematical Programming}, 191\penalty0 (1):\penalty0 141--182,
  2022.
\newblock \doi{10.1007/s10107-019-01451-7}.

\bibitem[Gao et~al.(2024)Gao, Chen, and Kleywegt]{Gao2024}
Rui Gao, Xi~Chen, and Anton~J. Kleywegt.
\newblock Wasserstein distributionally robust optimization and variation
  regularization.
\newblock \emph{Operations Research}, 72\penalty0 (3):\penalty0 1177--1191,
  2024.
\newblock \doi{10.1287/opre.2022.2383}.

\bibitem[Goerigk and Kurtz(2023)]{Goerigk2023}
Marc Goerigk and Jannis Kurtz.
\newblock Data-driven robust optimization using deep neural networks.
\newblock \emph{Computers \& Operations Research}, 151:\penalty0 106087, 2023.

\bibitem[Goodfellow et~al.(2014)Goodfellow, Pouget-Abadie, Mirza, Xu,
  Warde-Farley, Ozair, Courville, and Bengio]{Goodfellow2014}
Ian~J. Goodfellow, Jean Pouget-Abadie, Mehdi Mirza, Bing Xu, David
  Warde-Farley, Sherjil Ozair, Aaron Courville, and Yoshua Bengio.
\newblock Generative adversarial nets.
\newblock In \emph{Advances in Neural Information Processing Systems
  (NeurIPS)}, 2014.

\bibitem[Gulrajani et~al.(2017)Gulrajani, Ahmed, Arjovsky, Dumoulin, and
  Courville]{Gulrajani2017}
Ishaan Gulrajani, Faruk Ahmed, Martin Arjovsky, Vincent Dumoulin, and Aaron
  Courville.
\newblock Improved training of {W}asserstein {GAN}s.
\newblock In \emph{Advances in Neural Information Processing Systems
  (NeurIPS)}, 2017.

\bibitem[{Gurobi Optimization, LLC}(2024)]{Gurobi}
{Gurobi Optimization, LLC}.
\newblock \emph{{Gurobi Optimizer Reference Manual}}, 2024.
\newblock URL \url{https://www.gurobi.com}.

\bibitem[Ho et~al.(2020)Ho, Jain, and Abbeel]{Ho2020}
Jonathan Ho, Ajay Jain, and Pieter Abbeel.
\newblock Denoising diffusion probabilistic models.
\newblock In \emph{Advances in Neural Information Processing Systems
  (NeurIPS)}, 2020.

\bibitem[Hong et~al.(2021)Hong, Huang, and Lam]{Hong2021}
L.~Jeff Hong, Zhiyuan Huang, and Henry Lam.
\newblock Learning-based robust optimization: {P}rocedures and statistical
  guarantees.
\newblock \emph{Management Science}, 67\penalty0 (6):\penalty0 3447--3467,
  2021.
\newblock \doi{10.1287/mnsc.2020.3640}.

\bibitem[Jiang et~al.(2017)Jiang, Zheng, Tan, Tang, and Zhou]{Jiang2017VaDE}
Zhuxi Jiang, Yin Zheng, Huachun Tan, Bangsheng Tang, and Hanning Zhou.
\newblock Variational deep embedding: An unsupervised and generative approach
  to clustering.
\newblock In \emph{Proceedings of the Twenty-Sixth International Joint
  Conference on Artificial Intelligence ({IJCAI}-17)}, pages 1965--1972, 2017.
\newblock \doi{10.24963/ijcai.2017/273}.

\bibitem[Kammammettu et~al.(2024)Kammammettu, Yang, and Li]{Kammammettu2024}
Sanjula Kammammettu, Shu-Bo Yang, and Zukui Li.
\newblock Distributionally robust optimization using optimal transport for
  {G}aussian mixture models.
\newblock \emph{Optimization and Engineering}, 25\penalty0 (3):\penalty0
  1571--1596, 2024.
\newblock \doi{10.1007/s11081-023-09856-2}.

\bibitem[Kingma and Welling(2014)]{KingmaWelling2014}
Diederik~P. Kingma and Max Welling.
\newblock Auto-encoding variational {B}ayes.
\newblock In \emph{International Conference on Learning Representations
  (ICLR)}, 2014.

\bibitem[Kuhn et~al.(2019)Kuhn, Esfahani, Nguyen, and
  Shafieezadeh-Abadeh]{Kuhn2019}
Daniel Kuhn, Peyman~Mohajerin Esfahani, Viet~Anh Nguyen, and Soroosh
  Shafieezadeh-Abadeh.
\newblock Wasserstein distributionally robust optimization: {T}heory and
  applications in machine learning.
\newblock In \emph{INFORMS Tutorials in Operations Research}, pages 130--166.
  INFORMS, 2019.

\bibitem[Kuhn et~al.(2025)Kuhn, Shafiee, and
  Wiesemann]{KuhnShafieeWiesemann2025}
Daniel Kuhn, Soroosh Shafiee, and Wolfram Wiesemann.
\newblock Distributionally robust optimization.
\newblock \emph{Acta Numerica}, 34:\penalty0 579--804, 2025.
\newblock \doi{10.1017/S0962492924000084}.

\bibitem[Li et~al.(2012)Li, Tang, and Floudas]{Li2012}
Zukui Li, Qiuhua Tang, and Christodoulos~A. Floudas.
\newblock A comparative theoretical and computational study on robust
  counterpart optimization: {II}. {P}robabilistic guarantees on constraint
  satisfaction.
\newblock \emph{Industrial \& Engineering Chemistry Research}, 51\penalty0
  (19):\penalty0 6769--6788, 2012.
\newblock \doi{10.1021/ie201651s}.

\bibitem[Loger et~al.(2024)Loger, Dolgui, Lehu{\'e}d{\'e}, and
  Massonnet]{Loger2024}
Beno{\^i}t Loger, Alexandre Dolgui, Fabien Lehu{\'e}d{\'e}, and Guillaume
  Massonnet.
\newblock Approximate kernel learning uncertainty set for robust combinatorial
  optimization.
\newblock \emph{INFORMS Journal on Computing}, 36\penalty0 (3):\penalty0
  900--917, 2024.
\newblock \doi{10.1287/ijoc.2022.0330}.

\bibitem[Lu et~al.(2025)Lu, Ding, Wu, and Yuan]{Lu2025GMM}
Weiguo Lu, Deng Ding, Fengyan Wu, and Gangnan Yuan.
\newblock An efficient {G}aussian mixture model and its application to neural
  networks.
\newblock \emph{Knowledge-Based Systems}, 310:\penalty0 112942, 2025.
\newblock \doi{10.1016/j.knosys.2024.112942}.

\bibitem[Lu and Lu(2020)]{LuLu2020}
Yulong Lu and Jianfeng Lu.
\newblock A universal approximation theorem of deep neural networks for
  expressing probability distributions.
\newblock In \emph{Advances in Neural Information Processing Systems
  (NeurIPS)}, 2020.

\bibitem[Makhzani et~al.(2016)Makhzani, Shlens, Jaitly, Goodfellow, and
  Frey]{Makhzani2015}
Alireza Makhzani, Jonathon Shlens, Navdeep Jaitly, Ian Goodfellow, and Brendan
  Frey.
\newblock Adversarial autoencoders.
\newblock In \emph{International Conference on Learning Representations
  ({ICLR}) Workshop}, 2016.

\bibitem[Marousi and Charitopoulos(2025)]{Marousi2025}
Asimina Marousi and Vassilis~M. Charitopoulos.
\newblock Global robust optimisation for non-convex quadratic programs:
  {A}pplication to pooling problems.
\newblock \emph{Systems and Control Transactions}, 4:\penalty0 1592--1597,
  2025.
\newblock \doi{10.69997/sct.168949}.

\bibitem[Mohajerin~Esfahani and Kuhn(2018)]{MohajerinEsfahani2018}
Peyman Mohajerin~Esfahani and Daniel Kuhn.
\newblock Data-driven distributionally robust optimization using the
  {W}asserstein metric: performance guarantees and tractable reformulations.
\newblock \emph{Mathematical Programming}, 171\penalty0 (1):\penalty0 115--166,
  2018.
\newblock \doi{10.1007/s10107-017-1172-1}.

\bibitem[Papamakarios et~al.(2021)Papamakarios, Nalisnick, Rezende, Mohamed,
  and Lakshminarayanan]{Papamakarios2021}
George Papamakarios, Eric Nalisnick, Danilo~Jimenez Rezende, Shakir Mohamed,
  and Balaji Lakshminarayanan.
\newblock Normalizing flows for probabilistic modeling and inference.
\newblock \emph{Journal of Machine Learning Research}, 22\penalty0
  (57):\penalty0 1--64, 2021.

\bibitem[Ravindran(2008)]{Ravindran2008}
A.~Ravi Ravindran, editor.
\newblock \emph{Operations Research Methodologies}.
\newblock CRC Press, Boca Raton, FL, 2008.
\newblock ISBN 9781420091823.

\bibitem[Rezende and Viola(2018)]{Rezende2018Taming}
Danilo~Jimenez Rezende and Fabio Viola.
\newblock Taming {VAE}s.
\newblock \emph{arXiv:1810.00597}, 2018.

\bibitem[Soyster(1973)]{Soyster1973}
Allen~L. Soyster.
\newblock Technical note---convex programming with set-inclusive constraints
  and applications to inexact linear programming.
\newblock \emph{Operations Research}, 21\penalty0 (5):\penalty0 1154--1157,
  1973.

\bibitem[Stratigakos and Andrianesis(2025)]{Stratigakos2025}
Akylas Stratigakos and Panagiotis Andrianesis.
\newblock Learning data-driven uncertainty set partitions for robust and
  adaptive energy forecasting with missing data.
\newblock \emph{arXiv preprint arXiv:2503.20410}, 2025.

\bibitem[Tjeng et~al.(2019)Tjeng, Xiao, and Tedrake]{Tjeng2019}
Vincent Tjeng, Kai~Y. Xiao, and Russ Tedrake.
\newblock Evaluating robustness of neural networks with mixed integer
  programming.
\newblock In \emph{International Conference on Learning Representations
  (ICLR)}, 2019.

\bibitem[Tolstikhin et~al.(2018)Tolstikhin, Bousquet, Gelly, and
  Sch{\"o}lkopf]{Tolstikhin2018}
Ilya Tolstikhin, Olivier Bousquet, Sylvain Gelly, and Bernhard Sch{\"o}lkopf.
\newblock Wasserstein auto-encoders.
\newblock In \emph{International Conference on Learning Representations
  (ICLR)}, 2018.

\bibitem[Tomczak and Welling(2018)]{Tomczak2018VampPrior}
Jakub~M. Tomczak and Max Welling.
\newblock {VAE} with a {V}amp{P}rior.
\newblock In \emph{AISTATS}, pages 1214--1223, 2018.

\bibitem[Wiesemann et~al.(2014)Wiesemann, Kuhn, and Sim]{WiesemannKuhnSim2014}
Wolfram Wiesemann, Daniel Kuhn, and Melvyn Sim.
\newblock Distributionally robust convex optimization.
\newblock \emph{Operations Research}, 62\penalty0 (6):\penalty0 1358--1376,
  2014.
\newblock \doi{10.1287/opre.2014.1314}.

\bibitem[Xu et~al.(2024)Xu, Lee, Cheng, and Xie]{Xu2024flowdro}
Chen Xu, Jonghyeok Lee, Xiuyuan Cheng, and Yao Xie.
\newblock Flow-based distributionally robust optimization.
\newblock \emph{IEEE Journal on Selected Areas in Information Theory},
  5:\penalty0 62--77, 2024.
\newblock \doi{10.1109/JSAIT.2024.3370699}.

\end{thebibliography}

\end{document}